\documentclass[preprint, 12pt]{elsarticle}
\journal{elsevier}
\usepackage{xr}
\usepackage{array,bm,amsmath,amssymb,mathrsfs,epsfig,epstopdf}
\usepackage{amsthm}
\usepackage{setspace} 
\usepackage{CJK}
\usepackage{color}
\usepackage{lscape}
\usepackage{verbatim}
\usepackage{enumerate}
\usepackage{booktabs}
\usepackage{threeparttable}
\usepackage{multicol}
\usepackage{longtable}
\usepackage{multirow}
\usepackage{bm}
\usepackage{bbm}
\usepackage{cases}
\usepackage{caption}
\usepackage{natbib}
\usepackage{algorithm}
\usepackage{algpseudocode}
\usepackage{subcaption}
\usepackage{makecell}

\usepackage[table]{xcolor}
\newcommand{\best}[1]{\textbf{#1}}
\newcommand{\competitive}[1]{#1$^\dagger$} 

\usepackage{geometry}
\usepackage{framed}
\usepackage[pdftex,bookmarksnumbered,bookmarksopen,colorlinks, citecolor=blue,linkcolor=blue,urlcolor=blue]{hyperref}
\usepackage{graphicx}
\usepackage{amsmath}
\usepackage{amsfonts}
\usepackage{float}
\usepackage{ragged2e}
\usepackage[figuresright]{rotating}
\usepackage{mathtools}
\usepackage{pifont}
\usepackage{indentfirst}  
\usepackage{listings}
\definecolor{codegreen}{rgb}{0,0.6,0}
\definecolor{codegray}{rgb}{0.5,0.5,0.5}
\definecolor{codepurple}{rgb}{0.58,0,0.82}
\definecolor{backcolour}{rgb}{0.95,0.95,0.92}
\lstdefinestyle{mystyle}{
    backgroundcolor=\color{backcolour},   
    commentstyle=\color{codegreen},
    keywordstyle=\color{magenta},
    numberstyle=\tiny\color{codegray},
    stringstyle=\color{codepurple},
    basicstyle=\ttfamily\footnotesize,
    breakatwhitespace=false,         
    breaklines=true,                 
    captionpos=b,                    
    keepspaces=true,                 
    numbers=left,                    
    numbersep=5pt,                  
    showspaces=false,                
    showstringspaces=false,
    showtabs=false,                  
    tabsize=2
}

\newcommand\Vector[1]{\boldsymbol{#1}}

\newcommand\vs{{\Vector{s}}}

\newtheorem{theorem}{Theorem}[section]

\newtheorem{definition}{Definition}[section]

\begin{document}
\title{\bf A Unified Detection Framework for AI-Related Content and Artifacts}

\author[1]{Xifeng Zhang}
\ead{cnxifeng9819@163.com}

\author[1]{Tao Hu}
\ead{hutao@cnu.edu.cn}

\author[2]{Yijie Peng}
\ead{pengyijie@pku.edu.cn}

\author[3,4]{Wan Tian\corref{cor1}}
\ead{wantian61@foxmail.com}

\cortext[cor1]{Corresponding author.}

\address[1]{School of Mathematical Sciences, Capital Normal University, Beijing 100048, China}
\address[2]{School of Management and Engineering, Nanjing University, Nanjing 210008, China}
\address[3]{Advanced Institute of Information Technology, Peking University}
\address[4]{Wangxuan Institute of Computer Technology, Peking University, Beijing 100871, China}

\begin{abstract}
    Artificial intelligence (AI) is a double-edged sword: while it has achieved remarkable success across a wide range of domains, its deployment also calls for effective oversight and regulation, for which the detection of AI-related content and artifacts is perhaps the most direct and cost-effective approach. To this end, we propose a unified detection framework based on Mahalanobis distance scores (MDS), applicable to several important settings, including the detection of large language model (LLM) generated text, hallucination, watermark, and adversarial examples. A key component of the proposed method is to accurately characterize the positive class—such as human-generated text, factual statements, unwatermarked text, or non-adversarial samples—which requires an efficient and robust estimator of the covariance matrix of deep representations of positive samples before computing the MDS. Since the positive samples typically consist of multiple classes, and these classes may exhibit both homogeneity and heterogeneity, we develop joint estimation methods for both the casewise and cellwise minimum covariance determinant (MCD) estimators. We provide efficient optimization algorithms for both estimators and prove their convergence. We provide a reasonable definition of the breakdown point for the joint estimators and prove their corresponding high breakdown point properties. Empirical evaluations confirm the effectiveness of the proposed detection framework.

\end{abstract}

\maketitle

\section{Introduction}

AI, particularly with the advent of foundation models, has rapidly become a general-purpose technology with broad scientific, industrial, and societal impact \citep{bommasani2021opportunities}. Recent advances in LLMs and generative models have enabled the automated production, transformation, and interaction with diverse digital content, including natural language, code, and multimodal data \citep{bubeck2023sparks}. While these capabilities create substantial opportunities in education, healthcare, scientific discovery, and software development, they also introduce AI-related artifacts that are difficult to monitor at scale, such as synthetic text, hallucinated or factually inconsistent outputs, watermark traces, and adversarial perturbations \citep{goodfellow2015explaining,carlini2017towards,kirchenbauer2023watermarking}. These artifacts often manifest as observable statistical irregularities in model outputs or learned representations, for which detection provides a scalable and cost-effective mechanism for AI oversight, especially when the underlying model is inaccessible or continually changing. Existing detection methods can be broadly divided into white-box and black-box paradigms. White-box methods assume access to internal model information, such as parameters, gradients, logits, or decoding probabilities, and can exploit model-specific signals or embedded mechanisms \citep{kirchenbauer2023watermarking}. In contrast, black-box methods operate only on observable outputs or representations obtained from accessible feature extractors, making them more suitable for proprietary, evolving, or API-only systems \citep{manakul2023selfcheckgpt}. From a practical perspective, we therefore focus on black-box detection and discuss this class of methods below.

Despite being studied separately, LLM-generated text detection, hallucination detection, adversarial example detection, and watermark detection can all be viewed as determining whether a sample deviates from a reference distribution in a suitable representation space.
LLM-generated text detection can be viewed as identifying distributional discrepancies between human-written and model-generated text \citep{gehrmann2019gltr,mitchell2023detectgpt}; hallucination detection seeks to find outputs inconsistent with factual, contextual, or self-consistency-based evidence 
\citep{ji2023survey,manakul2023selfcheckgpt}; adversarial example detection aims to recognize abnormal inputs or representations \citep{goodfellow2015explaining,carlini2017towards,metzen2017detecting, tian2024abnormal}; and watermark detection uncovers deliberately embedded statistical signals in generated content \citep{kirchenbauer2023watermarking}. 

A wide range of black-box detection methods have been developed under this perspective. Likelihood-based methods exploit token- or sequence-level probability patterns assigned by language models \citep{gehrmann2019gltr,mitchell2023detectgpt}. Classifier-based methods train supervised or contrastive models on learned representations, with applications to LLM-generated text detection, hallucination detection, and adversarial detection \citep{solaiman2019release}. 
Watermark-specific methods embed and subsequently detect statistical signatures in generated outputs \citep{kirchenbauer2023watermarking}. However, these methods are often task-specific, rely on tailored features or supervised training data, and require retraining or careful data curation when transferred to new domains or models. They also rarely provide a unified statistical treatment of heterogeneous reference samples arising from multiple classes, sources, domains, or generation mechanisms.

To address these limitations, we propose a general black-box detection framework based on MDS. The key idea is to characterize positive samples in a learned representation space and detect negative samples that deviate from this characterization. Specifically, for positive samples, we use an appropriate pretrained model to extract deep representations, estimate their location and covariance structure, and compute MDS for detection. For MDS estimation, we adopt both casewise MCD \citep{hubert2010minimum} and cellwise MCD \citep{raymaekers2024cellwise} as the basic principles. Since positive samples may consist of multiple classes, we further propose joint estimators for both the casewise and cellwise settings using a reparameterization technique, together with efficient optimization algorithms. We establish the convergence of the proposed algorithms and the high breakdown point properties of the resulting estimators. Extensive experiments on LLM-generated text detection, hallucination detection, watermark detection, and adversarial example detection demonstrate the effectiveness and generality of the proposed framework.

It is worth noting that our approach is inspired by \citet{lee2018simple}, who applied class-conditional Gaussian modeling to adversarial and out-of-distribution detection.
However, their method assumes a common covariance matrix across classes and relies on empirical covariance estimation, limiting its robustness in heterogeneous or contaminated settings.
In summary, our contributions are as follows:
\begin{enumerate}
    \item We propose a unified black-box detection framework for a broad class of AI-related content and artifacts, including LLM-generated text, hallucinations, watermarked text, and adversarial examples. The framework is flexible and broadly applicable across different detection tasks.

    \item We develop joint estimation methods for multiple covariance matrices in low-dimensional settings. The proposed methods cover both casewise and cellwise contamination mechanisms and allow positive samples to contain multiple classes with shared and class-specific covariance structures.

    \item We introduce a suitable definition of the breakdown point for joint estimators and establish the high-breakdown-point properties of the proposed joint estimators.
\end{enumerate}

The remainder of this paper is organized as follows. Section~\ref{sec2} introduces the unified detection framework. Sections~\ref{sec3} and~\ref{sec4} present the proposed MCD estimators and their multi-class extensions. Section~\ref{sec5Optimization} develops optimization algorithms, and Section~\ref{sec6Theoretical} establishes convergence guarantees and high-breakdown-point properties. Section~\ref{sec6} reports experimental results, and Section~\ref{sec7} concludes the paper. Implementation details and proofs are provided in the appendices.

\section{Unified detection framework via Mahalanobis Distance Scores} \label{sec2}


In this section, we introduce the proposed unified framework for detecting AI-related content and artifacts. For ease of exposition, we first consider the single-class positive-sample setting in Section~\ref{sec2.1}, and then extend the framework to the multi-class setting in Section~\ref{sec2.2}. Figure~\ref{fig:framework} provides a schematic illustration of the overall framework.

\begin{figure}[H]
\centering
\includegraphics[width=1\linewidth]{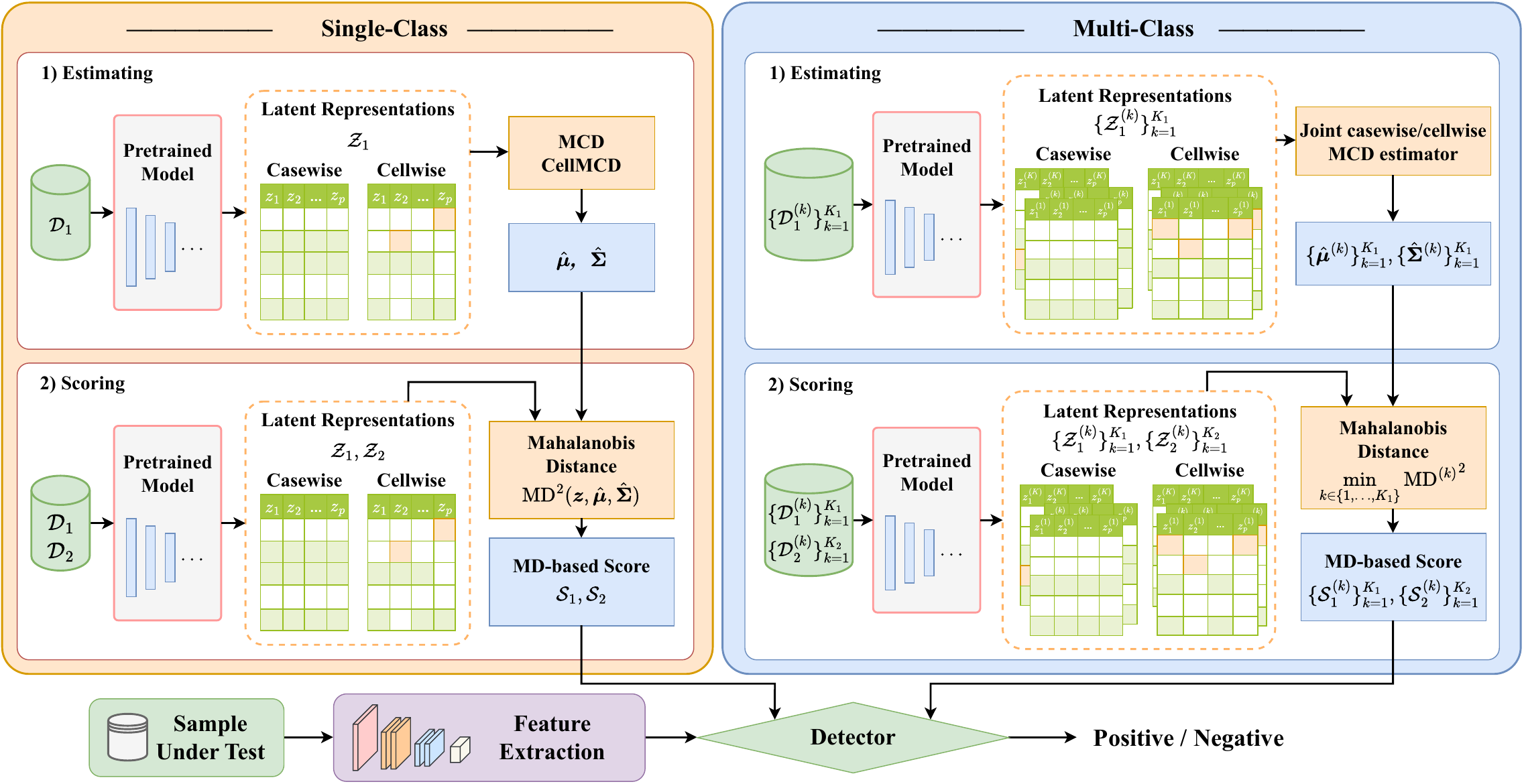} 
\caption{Overview of the proposed framework for detecting AI-related content and artifacts under single-class and multi-class settings.} 
\label{fig:framework}
\end{figure}

\subsection{Single-class MDS detector} \label{sec2.1}
We first consider the setting where positive samples are assumed to come from a single class. For example, in image adversarial example detection, the positive samples may be clean images from one class, while the negative samples are their adversarially perturbed counterparts. Let the set of positive samples be
\(
\mathcal{D}_{1} = \{(\boldsymbol{x}_{i}, Y_{i}=1)\}^n_{i=1},
\)
and let the set of negative samples be
\(
\mathcal{D}_{2} = \bigcup_{k=1}^{K_{2}} \mathcal{D}_{2}^{(k)},
\)
where
\(
\mathcal{D}_{2}^{(k)}
=
\{(\tilde{\boldsymbol{x}}_{i}^{(k)},Y_{i}^{(k)}=-1)\}_{i=1}^{m_k},
k=1,2,\ldots,K_{2}.
\)
The complete dataset is denoted by $\mathcal{D} = \mathcal{D}_{1} \cup \mathcal{D}_{2}$. Let $\mathcal{M}$ be a pretrained model trained on positive samples and used to extract latent representations. For each input $\boldsymbol{x} \in \mathbb{R}^{p}$, we obtain a representation $\boldsymbol{z} \in \mathbb{R}^{q}$ from a chosen layer of $\mathcal{M}$. The corresponding representation sets are denoted by
\[
    \mathcal{Z}_{1} = \{(\boldsymbol{z}_{i}, Y_{i}=1)\}_{i=1}^{n},
    \quad
    \mathcal{Z}_{2}
    =
    \bigcup_{k=1}^{K_{2}}\mathcal{Z}_{2}^{(k)},
\]
where $\mathcal{Z}_{2}^{(k)}=\{(\tilde{\boldsymbol{z}}_{i}^{(k)},Y_{i}^{(k)}=-1)\}_{i=1}^{m_k}$, $k=1,2,\ldots,K_2$. To characterize the positive class, we estimate its mean vector $\boldsymbol{\mu}$ and covariance matrix $\boldsymbol{\Sigma}$ based on the deep representations of positive samples in $\mathcal{Z}_{1}$. This choice is natural because positive samples are typically abundant and easier to obtain, and pretrained models are usually trained primarily on positive samples. Since the positive samples used for estimation may still contain contamination, direct empirical estimation of $\boldsymbol{\mu}$ and $\boldsymbol{\Sigma}$ can be unstable. Therefore, we use robust estimators, denoted by $\widehat{\boldsymbol{\mu}}$ and $\widehat{\boldsymbol{\Sigma}}$, whose construction will be detailed in Section~\ref{sec3}. Using these estimates, the Mahalanobis distance scores for positive and negative samples are defined as
\[
    d_{i}  = (\boldsymbol{z}_{i} - \widehat{\boldsymbol{\mu}})^{\top}
    \widehat{\boldsymbol{\Sigma}}^{-1}
    (\boldsymbol{z}_{i} - \widehat{\boldsymbol{\mu}}),
    \quad
    \tilde{d}_{i}^{(k)} = (\tilde{\boldsymbol{z}}_{i}^{(k)} - \widehat{\boldsymbol{\mu}})^{\top}
    \widehat{\boldsymbol{\Sigma}}^{-1}
    (\tilde{\boldsymbol{z}}_{i}^{(k)} - \widehat{\boldsymbol{\mu}}), \quad k=1,2,\ldots,K_2.
\]
The resulting scores measure how far each sample is from the positive class in the learned representation space. Intuitively, a larger score indicates that the sample is less consistent with the positive class and therefore more likely to be negative. Classification can then be performed on the one-dimensional score sets
\[
    \mathcal{S}_{1} = \{(d_i,Y_i=1)\}_{i=1}^{n},
    \quad
    \mathcal{S}_{2}^{(k)}
    =
    \{(\tilde d_i^{(k)},Y_i^{(k)}=-1)\}_{i=1}^{m_k},
    \quad
    k=1,2,\ldots,K_{2},
\]
by training a score-based binary classifier.

\subsection{Multi-class MDS detector}\label{sec2.2}
We next consider the multi-class setting, where positive samples may come from multiple classes or subpopulations. For example, in image adversarial example detection, the positive samples may consist of clean images from different digit classes, while the negative samples are their adversarially perturbed counterparts. Let the positive and negative sample sets be
\(
    \mathcal{D}_{1} = \bigcup_{k=1}^{K_{1}} \mathcal{D}_{1}^{(k)}
\)
and 
\(
\mathcal{D}_{2} = \bigcup_{k=1}^{K_{2}} \mathcal{D}_{2}^{(k)},
\)
respectively, where
\(
\mathcal{D}_{1}^{(k)}
=
\{(\boldsymbol{x}_{i}^{(k)},Y_{i}^{(k)}=1)\}^{n_k}_{i=1}, k=1,2,\ldots,K_{1}
\)
and 
\(
\mathcal{D}_{2}^{(k)}
=
\{(\tilde{\boldsymbol{x}}_{i}^{(k)},Y_{i}^{(k)}=-1)\}_{i=1}^{m_k},
k=1,2,\ldots,K_{2}.
\)
As in the single-class setting, a pretrained model $\mathcal{M}$ is used to extract latent representations. The corresponding representation sets are denoted by
\[
    \mathcal{Z}_{1}
    =
    \bigcup_{k=1}^{K_{1}}\mathcal{Z}_{1}^{(k)},
    \quad
    \mathcal{Z}_{2}
    =
    \bigcup_{k=1}^{K_{2}}\mathcal{Z}_{2}^{(k)},
\]
where $\mathcal{Z}_{1}^{(k)}=\{(\boldsymbol{z}_{i}^{(k)},Y_{i}^{(k)}=1)\}_{i=1}^{n_k}$ and
$\mathcal{Z}_{2}^{(k)}=\{(\tilde{\boldsymbol{z}}_{i}^{(k)},Y_{i}^{(k)}=-1)\}_{i=1}^{m_k}$. To characterize the positive classes, we estimate class-specific mean vectors
$\boldsymbol{\mu}^{(k)}$ and covariance matrices $\boldsymbol{\Sigma}^{(k)}$ based on the deep representations of positive samples in $\mathcal{Z}_{1}^{(k)}$, $k=1,2,\ldots,K_{1}$. This multi-class formulation allows different positive classes to have distinct geometric structures in the representation space. Since the positive samples used for estimation may still contain contamination, direct empirical estimation of $\boldsymbol{\mu}^{(k)}$ and $\boldsymbol{\Sigma}^{(k)}$ can be unstable. Therefore, we use joint robust estimators, denoted by
\(
\{\widehat{\boldsymbol{\mu}}^{(k)}\}_{k=1}^{K_{1}}
\text{ and }
\{\widehat{\boldsymbol{\Sigma}}^{(k)}\}_{k=1}^{K_{1}},
\)
whose construction will be detailed in Section~\ref{sec4}.

Using these estimates, the Mahalanobis distance score for each sample is defined as its minimum distance to the positive classes:
\[
    \begin{aligned}
        d_{i}^{(k)}
&=
\min_{1\leq j\leq K_{1}}
\left\{
    (\boldsymbol{z}_{i}^{(k)}-\widehat{\boldsymbol{\mu}}^{(j)})^{\top}
    (\widehat{\boldsymbol{\Sigma}}^{(j)})^{-1}
    (\boldsymbol{z}_{i}^{(k)}-\widehat{\boldsymbol{\mu}}^{(j)})
\right\},
\quad k=1,2,\ldots,K_{1}, \\
\tilde d_{i}^{(k)}
&=
\min_{1\leq j\leq K_{1}}
\left\{
    (\tilde{\boldsymbol{z}}_{i}^{(k)}-\widehat{\boldsymbol{\mu}}^{(j)})^{\top}
    (\widehat{\boldsymbol{\Sigma}}^{(j)})^{-1}
    (\tilde{\boldsymbol{z}}_{i}^{(k)}-\widehat{\boldsymbol{\mu}}^{(j)})
\right\},
\quad k=1,2,\ldots,K_{2}.
    \end{aligned}
\]
To ensure that the Mahalanobis distance is well-defined in high-dimensional settings, a positive-definiteness constraint is imposed during covariance estimation. Details are provided in Sections \ref{sec3}-\ref{sec4}. In addition, to mitigate the computational burden and numerical instability associated with high-dimensional covariance estimation, Gaussian random projection is applied to reduce the dimensionality of the representations before covariance estimation. Details are provided in \ref{app:gaussian_reduction}. This minimum-distance rule reflects the intuition that a sample should be regarded as normal if it is close to at least one positive class, whereas a sample far from all positive classes is more likely to be negative. Classification can then be performed on the one-dimensional score sets
\(
\mathcal{S}_{1}^{(k)}
=
\{(d_i^{(k)},Y_i^{(k)}=1)\}_{i=1}^{n_k},
k=1,2,\ldots,K_{1},
\)
and
\(
\mathcal{S}_{2}^{(k)}
=
\{(\tilde d_i^{(k)},Y_i^{(k)}=-1)\}_{i=1}^{m_k},
k=1,2,\ldots,K_{2},
\)
by training a score-based binary classifier.

\section{Single-class robust estimation}\label{sec3}

In the single-class MDS detector introduced in Section~\ref{sec2.1}, the key step is to estimate the mean vector and covariance matrix of the deep representations of positive samples. Let
\(
\boldsymbol{Z}=(\boldsymbol{z}_1,\boldsymbol{z}_2,\ldots,\boldsymbol{z}_n)^\top \in \mathbb{R}^{n\times q}
\)
denote the representation matrix of positive samples. Since these samples may still contain contaminated observations or corrupted feature entries, empirical estimates of the mean and covariance matrix can be unstable. We therefore introduce two robust estimation methods: the casewise MCD, which is designed for contaminated observations, and the cellwise MCD, which is designed for contamination at the feature-entry level.

\subsection{Casewise minimum covariance determinant}

We first consider casewise contamination, where entire observations may be unreliable. To motivate the estimator, consider a Gaussian working model for the positive representations,
\(
\boldsymbol{z}_i \sim \mathcal{N}_{q}(\boldsymbol{\mu},\boldsymbol{\Sigma}), i=1, 2, \ldots, n.
\)
The density is
\(
f(\boldsymbol{z}_i;\boldsymbol{\mu},\boldsymbol{\Sigma})
=
\frac{1}{(2\pi)^{q/2}|\boldsymbol{\Sigma}|^{1/2}}
\exp\left\{
    -\frac{1}{2}\mathrm{MD}^2(\boldsymbol{z}_i,\boldsymbol{\mu},\boldsymbol{\Sigma})
\right\},
\)
where
\(
\mathrm{MD}^2(\boldsymbol{z}_i,\boldsymbol{\mu},\boldsymbol{\Sigma})
=
(\boldsymbol{z}_i-\boldsymbol{\mu})^\top
\boldsymbol{\Sigma}^{-1}
(\boldsymbol{z}_i-\boldsymbol{\mu}).
\)
The negative log-likelihood, up to a constant factor, is
\[
    \ell (\boldsymbol{\mu},\boldsymbol{\Sigma};\boldsymbol{z}_i)
    =
    \ln|\boldsymbol{\Sigma}|
    +
    \mathrm{MD}^2(\boldsymbol{z}_i,\boldsymbol{\mu},\boldsymbol{\Sigma})
    +
    q\ln(2\pi).
\]
If all observations are reliable, minimizing
\(
\sum_{i=1}^{n}\ell (\boldsymbol{\mu},\boldsymbol{\Sigma};\boldsymbol{z}_i) 
\)
leads to the usual empirical mean and covariance matrix. However, this estimator is sensitive to contaminated observations.

The casewise MCD estimator addresses this issue by selecting a subset of \(h\) observations with the smallest covariance determinant. Let \(\boldsymbol{H}=(w_1,w_2,\ldots,w_n)^\top\), where \( w_i \in \{0,1\},  i\in\{1,2,\ldots,n\} \), indicates whether the \(i\)-th observation is retained, with \(\sum_{i=1}^{n} w_i = h\). The casewise MCD estimator solves
\begin{equation}\label{eq:casewise_mcd_obj}
    \min_{\boldsymbol{\mu},\boldsymbol{\Sigma},\boldsymbol{H}}
    \sum_{i=1}^{n}
    w_i
    \ell (\boldsymbol{\mu},\boldsymbol{\Sigma};\boldsymbol{z}_i),
    \quad
    \text{s.t. }
    w_i\in\{0,1\},\ 
    \sum_{i=1}^{n}w_i=h.
\end{equation}
For fixed weights \(\boldsymbol{H}\), the estimators are
\begin{equation*}
    \widehat{\boldsymbol{\mu}}
    =
    \frac{1}{h}
    \sum_{i=1}^{n}
    w_i \boldsymbol{z}_i,\quad \widehat{\boldsymbol{\Sigma}}
    =
    \frac{1}{h}
    \sum_{i=1}^{n}
    w_i
    (\boldsymbol{z}_i-\widehat{\boldsymbol{\mu}})
    (\boldsymbol{z}_i-\widehat{\boldsymbol{\mu}})^\top .
\end{equation*}
Substituting these estimates into \eqref{eq:casewise_mcd_obj} shows that the casewise MCD is equivalent to finding the subset of size \(h\) whose covariance matrix has the minimum determinant.

It is worth noting that directly solving \eqref{eq:casewise_mcd_obj} is computationally challenging, since it involves a combinatorial search over all subsets of size \(h\). In practice, the casewise MCD is commonly computed using the Fast-MCD algorithm proposed by \citet{rousseeuw1999fast}, which efficiently approximates the MCD solution through concentration steps and multiple initial subsets. The tuning parameter \(h\) controls the trade-off between robustness and efficiency: a smaller \(h\) yields higher robustness against contamination, whereas a larger \(h\) uses more observations and can improve statistical efficiency when contamination is mild. A common choice is \(h=\lfloor (n+q+1)/2 \rfloor\), which attains the highest breakdown point, while larger values such as \(h=\lfloor 0.75n \rfloor\) are often used when higher efficiency is desired \citep{hubert2010minimum, rousseeuw1999fast}. The resulting
\(\widehat{\boldsymbol{\mu}}\) and \(\widehat{\boldsymbol{\Sigma}}\) are then used in the single-class MDS detector. 

\subsection{Cellwise minimum covariance determinant}

The casewise MCD removes entire observations and is therefore suitable when contamination affects whole samples. In deep representations, however, contamination may occur only in a subset of feature entries. Removing an entire observation in this case may discard useful information. The cellwise MCD addresses this problem by assigning binary weights at the cell level.
Let $\boldsymbol{W}=(w_{ij})_{1\leq i \leq n, 1\leq j \leq q}\in\{0,1\}^{n\times q}$ be a cellwise indicator matrix, where \(w_{ij}=1\) means that the \(j\)-th entry of \(\boldsymbol{z}_i\) is retained, and \(w_{ij}=0\) means that the entry is flagged as contaminated. Denote by \(\boldsymbol{w}_i\) the \(i\)-th row of \(\boldsymbol{W}\). For each observation, let
\(
\boldsymbol{z}_i^{(\boldsymbol{w}_i)}
\)
be the subvector of \(\boldsymbol{z}_i\) containing only retained entries. Similarly, \(\boldsymbol{\mu}^{(\boldsymbol{w}_i)}\) is the corresponding subvector of \(\boldsymbol{\mu}\), and
\(\boldsymbol{\Sigma}^{(\boldsymbol{w}_i)}\) is the corresponding principal submatrix of \(\boldsymbol{\Sigma}\). Let
\(
q^{(\boldsymbol{w}_i)}
=
\sum_{j=1}^{q} w_{ij}
\)
denote the number of retained entries in the \(i\)-th observation.

Under the same Gaussian working model, the observed likelihood for the retained entries of \(\boldsymbol{z}_i\) is
\[
    f\left(\boldsymbol{z}_{i}^{(\boldsymbol{w}_{i})};
    \boldsymbol{\mu}^{(\boldsymbol{w}_{i})},\boldsymbol{\Sigma}^{(\boldsymbol{w}_{i})}\right)
    =
    \frac{
        1
    }{
        (2\pi)^{q^{(\boldsymbol{w}_i)}/2}
        |\boldsymbol{\Sigma}^{(\boldsymbol{w}_i)}|^{1/2}
    }
    \exp\left\{
        -\frac{1}{2}
        \mathrm{MD}^{2} \left( \boldsymbol{z}_{i}^{(\boldsymbol{w}_{i})},
        \boldsymbol{\mu}^{(\boldsymbol{w}_{i})},\boldsymbol{\Sigma}^{(\boldsymbol{w}_{i})} \right)
    \right\},
\]
where the partial Mahalanobis distance is
\[
    \mathrm{MD}^{2} \left( \boldsymbol{z}_{i}^{(\boldsymbol{w}_{i})},
    \boldsymbol{\mu}^{(\boldsymbol{w}_{i})},\boldsymbol{\Sigma}^{(\boldsymbol{w}_{i})} \right)
    =
    \left(
        \boldsymbol{z}_{i}^{(\boldsymbol{w}_{i})}
        -
        \boldsymbol{\mu}^{(\boldsymbol{w}_{i})}
    \right)^{\top}
    \left(
        \boldsymbol{\Sigma}^{(\boldsymbol{w}_{i})}
    \right)^{-1}
    \left(
        \boldsymbol{z}_{i}^{(\boldsymbol{w}_{i})}
        -
        \boldsymbol{\mu}^{(\boldsymbol{w}_{i})}
    \right).
\]
The corresponding negative log-likelihood is
\[
    \ell \left(\boldsymbol{\mu}^{(\boldsymbol{w}_{i})},\boldsymbol{\Sigma}^{(\boldsymbol{w}_{i})}; \boldsymbol{z}_{i}^{(\boldsymbol{w}_{i})}\right)
    =
    \ln
    |\boldsymbol{\Sigma}^{(\boldsymbol{w}_{i})}|
    +
    q^{(\boldsymbol{w}_{i})}\ln(2\pi)
    +
    \mathrm{MD}^{2} \left( \boldsymbol{z}_{i}^{(\boldsymbol{w}_{i})},
    \boldsymbol{\mu}^{(\boldsymbol{w}_{i})},\boldsymbol{\Sigma}^{(\boldsymbol{w}_{i})} \right).
\]

Unlike ordinary missing-data problems, the cellwise indicator matrix \(\boldsymbol{W}\) is not observed in advance but is estimated from the data. To avoid flagging too many entries, the cellwise MCD imposes a lower bound on the number of retained cells in each column and adds a penalty for flagged cells. Specifically, for a given \(h\), we require
\(
\|\boldsymbol{W}_{\cdot j}\|_0 \geq h, j=1,2,\ldots,q,
\)
where \(\|\boldsymbol{W}_{\cdot j}\|_0\) is the number of retained entries in the \(j\)-th column. We also impose
\(
\lambda_{\min}(\boldsymbol{\Sigma}) \geq a,
\)
where \(\lambda_{\min}(\boldsymbol{\Sigma})\) denotes the smallest eigenvalue of \(\boldsymbol{\Sigma}\), for some \(a>0\), thereby ensuring that the covariance matrix is nonsingular. The cellwise MCD estimator is then defined as the solution to
\begin{equation}\label{eq:cellwise_mcd_obj}
    \begin{aligned}
        \min_{\boldsymbol{\mu},\boldsymbol{\Sigma},\boldsymbol{W}}
        \quad &
        \sum_{i=1}^{n}
        \ell \left(\boldsymbol{\mu}^{(\boldsymbol{w}_{i})},\boldsymbol{\Sigma}^{(\boldsymbol{w}_{i})};\boldsymbol{z}_{i}^{(\boldsymbol{w}_{i})}\right) 
        +
        \sum_{j=1}^{q}
        b_j
        \|\mathbf{1}_{n}-\boldsymbol{W}_{\cdot j}\|_0
        \\
        \mathrm{s.t.}
        \quad &
        \boldsymbol{W}\in\{0,1\}^{n\times q},
        \quad
        \|\boldsymbol{W}_{\cdot j}\|_0 \geq h,\quad j=1,2,\ldots,q,\quad
        \lambda_{\min}(\boldsymbol{\Sigma})\geq a .
    \end{aligned}
\end{equation}
Here \(\mathbf{1}_{n}\) denotes the \(n\)-dimensional all-ones vector, \(\|\mathbf{1}_{n}-\boldsymbol{W}_{\cdot j}\|_0\) counts the number of flagged cells in the \(j\)-th feature dimension, and \(b_j>0\) controls the penalty for flagging cells in that dimension. If an observation has no retained entries, i.e., $q^{(\boldsymbol w_i)}=0$, we set $\ln |\boldsymbol{\Sigma}^{(\boldsymbol w_i)}| = 0$ and $\mathrm{MD}^2( \boldsymbol z_i^{(\boldsymbol w_i)}, \boldsymbol\mu^{(\boldsymbol w_i)}, \boldsymbol\Sigma^{(\boldsymbol w_i)} )=0$.

It is worth noting that directly solving \eqref{eq:cellwise_mcd_obj} is computationally challenging, because the objective involves both continuous parameters \((\boldsymbol{\mu},\boldsymbol{\Sigma})\) and a high-dimensional binary matrix \(\boldsymbol{W}\). Following \citet{raymaekers2024cellwise}, the cellwise MCD can be computed through an iterative procedure that alternates between updating the cellwise indicator matrix and estimating the mean vector and covariance matrix from the retained cells. The penalty parameters \(b_j\) regulate the degree of cellwise flagging and are typically chosen so that, in the absence of contamination, only a small expected fraction of cells is flagged. The parameter \(h\) specifies the minimum number of retained entries in each feature dimension and prevents degenerate solutions that remove too many cells. Thus, \(h\) and \(b_j\) jointly control the robustness--efficiency trade-off: more aggressive flagging improves robustness to cellwise contamination, whereas retaining more cells improves efficiency when contamination is mild. The resulting \(\widehat{\boldsymbol{\mu}}\) and \(\widehat{\boldsymbol{\Sigma}}\) are then used in the single-class MDS detector.

\section{Multi-class robust estimation}\label{sec4}

In the multi-class MDS detector introduced in Section~\ref{sec2.2}, positive samples may come from multiple classes. The covariance structures of different classes may exhibit both homogeneity and heterogeneity. Estimating the covariance matrix for each class separately may fail to borrow information across classes and thus lead to insufficient statistical efficiency. Therefore, building on the casewise and cellwise MCD principles, we extend these two robust estimators to the multi-class setting and propose corresponding joint estimation methods. For notational simplicity, we write the total number of classes as $K$ in this section. 

\subsection{Joint casewise MCD estimator}

We first consider casewise contamination, where entire observations may be unreliable within each class. Let \(\boldsymbol{Z}=\{\boldsymbol{Z}^{(1)},\ldots,\boldsymbol{Z}^{(K)}\}\), where \( \boldsymbol{Z}^{(k)}=(\boldsymbol{z}_{1}^{(k)},\ldots,\boldsymbol{z}_{n_k}^{(k)})^\top \in \mathbb{R}^{n_k\times q} \) denotes the representation matrix of the \(k\)-th positive class, \(k=1,2,\ldots,K\). As a working model, we assume
\(
\boldsymbol{z}_{i}^{(k)}
\sim
\mathcal{N}_{q}
(\boldsymbol{\mu}^{(k)}, \boldsymbol{\Sigma}^{(k)}),
i=1,2,\ldots,n_k, k=1,2,\ldots,K,
\)
where \(\boldsymbol{\mu}^{(k)}\in\mathbb{R}^{q}\) and
\(\boldsymbol{\Sigma}^{(k)}\in\mathbb{R}^{q\times q}\) denote the class-specific mean vector and covariance matrix. Following the reparameterization idea in the joint estimation of multiple graphical models \citep{guo2011joint}, we introduce the following decomposition to explicitly model both shared and class-specific dependence structures:
\begin{equation}\label{eq:sigma-decomposition}
\boldsymbol{\Sigma}^{(k)}
=
\boldsymbol{\Theta} \odot \boldsymbol{\Gamma}^{(k)},
\end{equation}
where \(\odot\) denotes the Hadamard product. Equivalently, for
\(\boldsymbol{\Sigma}^{(k)}=(\sigma_{ij}^{(k)})_{1\leq i,j \leq q}\),
\[
    \sigma_{ij}^{(k)}
    =
    \begin{cases}
        \theta_{ij}\gamma_{ij}^{(k)}, & i\neq j,\\
        \gamma_{ii}^{(k)}, & i=j,
    \end{cases}
\]
with \(\theta_{ii}=1\) for \(i=1,2,\ldots,q\). Here
\(\boldsymbol{\Theta}=(\theta_{ij})_{1\leq i,j \leq q}\) represents the shared covariance pattern across classes, whereas
\(\boldsymbol{\Gamma}^{(k)}=(\gamma_{ij}^{(k)})_{1\leq i,j \leq q}\) captures class-specific variation. This reparameterization separates shared dependence ($\boldsymbol{\Theta}$) from class-specific deviations ($\boldsymbol{\Gamma}^{(k)}$).
We use the Root Mean Square (RMS) correlation normalization rule to ensure the identifiability of the decomposition, and details are given in Appendix~\ref{app:decomp_indent}.

For $k=1,2,\ldots,K$, let \(w_i^{(k)}\in\{0,1\}\) indicate whether the \(i\)-th observation in class \(k\) is retained. We write
\(
\boldsymbol{H}^{(k)}
=
(w_1^{(k)},\ldots,w_{n_k}^{(k)})^\top
\)
for the vector of casewise weights in class \(k\) with
\(
\sum_{i=1}^{n_k}w_i^{(k)}=h_k.
\)
The joint casewise MCD estimator is defined as the solution to
\begin{equation}\label{eq:joint_casewise_mcd}
    \begin{aligned}
        \min_{\{\boldsymbol{\mu}^{(k)}, \boldsymbol{H}^{(k)}, \boldsymbol{\Sigma}^{(k)}\}_{k=1}^{K}}
        \quad & 
        \sum_{k=1}^{K} \sum_{i=1}^{n_k} w_i^{(k)} \tilde{\ell}(\boldsymbol{\mu}^{(k)}, \boldsymbol{\Sigma}^{(k)};\boldsymbol{z}_i^{(k)})
        +
        \mathbf{P}(\boldsymbol{\Theta}, \{\boldsymbol{\Gamma}^{(k)}\}_{k=1}^{K})
        \\
        \mathrm{s.t.}
        \quad &
        w_i^{(k)}\in\{0,1\},
        \quad
        \sum_{i=1}^{n_k}w_i^{(k)}=h_k,
        \quad
        \boldsymbol{\Sigma}^{(k)}=\boldsymbol{\Theta}\odot\boldsymbol{\Gamma}^{(k)},
        \quad
        \lambda_{\min}(\boldsymbol{\Sigma}^{(k)})\geq a_k,
    \end{aligned}
\end{equation}
where
\[
    \tilde{\ell}({\boldsymbol{\mu}}^{(k)}, \boldsymbol{\Sigma}^{(k)};\boldsymbol{z}_i^{(k)}) =  \ln|\boldsymbol{\Sigma}^{(k)}| + q\ln(2\pi) + \mathrm{MD}^2 \left( \boldsymbol{z}_{i}^{(k)}, \boldsymbol{\mu}^{(k)}, \boldsymbol{\Sigma}^{(k)} \right)
\]
with
\(
\mathrm{MD}^2
\left(
    \boldsymbol{z}_{i}^{(k)},
    \boldsymbol{\mu}^{(k)},
    \boldsymbol{\Sigma}^{(k)}
\right)
=
(\boldsymbol{z}_{i}^{(k)}-\boldsymbol{\mu}^{(k)})^\top
(\boldsymbol{\Sigma}^{(k)})^{-1}
(\boldsymbol{z}_{i}^{(k)}-\boldsymbol{\mu}^{(k)})
\), $k=1,2,\ldots,K$. Following \citet{guo2011joint}, we impose the penalty
\begin{equation}\label{eq:joint_casewise_penalty}
    \mathbf{P}(\boldsymbol{\Theta}, \{\boldsymbol{\Gamma}^{(k)}\}_{k=1}^{K})
    =
    \lambda_1 \|\boldsymbol{\Theta}\|_{1,\mathrm{off}} + \lambda_2 \sum_{k=1}^{K} \|\boldsymbol{\Gamma}^{(k)}\|_{1,\mathrm{off}}
    :=
    \lambda_1
    \sum_{i\neq j}
    |\theta_{ij}|
    +
    \lambda_2
    \sum_{k=1}^{K}
    \sum_{i\neq j}
    |\gamma_{ij}^{(k)}|,
\end{equation}
where \(\lambda_1,\lambda_2>0\) control the global and class-specific sparsity levels, respectively. Here, $\|\cdot\|_{1, \mathrm{off}}$ denotes the sum of the absolute values of the off-diagonal elements of a matrix. The first term promotes a sparse shared structure through \(\boldsymbol{\Theta}\): when \(\theta_{ij}\) is set to zero, the corresponding covariance entry is removed from all classes simultaneously, encouraging homogeneity across classes. The second term regularizes the class-specific components \(\boldsymbol{\Gamma}^{(k)}\), allowing different classes to retain distinct covariance magnitudes and signs when the shared component is nonzero, thereby preserving heterogeneity. Therefore, the penalty jointly encourages common structure across classes and flexible class-specific deviations. The constraint \(\lambda_{\min}(\boldsymbol{\Sigma}^{(k)})\geq a_k > 0\) ensures that each covariance matrix is nonsingular. 


\subsection{Joint cellwise MCD estimator}
We next consider cellwise contamination, where only a subset of feature entries may be unreliable. In this setting, removing entire observations can be inefficient, since many uncontaminated entries may still contain useful information. The joint cellwise MCD estimator therefore assigns binary weights at the cell level within each class. For class \(k\), let
\(
\boldsymbol{W}^{(k)}=(w_{ij}^{(k)})_{1\leq i \leq n_k, 1\leq j \leq q}\in\{0,1\}^{n_k\times q}
\)
be the cellwise indicator matrix, where \(w_{ij}^{(k)}=1\) means that the \(j\)-th entry of \(\boldsymbol{z}_{i}^{(k)}\) is retained, and \(w_{ij}^{(k)}=0\) means that the entry is flagged as contaminated. Let \(\boldsymbol{w}_{i}^{(k)}\) denote the \(i\)-th row of \(\boldsymbol{W}^{(k)}\), and let \( q^{(\boldsymbol{w}_{i}^{(k)})} = \sum_{j=1}^{q} w_{ij}^{(k)} \) be the number of retained entries in observation \(i\) of class \(k\). For any vector or matrix, the superscript
\((\boldsymbol{w}_{i}^{(k)})\) denotes the subvector or principal submatrix corresponding to the retained entries. Thus,
\(
\boldsymbol{z}_{i}^{(\boldsymbol{w}_{i}^{(k)}, k)},
\boldsymbol{\mu}^{(\boldsymbol{w}_{i}^{(k)}, k)},
\boldsymbol{\Sigma}^{(\boldsymbol{w}_{i}^{(k)}, k)}
\)
denote the retained subvector of \(\boldsymbol{z}_{i}^{(k)}\), the corresponding subvector of
\(\boldsymbol{\mu}^{(k)}\), and the corresponding principal submatrix of
\(\boldsymbol{\Sigma}^{(k)}\), respectively. Using the same covariance decomposition in \eqref{eq:sigma-decomposition}, the joint cellwise MCD estimator is defined by minimizing
\begin{equation}\label{eq:joint_cellwise_mcd}
    \begin{aligned}
        \min_{\{\boldsymbol{\mu}^{(k)}, \boldsymbol{W}^{(k)}, \boldsymbol{\Sigma}^{(k)}\}_{k=1}^{K}}
        \quad &
        \sum_{k=1}^{K}
        \sum_{i=1}^{n_k}
        \tilde{\ell} \left(
            \boldsymbol{\mu}^{(\boldsymbol{w}_{i}^{(k)}, k)},
            \boldsymbol{\Sigma}^{(\boldsymbol{w}_{i}^{(k)}, k)};
            \boldsymbol{z}_{i}^{(\boldsymbol{w}_{i}^{(k)}, k)}
        \right)
        \\
              &\quad
              +
              \sum_{k=1}^{K}
              \sum_{j=1}^{q}
              b_j^{(k)}
              \left\|
              \mathbf{1}_{n_k}
              -
              \boldsymbol{W}_{\cdot j}^{(k)}
              \right\|_0
              +
              \mathbf{P}(\boldsymbol{\Theta}, \{\boldsymbol{\Gamma}^{(k)}\}_{k=1}^{K})
              \\
              \mathrm{s.t.}
        \quad &
        \boldsymbol{W}^{(k)}\in\{0,1\}^{n_k\times q},
        \quad
        \left\|\boldsymbol{W}_{\cdot j}^{(k)}\right\|_0\geq h_k, \quad j=1,2,\ldots,q,
        \\
              &
              \boldsymbol{\Sigma}^{(k)}=\boldsymbol{\Theta}\odot\boldsymbol{\Gamma}^{(k)},
              \quad
              \lambda_{\min}(\boldsymbol{\Sigma}^{(k)})\geq a_k > 0,
              \quad k=1,2,\ldots,K.
    \end{aligned}
\end{equation}
Let $\mathrm{MD}^2 (\boldsymbol{z}_{i}^{(\boldsymbol{w}_{i}^{(k)}, k)},\boldsymbol{\mu}^{(\boldsymbol{w}_{i}^{(k)}, k)},\boldsymbol{\Sigma}^{(\boldsymbol{w}_{i}^{(k)}, k)}) = (\boldsymbol{z}_{i}^{(\boldsymbol{w}_{i}^{(k)},k)}-\boldsymbol{\mu}^{(\boldsymbol{w}_{i}^{(k)},k)})^{\top} (\boldsymbol{\Sigma}^{(\boldsymbol{w}_{i}^{(k)},k)})^{-1} (\boldsymbol{z}_{i}^{(\boldsymbol{w}_{i}^{(k)},k)}-\boldsymbol{\mu}^{(\boldsymbol{w}_{i}^{(k)},k)})$ be the partial Mahalanobis distance, we have $\tilde{\ell}(\boldsymbol{\mu}^{(\boldsymbol{w}_{i}^{(k)}, k)},\boldsymbol{\Sigma}^{(\boldsymbol{w}_{i}^{(k)}, k)};\boldsymbol{z}_{i}^{(\boldsymbol{w}_{i}^{(k)}, k)}) = \ln |\boldsymbol{\Sigma}^{(\boldsymbol{w}_{i}^{(k)},k)}| + q^{(\boldsymbol{w}_{i}^{(k)})}\ln(2\pi) + \mathrm{MD}^2 (\boldsymbol{z}_{i}^{(\boldsymbol{w}_{i}^{(k)}, k)},\boldsymbol{\mu}^{(\boldsymbol{w}_{i}^{(k)}, k)},\boldsymbol{\Sigma}^{(\boldsymbol{w}_{i}^{(k)}, k)})$.
Here, \(\mathbf{1}_{n_k}\) is the \(n_k\)-dimensional all-ones vector, and
\(
\|
\mathbf{1}_{n_k}
-
\boldsymbol{W}_{\cdot j}^{(k)}
\|_0
\)
counts the number of flagged cells in the \(j\)-th feature dimension of class \(k\). The constants \(b_j^{(k)}\) penalize cellwise flagging, while the constraints
\(
\|\boldsymbol{W}_{\cdot j}^{(k)}\|_0\geq h_k
\)
prevent too many entries from being removed in any feature dimension. If all entries of an observation are flagged, we adopt the empty-pattern convention \(q^{(\boldsymbol w_i^{(k)})}=0\), \(\ln |\boldsymbol{\Sigma}^{(\boldsymbol w_i^{(k)},k)}|=0\), and \(\mathrm{MD}^2=0\), so that the objective function remains well-defined. $\lambda_{\min}(\boldsymbol{\Sigma}^{(k)})\geq a_k > 0$ guarantees that the covariance matrix is nonsingular. The penalty \(\mathbf{P}(\boldsymbol{\Theta}, \{\boldsymbol{\Gamma}^{(k)}\}_{k=1}^{K})\) is the same as in \eqref{eq:joint_casewise_penalty}, encouraging homogeneity through the shared component \(\boldsymbol{\Theta}\) while preserving heterogeneity by the class-specific components \(\boldsymbol{\Gamma}^{(k)}\). 

\section{Optimization algorithms} \label{sec5Optimization}

In this section, we develop optimization algorithms for the two joint estimators introduced in Section~\ref{sec4}. The main difficulty is that both estimators involve discrete selection variables and continuous covariance parameters under the shared-structure reparameterization
\(\boldsymbol{\Sigma}^{(k)}=\boldsymbol{\Theta}\odot\boldsymbol{\Gamma}^{(k)}\).
For the joint casewise MCD estimator, the discrete variables correspond to the retained observations in each class. For the joint cellwise MCD estimator, the discrete variables correspond to retained cells in each class-specific representation matrix. We solve both problems by alternating between updating the discrete selection variables and updating the continuous parameters. We use $t \in \mathbb{N}$ as a superscript on the estimator to represent the result obtained in the $t$-th iteration, where $\mathbb{N}$ represents the set of non-negative integers, and $t=0$ represents the initial value.

\subsection{Optimization for joint casewise MCD estimator}\label{sec5Optimization1}

Under casewise contamination, we first initialize, for each group $k$, a subset $\boldsymbol{H}^{(k,0)}$ of size $h_k$. Starting from this initial subset, the algorithm iteratively updates both the model parameters and the subset composition until convergence. We first update the group means based on the current subset:
\begin{equation*}
    \widehat{\boldsymbol{\mu}}^{(k,t+1)} = \frac{\sum_{i=1}^{n_k} w_i^{(k,t)} \boldsymbol{z}_i^{(k)}}{\sum_{i=1}^{n_k} w_i^{(k,t)}} = \frac{\sum_{i=1}^{n_k} w_i^{(k,t)} \boldsymbol{z}_i^{(k)}}{h_k}, \quad k = 1,2,\ldots, K.
\end{equation*}

We next update the shared structure $\widehat{\boldsymbol{\Theta}}^{(t+1)}$ and the group-specific heterogeneity components $\widehat{\boldsymbol{\Gamma}}^{(k,t+1)}$ via alternating optimization:

\begin{itemize}
    \item We first fix $\widehat{\boldsymbol{\Theta}}^{(t)}$ and update each $\widehat{\boldsymbol{\Gamma}}^{(k,t+1)}$:
        \begin{equation*}
            \widehat{\boldsymbol{\Gamma}}^{(k,t+1)} = \arg\min_{\boldsymbol{\Gamma}} \sum_{i=1}^{n_k} w_i^{(k,t)} \tilde{\ell}(\widehat{\boldsymbol{\mu}}^{(k,t+1)}, \widehat{\boldsymbol{\Theta}}^{(t)} \odot \boldsymbol{\Gamma}; \boldsymbol{z}_i^{(k)}) + \lambda_2 \sum_{i \neq j} |\gamma_{ij}|.
        \end{equation*}
    \item We then fix all $\widehat{\boldsymbol{\Gamma}}^{(k,t+1)}$ and update the shared structure $\widehat{\boldsymbol{\Theta}}^{(t+1)}$:
        \begin{equation*}
            \widehat{\boldsymbol{\Theta}}^{(t+1)} = \arg\min_{\boldsymbol{\Theta}} \sum_{k=1}^{K} \sum_{i=1}^{n_k} w_i^{(k,t)} \tilde{\ell}(\widehat{\boldsymbol{\mu}}^{(k,t+1)}, \boldsymbol{\Theta} \odot \widehat{\boldsymbol{\Gamma}}^{(k,t+1)}; \boldsymbol{z}_i^{(k)}) + \lambda_1 \sum_{i \neq j} |\theta_{ij}|.
        \end{equation*}
\end{itemize}

The covariance matrices are updated by $\widehat{\boldsymbol{\Sigma}}^{(k, t+1)} = \widehat{\boldsymbol{\Theta}}^{(t+1)} \odot \widehat{\boldsymbol{\Gamma}}^{(k,t+1)}, k = 1,2,\ldots, K$. After updating the parameters, the algorithm computes the Mahalanobis distances of all samples to their group centers using the updated mean and covariance estimates. Within each group, the $h_k$ samples with the smallest Mahalanobis distances are selected as unflagged samples, forming the updated subset $\boldsymbol{H}^{(k,t+1)}$. This process is repeated until convergence. To enhance robustness, a reweighting step is performed after convergence, following the Fast-MCD algorithm \cite{rousseeuw1999fast}. The subset is updated based on Mahalanobis distances, and the final parameter estimates are computed with respect to this reweighted subset. The complete computation procedure is presented in Algorithm~\ref{alg:fastmcd_joint}. 
\begin{algorithm}[htbp]
    \caption{Fast-MCD Algorithm for Joint Casewise MCD Estimator (JCASEMCD)}
    \label{alg:fastmcd_joint} 
    \begin{algorithmic}
        \Require Data $\boldsymbol{Z}$, subset sizes $\{h_k\}_{k=1}^{K}$, penalties $\lambda_1, \lambda_2$, error threshold $\varepsilon$, eigenvalue thresholds $\{a_k\}_{k=1}^K$.

        \State Initialize $t = 0$ and generate initial estimates $\widehat{\boldsymbol{\mu}}^{(k,0)}$, $\widehat{\boldsymbol{\Theta}}^{(0)}$, $\widehat{\boldsymbol{\Gamma}}^{(k,0)}$, $\widehat{\boldsymbol{\Sigma}}^{(k, 0)}$, and initial subsets $\boldsymbol{H}^{(k,0)}$ for $k = 1,2,\ldots,K$.

        \While{not converged}
        \State Update group means:
        \[
            \widehat{\boldsymbol{\mu}}^{(k,t+1)} \leftarrow \frac{\sum_{i=1}^{n_k} w_i^{(k,t)} \boldsymbol{z}_i^{(k)}}{h_k}, \quad k=1,2,\ldots,K.
        \]

        \State Update structure parameters:
        \begin{enumerate}
            \item For $k=1,2,\ldots,K$, fix $\widehat{\boldsymbol{\Theta}}^{(t)}$ and update $\widehat{\boldsymbol{\Gamma}}^{(k,t+1)}$:
                \[
                    \widehat{\boldsymbol{\Gamma}}^{(k,t+1)} \leftarrow \arg\min_{\boldsymbol{\Gamma}} \sum_{i=1}^{n_k} w_i^{(k,t)} \tilde{\ell} \left(\widehat{\boldsymbol{\mu}}^{(k,t+1)}, \widehat{\boldsymbol{\Theta}}^{(t)} \odot \boldsymbol{\Gamma}; \boldsymbol{z}_i^{(k)}\right) + \lambda_2 \sum_{i \neq j} |\gamma_{ij}|.
                \]
            \item Fix all $\widehat{\boldsymbol{\Gamma}}^{(k,t+1)}$ and update $\widehat{\boldsymbol{\Theta}}^{(t+1)}$:
                \[
                    \widehat{\boldsymbol{\Theta}}^{(t+1)} \leftarrow \arg\min_{\boldsymbol{\Theta}} \sum_{k=1}^{K} \sum_{i=1}^{n_k} w_i^{(k,t)} \tilde{\ell} \left(\widehat{\boldsymbol{\mu}}^{(k,t+1)}, \boldsymbol{\Theta} \odot \widehat{\boldsymbol{\Gamma}}^{(k,t+1)} ; \boldsymbol{z}_i^{(k)}\right) + \lambda_1 \sum_{i \neq j} |\theta_{ij}|.
                \]
        \end{enumerate}

        \State Update covariance matrices:
        \[
            \widehat{\boldsymbol{\Sigma}}^{(k, t+1)} \leftarrow \widehat{\boldsymbol{\Theta}}^{(t+1)} \odot \widehat{\boldsymbol{\Gamma}}^{(k,t+1)}, \quad k=1,2,\ldots,K.
        \]
        
        \State Update subsets $\boldsymbol{H}^{(k,t+1)}$ with the $h_k$ samples having smallest Mahalanobis distances.

        \If{$\max_{k} \|\widehat{\boldsymbol{\Sigma}}^{(k,t+1)} - \widehat{\boldsymbol{\Sigma}}^{(k,t)}\|_{F} < \varepsilon$ and $\max_{k} \|  \widehat{\boldsymbol{\mu}}^{(k,t+1)}  - \widehat{\boldsymbol{\mu}}^{(k,t)} \|_2  < \varepsilon$}
        \State Converged, exit loop.
        \Else
        \State $t \leftarrow t+1$.
        \EndIf
        \EndWhile

        \State Reweighting: Update weights $w_i^{(k,\text{final})}$ based on Mahalanobis distances and recompute $\widehat{\boldsymbol{\mu}}^{(k,\text{final})}$, $\widehat{\boldsymbol{\Sigma}}^{(k,\text{final})}$, $\widehat{\boldsymbol{\Theta}}^{(\text{final})}$, $\hat{\boldsymbol{\Gamma}}^{(k,\text{final})}$.

        \State \Return $\widehat{\boldsymbol{\mu}}^{(k,\text{final})}$, $\widehat{\boldsymbol{\Sigma}}^{(k,\text{final})}$, $\widehat{\boldsymbol{\Theta}}^{(\text{final})}$, $\hat{\boldsymbol{\Gamma}}^{(k,\text{final})}$
    \end{algorithmic}
\end{algorithm}

\subsection{Optimization for joint cellwise MCD estimator}\label{sec5Optimization2}

For the cellwise setting, it is difficult to directly solve Equation~\eqref{eq:joint_cellwise_mcd}. The optimization proceeds by first estimating the weight matrices $\boldsymbol{W}$ and then updating the parameters given the missing entries. For ease of discussion, we introduce notation and partitions. In the $t$-th iteration, for the $k$-th group, the $i$-th row is divided into three parts $\boldsymbol{z}_i^{(k)} = (z_{ij}^{(k)}, \boldsymbol{z}_{i,o}^{(k)\top}, \boldsymbol{z}_{i,-o}^{(k)\top})^{\top}$ after a suitable permutation of coordinates, 
where $z_{ij}^{(k)}$ is the element in row $i$ and column $j$, and $\boldsymbol{z}_{i,o}^{(k)}$ and $\boldsymbol{z}_{i,-o}^{(k)}$ correspond to entries marked as normal $(w_{i\tau}^{(k,t)} = 1)$ and abnormal $(w_{i\tau}^{(k,t)} = 0)$, $\tau=1,2,\ldots,q$, respectively, with $o = \{\tau \neq j: w_{i\tau}^{(k,t)} = 1\}$. Accordingly, the mean vector $\widehat{\boldsymbol{\mu}}^{(k)}$ and covariance matrix $\widehat{\boldsymbol{\Sigma}}^{(k)}$ are partitioned as
\[
    \widehat{\boldsymbol{\mu}}^{(k)} = \begin{pmatrix}
        \mu_j^{(k)} \\
        \widehat{\boldsymbol{\mu}}_o^{(k)} \\
        \widehat{\boldsymbol{\mu}}_{-o}^{(k)}
    \end{pmatrix}, \quad
    \widehat{\boldsymbol{\Sigma}}^{(k)} = \begin{pmatrix}
        \widehat{\sigma}_{jj}^{(k)} & \widehat{\boldsymbol{\Sigma}}_{j,o}^{(k)} & \widehat{\boldsymbol{\Sigma}}_{j,-o}^{(k)} \\
        \widehat{\boldsymbol{\Sigma}}_{o,j}^{(k)} & \widehat{\boldsymbol{\Sigma}}_{o,o}^{(k)} & \widehat{\boldsymbol{\Sigma}}_{o,-o}^{(k)} \\
        \widehat{\boldsymbol{\Sigma}}_{-o,j}^{(k)} & \widehat{\boldsymbol{\Sigma}}_{-o,o}^{(k)} & \widehat{\boldsymbol{\Sigma}}_{-o,-o}^{(k)}
    \end{pmatrix}.
\]

We first initialize each group separately by computing initial location estimates $\widehat{\boldsymbol{\mu}}^{(k,0)}$, covariance estimates $\widehat{\boldsymbol{\Sigma}}^{(k,0)}$, and weight matrices $\boldsymbol{W}^{(k,0)}$, ensuring that each column contains at least $h_k$ non-missing entries not flagged as outliers. An iterative process then alternates between updating cellwise weights and updating parameters. In the $t$-th iteration, with $\widehat{\boldsymbol{\mu}}^{(k,t)}$ and $\widehat{\boldsymbol{\Sigma}}^{(k,t)}$ fixed, the objective function used for updating the weight of entry $(i,j)$ in group $k$ is
\begin{equation}\label{eq:W_delta}
    \Delta_{ij}^{(k)} = \ln(C_{ij}^{(k)}) + \ln(2\pi) + \frac{(z_{ij}^{(k)} - \widehat{z}_{ij}^{(k)})^2}{C_{ij}^{(k)}} - b_j^{(k)},
\end{equation}
where $b_j^{(k)} = \max \{ \chi_{1-\alpha}^2(1) + \ln (2\pi) + \ln (1/([(\widehat{\boldsymbol{\Sigma}}^{(k,0)})^{-1}]_{jj}), 0\}$, with $\chi_{1-\alpha}^2(1)$ representing the chi-square $1-\alpha$ quantile with 1 degree of freedom. Moreover, the conditional mean and variance are
\begin{equation}\label{eq:cond_mean_var_W}
    \widehat{z}_{ij}^{(k)} = \widehat{\mu}_j^{(k,t)} + \widehat{\boldsymbol{\Sigma}}_{j,o}^{(k,t)} (\widehat{\boldsymbol{\Sigma}}_{o,o}^{(k,t)})^{-1} (\boldsymbol{z}_{i,o}^{(k)} - \widehat{\boldsymbol{\mu}}_o^{(k,t)}),
    \quad
    C_{ij}^{(k)} = \widehat{\sigma}_{jj}^{(k,t)} - \widehat{\boldsymbol{\Sigma}}_{j,o}^{(k,t)} (\widehat{\boldsymbol{\Sigma}}_{o,o}^{(k,t)})^{-1} (\widehat{\boldsymbol{\Sigma}}_{j,o}^{(k,t)})^\top.
\end{equation}
The index set $o$ is determined by the current weight matrix $\boldsymbol{W}^{(k,t)}$. All cells with $\Delta_{ij}^{(k)}<0$ are retained. If fewer than $h_k$ cells are retained in the $j$-th column, the smallest remaining values of $\Delta_{ij}^{(k)}$ are additionally retained until $h_k$ cells are retained \cite{raymaekers2024cellwise}. Then, the Expectation-Maximization (EM) algorithm \citep{dempster1977maximum} updates $\widehat{\boldsymbol{\mu}}^{(k,t+1)}$, $\widehat{\boldsymbol{\Gamma}}^{(k,t+1)}$, and $\widehat{\boldsymbol{\Theta}}^{(t+1)}$.

\texttt{E-step}. Given $\boldsymbol{W}^{(k,t+1)}$, the conditional expectation and covariance for the $i$-th observation in group $k$ are computed. Let $\mathcal{O}_i = \{\tau: w_{i\tau}^{(k,t+1)}=1\}$, $\mathcal{U}_i = \{\tau: w_{i\tau}^{(k,t+1)}=0\}$, and let $\boldsymbol{z}_{i, \mathcal{O}_i}^{(k)}$ be a subvector indexed by $\mathcal{O}_i$. The conditional expectation is $\boldsymbol{m}_i^{(k,t)} = ({\boldsymbol{m}_{i,\mathcal{O}_i}^{(k,t)}}^{\top}, {\boldsymbol{m}_{i,\mathcal{U}_i}^{(k,t)}}^{\top})^{\top}$, where
\begin{equation}\label{eq:cond_mean}
    \boldsymbol{m}_{i,\mathcal{O}_i}^{(k,t)} = \boldsymbol{z}_{i,\mathcal{O}_i}^{(k)}, \quad
    \boldsymbol{m}_{i,\mathcal{U}_i}^{(k,t)} = \widehat{\boldsymbol{\mu}}_{\mathcal{U}_i}^{(k,t)} + \widehat{\boldsymbol{\Sigma}}_{\mathcal{U}_i,\mathcal{O}_i}^{(k,t)} (\widehat{\boldsymbol{\Sigma}}_{\mathcal{O}_i,\mathcal{O}_i}^{(k,t)})^{-1} (\boldsymbol{z}_{i,\mathcal{O}_i}^{(k)} - \widehat{\boldsymbol{\mu}}_{\mathcal{O}_i}^{(k,t)}),
\end{equation}
and the conditional covariance $V_{i}^{(k,t)}$ satisfies
\begin{equation}\label{eq:cond_var}
    \begin{aligned}
        V_{i,{\mathcal{O}_i,\mathcal{O}_i}}^{(k,t)} = \mathbf{O},\quad
        V_{i,{\mathcal{O}_i,\mathcal{U}_i}}^{(k,t)} = V_{i,{\mathcal{U}_i,\mathcal{O}_i}}^{(k,t)} = \mathbf{0},\quad
        V_{i,{\mathcal{U}_i,\mathcal{U}_i}}^{(k,t)} = \widehat{\boldsymbol{\Sigma}}_{\mathcal{U}_i,\mathcal{U}_i}^{(k,t)} - \widehat{\boldsymbol{\Sigma}}_{\mathcal{U}_i,\mathcal{O}_i}^{(k,t)} (\widehat{\boldsymbol{\Sigma}}_{\mathcal{O}_i,\mathcal{O}_i}^{(k,t)})^{-1} \widehat{\boldsymbol{\Sigma}}_{\mathcal{O}_i,\mathcal{U}_i}^{(k,t)},
    \end{aligned}
\end{equation}
where $\mathbf{O}$ and $\mathbf{0}$ represent a matrix and a vector containing only zeros, respectively. When $\mathcal{O}_i=\varnothing$, we use the convention $\boldsymbol{m}_i^{(k,t)}=\widehat{\boldsymbol{\mu}}^{(k,t)}$ and $V_i^{(k,t)}=\widehat{\boldsymbol{\Sigma}}^{(k,t)}$. The conditional second-order central moment is then
\[
    \mathbb{E}\left[
    (\boldsymbol{z}_i^{(k)}-\boldsymbol{\mu}^{(k)})
    (\boldsymbol{z}_i^{(k)}-\boldsymbol{\mu}^{(k)})^\top
    \,\middle\vert\,
    \boldsymbol{z}_{i,\mathcal{O}_i}^{(k)} 
    \right]
    =
    V_i^{(k,t)}
    +
    (\boldsymbol{m}_i^{(k,t)}-\boldsymbol{\mu}^{(k)})
    (\boldsymbol{m}_i^{(k,t)}-\boldsymbol{\mu}^{(k)})^\top .
\]

\texttt{M-step}. We first update the location parameters $\widehat{\boldsymbol{\mu}}^{(k,t+1)}$:
\(
\widehat{\boldsymbol{\mu}}^{(k,t+1)} = \frac{1}{n_k} \sum_{i=1}^{n_k} \boldsymbol{m}_i^{(k,t)},
\)
and compute
\[
    S_{\mathrm{cell}}^{(k,t)} = \frac{1}{n_k} \sum_{i=1}^{n_k} \left[ V_i^{(k,t)} + (\boldsymbol{m}_i^{(k,t)} - \widehat{\boldsymbol{\mu}}^{(k,t+1)})(\boldsymbol{m}_i^{(k,t)} - \widehat{\boldsymbol{\mu}}^{(k,t+1)})^\top \right].
\]
Next, fixing $\widehat{\boldsymbol{\Theta}}^{(t)}$, the group-specific structure $\widehat{\boldsymbol{\Gamma}}^{(k,t+1)}$ is updated by solving
\[
    \widehat{\boldsymbol{\Gamma}}^{(k,t+1)} = \arg\min_{\boldsymbol{\Gamma}} n_k \left[ \ln |\widehat{\boldsymbol{\Theta}}^{(t)} \odot \boldsymbol{\Gamma}| + \mathrm{tr}((\widehat{\boldsymbol{\Theta}}^{(t)} \odot \boldsymbol{\Gamma})^{-1} S_{\mathrm{cell}}^{(k,t)}) \right] + \lambda_2 \sum_{i \neq j} |\gamma_{ij}|.
\]
Finally, the shared structure $\widehat{\boldsymbol{\Theta}}^{(t+1)}$ is updated by
\[
    \widehat{\boldsymbol{\Theta}}^{(t+1)} = \arg\min_{\boldsymbol{\Theta}} \sum_{k=1}^{K} n_k \left[ \ln |\boldsymbol{\Theta} \odot \widehat{\boldsymbol{\Gamma}}^{(k,t+1)}| + \mathrm{tr}((\boldsymbol{\Theta} \odot \widehat{\boldsymbol{\Gamma}}^{(k,t+1)})^{-1} S_{\mathrm{cell}}^{(k,t)}) \right] + \lambda_1 \sum_{i\neq j} |\theta_{ij}|.
\]

The covariance matrices for each group are then reconstructed via
\(
\widehat{\boldsymbol{\Sigma}}^{(k,t+1)} = \widehat{\boldsymbol{\Theta}}^{(t+1)} \odot \widehat{\boldsymbol{\Gamma}}^{(k,t+1)},
\)
and the algorithm iterates until convergence. The full procedure is summarized in Algorithm~\ref{alg:joint-cellwise-mcd}.

\begin{algorithm}[H]
    \caption{Computation procedure for joint cellwise MCD estimator (JCELLMCD)}
    \label{alg:joint-cellwise-mcd} 
    \begin{algorithmic}
        \Require $\boldsymbol{Z} = (\boldsymbol{Z}^{(1)}, \ldots, \boldsymbol{Z}^{(K)})$, $\{h_k\}_{k=1}^K$, $\{a_k\}_{k=1}^K$, $\lambda_1, \lambda_2$.
        \State Initialize $t=0$, $\widehat{\boldsymbol{\mu}}^{(k,0)}$, $\widehat{\boldsymbol{\Sigma}}^{(k,0)}$, $\boldsymbol{W}^{(k,0)}$, $\widehat{\boldsymbol{\Theta}}^{(0)}$, $\widehat{\boldsymbol{\Gamma}}^{(k,0)}$ and evaluate $b_j^{(k)}$ for $k=1,2,\ldots,K$.

        \While{not converged}
        \State For $k = 1,2,\ldots, K$, update cellwise weights (Algorithm~\ref{alg:cellwise-updater}):
        \[
            \boldsymbol{W}^{(k,t+1)} = \texttt{CellwiseWeightsUpdater} (\boldsymbol{Z}^{(k)}, \widehat{\boldsymbol{\mu}}^{(k,t)}, \widehat{\boldsymbol{\Sigma}}^{(k,t)}, \boldsymbol{W}^{(k,t)}, h_k, b_j^{(k)}).
        \]
        \State Update parameters via EM (Algorithm~\ref{alg:em-update}):
        \[
            (\widehat{\boldsymbol{\mu}}^{(k,t+1)}, \widehat{\boldsymbol{\Theta}}^{(t+1)}, \widehat{\boldsymbol{\Gamma}}^{(k,t+1)})_{k=1}^{K} = \texttt{EM}(\boldsymbol{Z}, \{\boldsymbol{W}^{(k,t+1)}\}_{k=1}^{K}, (\widehat{\boldsymbol{\mu}}^{(k,t)}, \widehat{\boldsymbol{\Theta}}^{(t)}, \widehat{\boldsymbol{\Gamma}}^{(k,t)})_{k=1}^{K}, \{a_k\}_{k=1}^{K}, \lambda_1, \lambda_2).
        \]
        \State Reconstruct covariance matrices: $\widehat{\boldsymbol{\Sigma}}^{(k,t+1)} = \widehat{\boldsymbol{\Theta}}^{(t+1)} \odot \widehat{\boldsymbol{\Gamma}}^{(k,t+1)}, k=1,2,\ldots,K$.
        \State Check convergence and update $t = t+1$.
        \EndWhile
        \State \Return $\widehat{\boldsymbol{\mu}}^{(k,\text{final})}, \widehat{\boldsymbol{\Sigma}}^{(k,\text{final})}, \widehat{\boldsymbol{\Theta}}^{(\text{final})}, \widehat{\boldsymbol{\Gamma}}^{(k,\text{final})}, \boldsymbol{W}^{(k,\text{final})}$
    \end{algorithmic}
\end{algorithm}

\begin{algorithm}[H]
    \caption{\texttt{CellwiseWeightsUpdater} step for Algorithm \ref{alg:joint-cellwise-mcd}}
    \label{alg:cellwise-updater}
    \begin{algorithmic}
        \Require $\boldsymbol{Z}^{(k)}$, $\widehat{\boldsymbol{\mu}}^{(k,t)}$,$\widehat{\boldsymbol{\Sigma}}^{(k,t)}$, $\boldsymbol{W}^{(k,t)}$, $h_k$, $b_j^{(k)}$
        \State ${\boldsymbol{W}}^{(k,t+1)} = \boldsymbol{W}^{(k,t)}$
        \For{$j = 1, \dots, q$}
        \For{$i = 1, \dots, n_k$}
        \State $o = \{\tau \neq j : w_{i\tau}^{(k,t+1)} = 1\}$.
        \State Compute conditional mean and variance with Equation~\eqref{eq:cond_mean_var_W}.
        \State Compute objective function $\Delta_{ij}^{(k)}$ with Equation~\eqref{eq:W_delta}.
        \EndFor
        \State Let $N_{-} = \#\{\Delta_{ij}^{(k)}<0,i=1,2,\ldots,n_k\}$, and update weights:
        \[
            {w}_{ij}^{(k,t+1)} = \mathbb{I}\{N_{-} \geq h_k\} \mathbb{I}\{\Delta_{ij}^{(k)}<0\} + \mathbb{I}\{N_{-} < h_k\} \mathbb{I}\{\Delta_{ij}^{(k)} \text{ among } h_k \text{ smallest}\}
        \]
        \EndFor
        \State \Return $\boldsymbol{W}^{(k,t+1)}$
    \end{algorithmic}
\end{algorithm}

\begin{algorithm}[H]
    \caption{\texttt{EM} step for Algorithm \ref{alg:joint-cellwise-mcd}}
    \label{alg:em-update}
    \begin{algorithmic}
        \Require $\boldsymbol{Z}$, $\{\boldsymbol{W}^{(k,t+1)}\}_{k=1}^{K}$, $(\widehat{\boldsymbol{\mu}}^{(k,t)}$, $\widehat{\boldsymbol{\Theta}}^{(t)}$, $\widehat{\boldsymbol{\Gamma}}^{(k,t)})_{k=1}^{K}$, $\{a_k\}_{k=1}^{K}$, $\lambda_1, \lambda_2$.
        \For{$k = 1, \dots, K$}
        \For{$i = 1, \dots, n_k$}
        \State Partition indices: $\mathcal{O}_i = \{\tau: w_{i\tau}^{(k)} = 1\}, \mathcal{U}_i = \{\tau: w_{i\tau}^{(k)} = 0\}$.
        \State Compute conditional mean $\boldsymbol{m}_i^{(k,t)}$ with Equation~\eqref{eq:cond_mean}.
        \State Compute conditional covariance $V_{i}^{(k,t)}$ with Equation~\eqref{eq:cond_var}.
        \EndFor
        \State Update location: $\widehat{\boldsymbol{\mu}}^{(k,t+1)} = \frac{1}{n_k} \sum_{i=1}^{n_k} \boldsymbol{m}_i^{(k,t)}$
        \State Compute $S_{\mathrm{cell}}^{(k,t)} = \frac{1}{n_k} \sum_{i=1}^{n_k} \big[V_i^{(k,t)} + (\boldsymbol{m}_i^{(k,t)} - \widehat{\boldsymbol{\mu}}^{(k,t+1)})(\boldsymbol{m}_i^{(k,t)} - \widehat{\boldsymbol{\mu}}^{(k,t+1)})^\top \big]$.
        \State Update group-specific structures:
        \[
            \widehat{\boldsymbol{\Gamma}}^{(k,t+1)} = \arg\min_{\boldsymbol{\Gamma}} n_k \left[ \ln |\widehat{\boldsymbol{\Theta}}^{(t)} \odot \boldsymbol{\Gamma}| + \mathrm{tr}((\widehat{\boldsymbol{\Theta}}^{(t)} \odot \boldsymbol{\Gamma})^{-1} S_{\mathrm{cell}}^{(k,t)}) \right] + \lambda_2 \sum_{i \neq j} |\gamma_{ij}|.
        \]
        \EndFor
        \State Update shared structure:
        \[
            \widehat{\boldsymbol{\Theta}}^{(t+1)} = \arg\min_{\boldsymbol{\Theta}} \sum_{k=1}^{K} n_k \left[ \ln |\boldsymbol{\Theta} \odot \widehat{\boldsymbol{\Gamma}}^{(k,t+1)}| + \mathrm{tr}((\boldsymbol{\Theta} \odot \widehat{\boldsymbol{\Gamma}}^{(k,t+1)})^{-1} S_{\mathrm{cell}}^{(k,t)}) \right] + \lambda_1 \sum_{i\neq j} |\theta_{ij}|.
        \]
        \State \Return $(\widehat{\boldsymbol{\mu}}^{(k,t+1)}, \widehat{\boldsymbol{\Theta}}^{(t+1)}, \widehat{\boldsymbol{\Gamma}}^{(k,t+1)})_{k=1}^{K}$
    \end{algorithmic}
\end{algorithm}

The updates for $\widehat{\boldsymbol{\Theta}}$ and $\{\widehat{\boldsymbol{\Gamma}}^{(k)}\}_{k=1}^{K}$, which correspond to the Fast-MCD algorithm for the joint casewise MCD estimator in Section~\ref{sec5Optimization1} and the EM-based procedure for the joint cellwise MCD estimator in Section~\ref{sec5Optimization2}, are essentially constrained optimization problems. We use proximal gradient descent with backtracking search to obtain approximate solutions. Detailed update formulas are provided in \ref{app:theta_gamma_optimization}. Details such as algorithm initialization, the identifiability of the covariance decomposition, and dimensionality reduction in the implementation are also given in \ref{app:details}.

\section{Theoretical properties} \label{sec6Theoretical}

In this section, we discuss the main theoretical properties of the proposed joint estimators. 
Specifically, the joint casewise MCD estimator is optimized via the Fast-MCD algorithm, and the joint cellwise MCD estimator is optimized via the EM-based procedure.
Detailed statements and proofs of these properties are provided in the subsequent theorems and \ref{app:proof}.

\subsection{Convergence guarantees}
The following theorems establish convergence guarantees for the proposed joint estimators. In particular, they show that the algorithms generate non-increasing sequences of their respective objective values and that these objective values converge to finite limits. Moreover, any interior accumulation point satisfies suitable blockwise stationarity conditions for the continuous variables, together with optimality or coordinatewise optimality of the corresponding discrete selection variables.

\begin{theorem}\label{thm:jointedfastmcd_converge}
    Let $\boldsymbol{\Omega} = (\boldsymbol{\Omega}^{(1)}, \ldots, \boldsymbol{\Omega}^{(K)})$, with $\boldsymbol{\Omega}^{(k)} = (\boldsymbol{\mu}^{(k)}, \boldsymbol{\Sigma}^{(k)})$ and $\boldsymbol{\Sigma}^{(k)} = \boldsymbol{\Theta} \odot \boldsymbol{\Gamma}^{(k)}$ for $k = 1,2,\ldots, K$. Define the regularized joint negative log-likelihood objective function as
    \begin{equation*}
        \begin{aligned}
            \mathcal{L}_{\mathrm{case}}(\boldsymbol{\Omega}, \boldsymbol{H})
            &= \sum\limits_{k=1}^{K} \sum\limits_{i=1}^{n_k} w_i^{(k)} \tilde{\ell}(\boldsymbol{\mu}^{(k)}, \boldsymbol{\Sigma}^{(k)};\boldsymbol{z}_i^{(k)}) + \mathbf{P}(\boldsymbol{\Theta}, \{\boldsymbol{\Gamma}^{(k)}\}_{k=1}^{K}) \\
            &:= \sum\limits_{k=1}^{K} Q_{\mathrm{case}}^{(k)}(\boldsymbol{\Omega}, \boldsymbol{H}) + \lambda_1 \|\boldsymbol{\Theta}\|_{1, \mathrm{off}} + \lambda_2 \sum\limits_{k=1}^{K} \|\boldsymbol{\Gamma}^{(k)}\|_{1, \mathrm{off}},
        \end{aligned}
    \end{equation*}
    where $w_i^{(k)} \in \{0,1\}$ and $\sum_{i=1}^{n_k} w_i^{(k)} = h_k$. Suppose $\lambda_{\min}(\boldsymbol{\Sigma}^{(k)}) \ge a_k > 0$, $\lambda_1 > 0$, and $\lambda_2 > 0$. For each iteration $t$, let $L_{\boldsymbol{\Gamma},k}^{(t)}$ and $L_{\boldsymbol{\Theta}}^{(t)}$ be the Lipschitz constants of the gradients of the smooth parts of the $\boldsymbol{\Gamma}^{(k)}$- and $\boldsymbol{\Theta}$-subproblems, respectively. The proximal gradient step sizes are chosen with a strict descent margin, namely, for some constants $\delta_{\boldsymbol{\Gamma}},\delta_{\boldsymbol{\Theta}}\in(0,1)$, $L_{\boldsymbol{\Gamma},k} \geq \sup_{s \geq 0} L_{\boldsymbol{\Gamma},k}^{(s)} > 0$ and $L_{\boldsymbol{\Theta}} \geq \sup_{s \geq 0} L_{\boldsymbol{\Theta}}^{(s)} > 0$, $k = 1,2,\ldots, K$,
    \begin{equation*}
        0<\inf_{s \geq 0} \eta_{\boldsymbol{\Gamma},k}^{(s)}\leq\eta_{\boldsymbol{\Gamma},k}^{(t)}
        \le
        \frac{1-\delta_{\boldsymbol{\Gamma}}}{L_{\boldsymbol{\Gamma},k}},
        \quad
        0<\inf_{s \geq 0} \eta_{\boldsymbol{\Theta}}^{(s)}\leq\eta_{\boldsymbol{\Theta}}^{(t)}
        \le
        \frac{1-\delta_{\boldsymbol{\Theta}}}{L_{\boldsymbol{\Theta}}}.
    \end{equation*}
    Then the pre-reweighting iterations of Algorithm~\ref{alg:fastmcd_joint} satisfy:
    \begin{enumerate}
        \item $\mathcal{L}_{\mathrm{case}}(\boldsymbol{\Omega}^{(t+1)}, \boldsymbol{H}^{(t+1)}) \le \mathcal{L}_{\mathrm{case}}(\boldsymbol{\Omega}^{(t+1)}, \boldsymbol{H}^{(t)}) \le \mathcal{L}_{\mathrm{case}}(\boldsymbol{\Omega}^{(t)}, \boldsymbol{H}^{(t)})$;
        \item The sequence $\mathcal{L}_{\mathrm{case}}(\boldsymbol{\Omega}^{(t)}, \boldsymbol{H}^{(t)})$ converges to a finite value $\mathcal{L}^{*}_{\mathrm{case}}$;
        \item Any interior accumulation point $(\boldsymbol{\Omega}^{*}, \boldsymbol{H}^{*})$, namely one satisfying $\lambda_{\min}(\boldsymbol{\Theta}^{*} \odot \boldsymbol{\Gamma}^{(k)*}) > a_k>0,k = 1,2,\ldots,K$ is blockwise stationary with respect to the continuous variables, and $\boldsymbol{H}^*$ is optimal for the discrete $\boldsymbol{H}$-block given $\boldsymbol{\Omega}^*$. 
    \end{enumerate} 
\end{theorem}

Theorem~\ref{thm:jointedfastmcd_converge} ensures that the Fast-MCD algorithm for the joint casewise MCD estimator, before the final reweighting step, produces a non-increasing sequence of the regularized negative log-likelihood values. This implies that each pre-reweighting iteration either decreases or maintains the objective value, and that the objective values converge to a finite limit. Moreover, any interior accumulation point is blockwise stationary with respect to the continuous parameters and is optimal with respect to the discrete casewise selection block. Consequently, practitioners can rely on the algorithm to produce stable estimates for casewise-contaminated data.

\begin{theorem}\label{thm:jointedcellmcd_converge}
    Let $\boldsymbol{\Omega} = (\boldsymbol{\Omega}^{(1)}, \ldots, \boldsymbol{\Omega}^{(K)})$, with $\boldsymbol{\Omega}^{(k)} = (\boldsymbol{\mu}^{(k)}, \boldsymbol{\Sigma}^{(k)})$ and $\boldsymbol{\Sigma}^{(k)} = \boldsymbol{\Theta} \odot \boldsymbol{\Gamma}^{(k)}$ for $k = 1,2,\ldots, K$. Denote the cellwise objective function as
    \begin{equation*}
        \begin{aligned}
            \mathcal{L}_{\mathrm{cell}}(\boldsymbol{\Omega}, \boldsymbol{W})
            &= \sum_{k=1}^K \sum_{i=1}^{n_k} \tilde{\ell} \left( \boldsymbol{\mu}^{(\boldsymbol{w}_{i}^{(k)}, k)}, \boldsymbol{\Sigma}^{(\boldsymbol{w}_{i}^{(k)}, k)}; \boldsymbol{z}_{i}^{(\boldsymbol{w}_{i}^{(k)}, k)} \right) \\
            &\quad + \sum_{k=1}^K \sum_{j=1}^q b_j^{(k)} \|\mathbf{1}_{n_k} - \boldsymbol{W}_{\cdot j}^{(k)}\|_0 + \lambda_1 \|\boldsymbol{\Theta}\|_{1, \mathrm{off}} + \lambda_2 \sum\limits_{k=1}^{K} \|\boldsymbol{\Gamma}^{(k)}\|_{1, \mathrm{off}},
        \end{aligned}
    \end{equation*}
    subject to $\lambda_{\min}(\boldsymbol{\Sigma}^{(k)}) \ge a_k > 0$, $\|\boldsymbol{W}_{\cdot j}^{(k)}\|_0 \ge h_k$, $b_j^{(k)}\ge0$ for all $j$ and $k$, $\lambda_1 > 0$, and $\lambda_2 > 0$. The proximal gradient step sizes are chosen with a strict descent margin, namely, for some constants $\delta_{\boldsymbol{\Gamma}},\delta_{\boldsymbol{\Theta}}\in(0,1)$, $\forall~t\ge 0$ and $k=1,2,\ldots,K$,
    \begin{equation*}
        0<\inf_{s \geq 0} \eta_{\boldsymbol{\Gamma},k}^{(s)}\leq\eta_{\boldsymbol{\Gamma},k}^{(t)}
        \le
        \frac{1-\delta_{\boldsymbol{\Gamma}}}{L_{\boldsymbol{\Gamma},k}},
        \quad
        0<\inf_{s \geq 0} \eta_{\boldsymbol{\Theta}}^{(s)}\leq\eta_{\boldsymbol{\Theta}}^{(t)}
        \le
        \frac{1-\delta_{\boldsymbol{\Theta}}}{L_{\boldsymbol{\Theta}}},
    \end{equation*}
    where $L_{\boldsymbol{\Gamma},k} > 0$, $L_{\boldsymbol{\Theta}} > 0$ are Lipschitz constants. Algorithm~\ref{alg:joint-cellwise-mcd} satisfies:
    \begin{enumerate}
        \item $\mathcal{L}_{\mathrm{cell}}(\boldsymbol{\Omega}^{(t+1)}, \boldsymbol{W}^{(t+1)}) \le \mathcal{L}_{\mathrm{cell}}(\boldsymbol{\Omega}^{(t)}, \boldsymbol{W}^{(t+1)}) \le \mathcal{L}_{\mathrm{cell}}(\boldsymbol{\Omega}^{(t)}, \boldsymbol{W}^{(t)})$;
        \item The sequence $\mathcal{L}_{\mathrm{cell}}(\boldsymbol{\Omega}^{(t)}, \boldsymbol{W}^{(t)})$ converges to a finite value; 
        \item Any interior limit point $(\boldsymbol{\Omega}^*, \boldsymbol{W}^*)$, namely one satisfying $\lambda_{\min}(\boldsymbol{\Theta}^{*} \odot \boldsymbol{\Gamma}^{(k)*}) > a_k>0, k = 1,2,\ldots,K$ is stationary with respect to the continuous variables for fixed $\boldsymbol{W}^*$, and $\boldsymbol{W}^*$ is coordinatewise optimal for the discrete cellwise weight update given $\boldsymbol{\Omega}^*$. 
    \end{enumerate}
\end{theorem}

Theorem~\ref{thm:jointedcellmcd_converge} establishes analogous convergence guarantees for the joint cellwise MCD estimator using the EM-based generalized procedure. The cellwise objective values decrease monotonically and converge to a finite limit. Moreover, any interior accumulation point satisfies the stationarity conditions for the continuous parameters with the cellwise weights fixed, while the limiting cellwise weight matrix is coordinatewise optimal for the discrete weight update. These results provide a theoretical justification for the stability of the proposed procedure under cellwise contamination and missing-entry patterns.

\subsection{Finite-sample breakdown property}
To provide context, the following theorems characterize the robustness of the proposed joint estimators under casewise and cellwise contamination. We establish finite-sample breakdown points, showing how much contamination the estimators can tolerate while maintaining bounded estimates. In order to discuss the breakdown point properties of the joint estimators, we introduce some definitions under the joint estimation framework. Let \(\boldsymbol{Z}^{\star}=\{\boldsymbol{Z}^{\star(1)},\ldots,\boldsymbol{Z}^{\star(K)}\}\), where \( \boldsymbol{Z}^{\star(k)}=(\boldsymbol{z}_{1}^{\star(k)},\ldots,\boldsymbol{z}_{n_k}^{\star(k)})^\top \in \mathbb{R}^{n_k\times q} \) denotes the clean representation matrix of the \(k\)-th positive class, \(k=1,2,\ldots,K\). Let \( \widehat{\boldsymbol{\Omega}} = (\widehat{\boldsymbol{\mu}}, \widehat{\boldsymbol{\Sigma}}) = (\{\widehat{\boldsymbol{\mu}}^{(k)}\}_{k=1}^K,\{\widehat{\boldsymbol{\Sigma}}^{(k)}\}_{k=1}^K) \) denote the joint estimator. For any vector $\mathbf{r} = (r_1,r_2,\ldots, r_K)$, define
\[
    \beta(\mathbf{r}) := \max_{1 \le k \le K} \frac{r_k}{n_k}.
\]

\begin{definition}[Joint Casewise Contamination]\label{def:jointfastmcd_contamination}
    Let $\mathcal{C}_{\mathrm{case}}(\mathbf{r})$ be the set of all contaminated samples $\tilde{\boldsymbol{Z}} = \{\tilde{\boldsymbol{Z}}^{(1)}, \ldots, \tilde{\boldsymbol{Z}}^{(K)}\}$, where for each group $k \in \{1,2,\ldots, K\}$, $\tilde{\boldsymbol{Z}}^{(k)}$ is obtained by replacing at most $r_k$ rows of $\boldsymbol{Z}^{\star(k)}$.
\end{definition}

\begin{definition}[Joint Cellwise Contamination]\label{def:jointcellmcd_contamination}
    Let $\mathcal{C}_{\mathrm{cell}}(\mathbf{r})$ be the set of all contaminated samples $\tilde{\boldsymbol{Z}}$, where for each group $k \in \{1,2,\ldots, K\}$ and each column $j \in \{1,2,\ldots, q\}$, $\tilde{\boldsymbol{Z}}^{(k)}$ is obtained by replacing at most $r_k$ cells in the $j$-th column of $\boldsymbol{Z}^{\star(k)}$.
\end{definition}

\begin{definition}[Finite-Sample Joint Breakdown Points]\label{def:jointmcd_breakdown}
    The joint location breakdown points are
    \begin{align*}
        \delta_{\mathrm{case},\boldsymbol{\mu}}^*(\widehat{\boldsymbol{\Omega}}, \boldsymbol{Z}^{\star}) &:= \inf \left\{ \beta(\mathbf{r}) : \sup_{\tilde{\boldsymbol{Z}} \in \mathcal{C}_{\mathrm{case}}(\mathbf{r})} \max_{1 \le k \le K} \|\widehat{\boldsymbol{\mu}}^{(k)}(\tilde{\boldsymbol{Z}}) - \widehat{\boldsymbol{\mu}}^{(k)}(\boldsymbol{Z}^{\star})\|_2 = \infty \right\}, \\
        \varepsilon_{\mathrm{cell},\boldsymbol{\mu}}^*(\widehat{\boldsymbol{\Omega}}, \boldsymbol{Z}^{\star}) &:= \inf \left\{ \beta(\mathbf{r}) : \sup_{\tilde{\boldsymbol{Z}} \in \mathcal{C}_{\mathrm{cell}}(\mathbf{r})} \max_{1 \le k \le K} \|\widehat{\boldsymbol{\mu}}^{(k)}(\tilde{\boldsymbol{Z}}) - \widehat{\boldsymbol{\mu}}^{(k)}(\boldsymbol{Z}^{\star})\|_2 = \infty \right\}.
    \end{align*}
    The joint scatter explosion and implosion breakdown points are
    \begin{align*}
        \delta_{\mathrm{case},+}(\widehat{\boldsymbol{\Omega}}, \boldsymbol{Z}^{\star}) &:= \inf \left\{ \beta(\mathbf{r}) : \sup_{\tilde{\boldsymbol{Z}} \in \mathcal{C}_{\mathrm{case}}(\mathbf{r})} \max_{1 \le k \le K} \lambda_{\max}(\widehat{\boldsymbol{\Sigma}}^{(k)}(\tilde{\boldsymbol{Z}})) = \infty \right\}, \\
        \varepsilon_{\mathrm{cell},+}(\widehat{\boldsymbol{\Omega}}, \boldsymbol{Z}^{\star}) &:= \inf \left\{ \beta(\mathbf{r}) : \sup_{\tilde{\boldsymbol{Z}} \in \mathcal{C}_{\mathrm{cell}}(\mathbf{r})} \max_{1 \le k \le K} \lambda_{\max}(\widehat{\boldsymbol{\Sigma}}^{(k)}(\tilde{\boldsymbol{Z}})) = \infty \right\}, \\
        \delta_{\mathrm{case},-}(\widehat{\boldsymbol{\Omega}}, \boldsymbol{Z}^{\star}) &:= \inf \left\{ \beta(\mathbf{r}) : \inf_{\tilde{\boldsymbol{Z}} \in \mathcal{C}_{\mathrm{case}}(\mathbf{r})} \min_{1 \le k \le K} \lambda_{\min}(\widehat{\boldsymbol{\Sigma}}^{(k)}(\tilde{\boldsymbol{Z}})) = 0 \right\}, \\
        \varepsilon_{\mathrm{cell},-}(\widehat{\boldsymbol{\Omega}}, \boldsymbol{Z}^{\star}) &:= \inf \left\{ \beta(\mathbf{r}) : \inf_{\tilde{\boldsymbol{Z}} \in \mathcal{C}_{\mathrm{cell}}(\mathbf{r})} \min_{1 \le k \le K} \lambda_{\min}(\widehat{\boldsymbol{\Sigma}}^{(k)}(\tilde{\boldsymbol{Z}})) = 0 \right\}.
    \end{align*}
    If no admissible contamination fraction in $[0,1]$ causes breakdown, the breakdown fraction is reported as 1.
\end{definition}

The above definition is appropriate because, within the joint estimation framework, the estimator comprises the location and scatter parameters of multiple positive classes. Unbounded shift, covariance explosion, or covariance degradation in the estimation of any single class will lead to the failure of the overall joint estimation. Therefore, we measure the joint contamination intensity by the maximum contamination ratio across all classes and define breakdown based on the worst-case scenario. When $K=1$, these definitions reduce to the classic casewise and cellwise breakdown definitions, maintaining consistency with the existing theoretical frameworks of MCD and cellMCD. Based on these definitions, we establish the breakdown point properties of the joint estimators. We assume the dataset is in general position: any $(q-1)$-dimensional affine subspace contains at most $q$ points.

\begin{theorem}[Joint Cellwise Breakdown Point]\label{thm:jointcellmcd_breakdown_point}
    Consider the joint cellwise estimator minimizing the joint objective subject to $\lambda_{\min}(\boldsymbol{\Sigma}^{(k)}) \ge a_k > 0$, $\|\boldsymbol{W}_{\cdot j}^{(k)}\|_0 \ge h_k$, and a non-negative penalty
    \[
        \mathbf{P}(\boldsymbol{\Theta}, \{\boldsymbol{\Gamma}^{(k)}\}_{k=1}^{K}) = \lambda_1 \sum_{i \ne j} |\theta_{ij}| + \lambda_2 \sum_{k=1}^K \sum_{i \ne j} |\gamma_{ij}^{(k)}|.
    \]
    Assume each clean group $\boldsymbol{Z}^{\star(k)}$ is in general position and
    \[
        h_k \ge \left\lfloor \frac{n_k}{2} \right\rfloor + 1, \quad k=1,2,\ldots,K.
    \]
    Let
    \[
        \beta_0 := \min_{1 \le k \le K} \frac{n_k - h_k + 1}{n_k}.
    \]
    Then,
    \[
        \varepsilon_{\mathrm{cell},-}(\widehat{\boldsymbol{\Omega}}, \boldsymbol{Z}^{\star}) = 1, \quad
        \varepsilon_{\mathrm{cell},+}(\widehat{\boldsymbol{\Omega}}, \boldsymbol{Z}^{\star}) \ge \beta_0, \quad
        \varepsilon_{\mathrm{cell},\boldsymbol{\mu}}^*(\widehat{\boldsymbol{\Omega}}, \boldsymbol{Z}^{\star}) \ge \beta_0.
    \]
    If $n_k = n$ and $h_k = h$ for all groups, these simplify to
    \[
        \varepsilon_{\mathrm{cell},-} = 1, \quad \varepsilon_{\mathrm{cell},+} \ge \frac{n - h + 1}{n}, \quad \varepsilon_{\mathrm{cell},\boldsymbol{\mu}}^* \ge \frac{n - h + 1}{n}.
    \]
\end{theorem}

Theorem~\ref{thm:jointcellmcd_breakdown_point} indicates that the joint cellwise MCD estimator is highly robust: it can tolerate a large proportion of contaminated cells without the location or covariance estimates becoming unbounded. The lower bound $\beta_0$ is sharp as a common guarantee for simultaneous boundedness of the location and scatter estimates under the stated contamination construction.

\begin{theorem}[Joint Casewise Breakdown Point]\label{thm:jointfastmcd_breakdown_point}
    Suppose each clean group is in general position and $h_k\ge \lfloor n_k/2\rfloor+1$, $k=1,2,\ldots,K$. For the penalized joint casewise estimator with $\lambda_{\min}(\boldsymbol{\Sigma}^{(k)}) \ge a_k > 0$ and $\mathbf{P}(\boldsymbol{\Theta}, \{\boldsymbol{\Gamma}^{(k)}\}_{k=1}^{K}) \ge 0$, let
    \[
        \beta_0 := \min_{1 \le k \le K} \frac{n_k - h_k + 1}{n_k}.
    \]
    Then,
    \[
        \delta_{\mathrm{case},-}(\widehat{\boldsymbol{\Omega}}, \boldsymbol{Z}^{\star}) = 1, \quad
        \delta_{\mathrm{case},+}(\widehat{\boldsymbol{\Omega}}, \boldsymbol{Z}^{\star}) \ge \beta_0, \quad
        \delta_{\mathrm{case},\boldsymbol{\mu}}^*(\widehat{\boldsymbol{\Omega}}, \boldsymbol{Z}^{\star}) \ge \beta_0.
    \]
    If $n_k = n$ and $h_k = h$ for all groups,
    \[
        \delta_{\mathrm{case},-} = 1, \quad \delta_{\mathrm{case},+} \ge \frac{n - h + 1}{n}, \quad \delta_{\mathrm{case},\boldsymbol{\mu}}^* \ge \frac{n - h + 1}{n}.
    \]
    This establishes a provable high breakdown lower bound for the penalized joint casewise estimator.
\end{theorem}

Theorem~\ref{thm:jointfastmcd_breakdown_point} similarly establishes that the joint casewise MCD estimator maintains bounded estimates even under substantial casewise contamination. The derived lower bounds on $\delta_{\mathrm{case},-}$, $\delta_{\mathrm{case},+}$, and $\delta_{\mathrm{case},\boldsymbol{\mu}}^*$ provide a clear, provable measure of the estimator’s finite-sample robustness.

\section{Real data applications} \label{sec6}
We evaluate our unified MDS-based framework on four real-world AI-related tasks: LLM-generated text, watermark, hallucination, and adversarial example detection. Performance is assessed using comprehensive metrics, including area under the receiver operating characteristic curve (ROC AUC), area under the precision-recall curve (PR AUC), F1 score, specificity, accuracy, precision, and recall. Our implementation is available at \url{https://github.com/Astringency/JointMCD.git}.

\subsection{LLM-generated text detection}

LLM-generated text detection aims to distinguish human-written passages from LLM-generated ones. We use the Human ChatGPT Comparison Corpus (HC3), a benchmark dataset widely adopted for AI-generated text detection, which contains paired human and ChatGPT-3.5 responses across multiple domains \citep{guo2023close,ouyang2022training,openai2022chatgpt}. We sample 1,000 instances from each domain, except for the ``wiki\_csai'' subset, where all 842 available samples are used. To assess cross-generator generalization, we further collect responses to the same questions from five recent LLMs: ChatGPT-5.4 \citep{openai2026gpt54T,openai2026gpt54}, ChatGPT-5.4 mini \citep{openai2026gpt54mininano}, Gemini-3 Flash Preview \citep{gemini2024gemini15,google2026gemini3flashpreview}, Gemini-3.1 Flash Lite Preview \citep{google2026gemini31flashlitepreview}, and Gemini-3.1 Pro Preview \citep{google2026gemini31propreview}. This yields five HC3-based datasets, each with 4,842 instances. We apply the proposed unified detection framework using RoBERTa-base as the pretrained encoder \citep{liu2019roberta}, and compare it with likelihood-based \citep{solaiman2019release}, rank-based \citep{gehrmann2019gltr}, entropy-based \citep{ippolito2020automatic}, DetectGPT \citep{mitchell2023detectgpt}, and RoBERTa baseline methods. For DetectGPT, we report both the raw perturbation discrepancy score, DetectGPT-d, and its normalized version, DetectGPT-z.

\begin{table}[htbp]
	\centering
	\small
	\renewcommand{\arraystretch}{0.8}
	\setlength{\tabcolsep}{5pt}
	\caption{Performance comparison for LLM-generated text detection. Best results are shown in bold, and second-best results are marked with $\dagger$. All metrics are computed with respect to the LLM-generated samples.}
	\label{tab:detect_LGC}
	\begin{tabular}{@{}c|c|ccccccc|c@{}}
		\hline
		LLM & Detector & \makecell{ROC\\AUC} & \makecell{PR\\AUC} & F1 & Specificity & Accuracy & Precision & Recall & Mean \\
		\hline
		\multirow{10}{*}{\makecell[c]{GPT-5.4}} 
		& Entropy       & 0.022 & 0.309 & 0.667 & 0.000 & 0.500 & 0.500 & 1.000 & 0.428  \\
		& Likelihood    & 0.956 & 0.937 & 0.928 & 0.897 & 0.926 & 0.902 & 0.955 & 0.929 \\
		& LogRank       & 0.954 & 0.934 & 0.932 & 0.894 & 0.930 & 0.901 & 0.966 & 0.930 \\
		& Rank          & 0.881 & 0.843 & 0.840 & 0.792 & 0.833 & 0.808 & 0.874 & 0.839 \\
		& RoBERTa-base  & 0.231 & 0.351 & 0.667 & 0.000 & 0.500 & 0.500 & 1.000 & 0.464 \\
		& RoBERTa-large & 0.235 & 0.353 & 0.667 & 0.000 & 0.500 & 0.500 & 1.000 & 0.465 \\
		& DetectGPT-d   & 0.974 & 0.976 & 0.928 & 0.907 & 0.926 & 0.911 & 0.945 & \competitive{0.938} \\
		& DetectGPT-z   & 0.958 & 0.949 & 0.921 & 0.893 & 0.919 & 0.898 & 0.945 & 0.926 \\
		& JCASEMCD      & 0.969 & 0.970 & 0.915 & 0.973 & 0.920 & 0.969 & 0.867 & \best{0.941} \\
		& JCELLMCD  & 0.963 & 0.957 & 0.913 & 0.905 & 0.912 & 0.906 & 0.919 & 0.925 \\
		\hline
		\multirow{10}{*}{\makecell[c]{GPT-5.4\\mini}} 
		& Entropy       & 0.027 & 0.309 & 0.667 & 0.000 & 0.500 & 0.500 & 1.000 & 0.429 \\
		& Likelihood    & 0.958 & 0.937 & 0.929 & 0.886 & 0.926 & 0.895 & 0.967 & \competitive{0.928} \\
		& LogRank       & 0.954 & 0.934 & 0.929 & 0.893 & 0.926 & 0.900 & 0.960 & \competitive{0.928} \\
		& Rank          & 0.854 & 0.825 & 0.799 & 0.746 & 0.790 & 0.767 & 0.835 & 0.802 \\
		& RoBERTa-base  & 0.251 & 0.357 & 0.667 & 0.000 & 0.500 & 0.500 & 1.000 & 0.468 \\
		& RoBERTa-large & 0.264 & 0.363 & 0.667 & 0.000 & 0.500 & 0.500 & 1.000 & 0.471 \\
		& DetectGPT-d   & 0.967 & 0.970 & 0.919 & 0.904 & 0.918 & 0.906 & 0.932 & \best{0.931} \\
		& DetectGPT-z   & 0.952 & 0.940 & 0.916 & 0.885 & 0.914 & 0.891 & 0.943 & 0.920 \\
		& JCASEMCD      & 0.951 & 0.955 & 0.892 & 0.925 & 0.895 & 0.920 & 0.866 & 0.915 \\
		& JCELLMCD  & 0.945 & 0.939 & 0.887 & 0.873 & 0.885 & 0.876 & 0.898 & 0.900 \\
		\hline
		\multirow{10}{*}{\makecell[c]{Gemini-3.1\\flash-lite\\preview}} 
		& Entropy       & 0.177 & 0.338 & 0.667 & 0.000 & 0.500 & 0.500 & 1.000 & 0.454 \\
		& Likelihood    & 0.838 & 0.845 & 0.790 & 0.838 & 0.798 & 0.824 & 0.758 & \competitive{0.813} \\
		& LogRank       & 0.848 & 0.855 & 0.798 & 0.867 & 0.809 & 0.850 & 0.752 & \best{0.825} \\
		& Rank          & 0.860 & 0.832 & 0.804 & 0.751 & 0.795 & 0.771 & 0.840 & 0.808 \\
		& RoBERTa-base  & 0.358 & 0.410 & 0.669 & 0.012 & 0.505 & 0.503 & 0.999 & 0.494 \\
		& RoBERTa-large & 0.295 & 0.372 & 0.667 & 0.000 & 0.500 & 0.500 & 1.000 & 0.476 \\
		& DetectGPT-d   & 0.714 & 0.807 & 0.713 & 0.878 & 0.750 & 0.836 & 0.622 & 0.760 \\
		& DetectGPT-z   & 0.719 & 0.804 & 0.720 & 0.860 & 0.751 & 0.821 & 0.641 & 0.759 \\
		& JCASEMCD      & 0.798 & 0.770 & 0.737 & 0.571 & 0.702 & 0.660 & 0.834 & 0.725 \\
		& JCELLMCD  & 0.737 & 0.688 & 0.700 & 0.584 & 0.673 & 0.647 & 0.762 & 0.685 \\
		\hline
		\multirow{10}{*}{\makecell[c]{Gemini-3\\flash\\preview}} 
		& Entropy       & 0.180 & 0.338 & 0.667 & 0.000 & 0.500 & 0.500 & 1.000 & 0.455 \\
		& Likelihood    & 0.823 & 0.832 & 0.776 & 0.814 & 0.783 & 0.802 & 0.753 & \competitive{0.798} \\
		& LogRank       & 0.834 & 0.843 & 0.787 & 0.814 & 0.792 & 0.805 & 0.770 & \best{0.806} \\
		& Rank          & 0.853 & 0.828 & 0.787 & 0.751 & 0.780 & 0.765 & 0.810 & 0.796 \\
		& RoBERTa-base  & 0.344 & 0.401 & 0.668 & 0.007 & 0.504 & 0.502 & 1.000 & 0.489 \\
		& RoBERTa-large & 0.274 & 0.364 & 0.667 & 0.000 & 0.500 & 0.500 & 1.000 & 0.472 \\
		& DetectGPT-d   & 0.690 & 0.785 & 0.691 & 0.803 & 0.717 & 0.762 & 0.632 & 0.726 \\
		& DetectGPT-z   & 0.691 & 0.776 & 0.686 & 0.827 & 0.720 & 0.779 & 0.613 & 0.727 \\
		& JCASEMCD      & 0.823 & 0.796 & 0.768 & 0.628 & 0.742 & 0.697 & 0.855 & 0.758 \\
		& JCELLMCD  & 0.763 & 0.715 & 0.717 & 0.616 & 0.695 & 0.668 & 0.773 & 0.707 \\
		\hline
		\multirow{10}{*}{\makecell[c]{Gemini-3.1\\pro\\preview}} 
		& Entropy       & 0.287 & 0.377 & 0.667 & 0.000 & 0.500 & 0.500 & 1.000 & 0.476 \\
		& Likelihood    & 0.698 & 0.722 & 0.669 & 0.015 & 0.507 & 0.503 & 0.998 & 0.587 \\
		& LogRank       & 0.740 & 0.707 & 0.699 & 0.476 & 0.647 & 0.610 & 0.818 & 0.671 \\
		& Rank          & 0.710 & 0.732 & 0.682 & 0.734 & 0.694 & 0.711 & 0.654 & \competitive{0.702} \\
		& RoBERTa-base  & 0.275 & 0.365 & 0.667 & 0.007 & 0.502 & 0.501 & 0.998 & 0.474 \\
		& RoBERTa-large & 0.179 & 0.338 & 0.667 & 0.000 & 0.500 & 0.500 & 1.000 & 0.455 \\
		& DetectGPT-d   & 0.518 & 0.666 & 0.667 & 0.001 & 0.501 & 0.500 & 1.000 & 0.550 \\
		& DetectGPT-z   & 0.506 & 0.653 & 0.667 & 0.000 & 0.500 & 0.500 & 1.000 & 0.547 \\
		& JCASEMCD      & 0.813 & 0.781 & 0.754 & 0.606 & 0.725 & 0.682 & 0.844 & \best{0.744} \\
		& JCELLMCD  & 0.748 & 0.695 & 0.717 & 0.600 & 0.691 & 0.662 & 0.783 & 0.699 \\
		\hline
	\end{tabular}
\end{table}

The results in Table~\ref{tab:detect_LGC} demonstrate the effectiveness and robustness of the proposed methods across different LLM generators: JCASEMCD achieves ROC AUC values of 0.969 and 0.951 on GPT-5.4 and GPT-5.4 mini, respectively, comparable to the strongest DetectGPT variants while yielding the highest specificity and precision, and remains competitive on the more challenging Gemini-based datasets, achieving the best overall performance on Gemini-3.1 Pro Preview in terms of ROC AUC, PR AUC, F1 score, and accuracy.

\subsection{Watermark detection}

Watermark detection can be viewed as a special case of LLM-generated text detection, where the generated text may contain deliberately embedded statistical traces. When the watermarking rule of the generating model is known, such traces can be detected by hypothesis-testing procedures specifically designed for the corresponding watermarking mechanism \citep{li2025statistical}. In contrast, the proposed framework does not require explicit knowledge of the watermarking rule, but instead detects distributional deviations in the space of deep representations.

We conduct watermark detection experiments based on Facebook OPT-1.3B \citep{zhang2022opt}. Three types of generated text are considered: non-watermarked text, denoted by Null; text generated with Gumbel-Max watermarking, denoted by Gumbel; and text generated with inverse-transform watermarking, denoted by Inverse. We consider two types of detection tasks: distinguishing human-written text from watermarked text, and distinguishing non-watermarked generated text from watermarked generated text. In these tasks, we denote unwatermarked text as positive samples and watermarked text as negative samples. The proposed method is compared with the hypothesis-testing framework of \citet{li2025statistical}, which assumes access to the underlying watermarking mechanism. The results are reported in Table~\ref{tab:watermark_all_metrics}.

\begin{table}[htbp]
    \centering
    \small
    \caption{Performance comparison across different watermark detection tasks.
    Best results are shown in bold, and second-best results are marked with $\dagger$. All metrics are computed with respect to the watermarked samples.}
    \label{tab:watermark_all_metrics}
    \begin{tabular}{c|c|ccccc}
        \hline
        Task & Method & F1 & Specificity & Accuracy & Precision & Recall \\
        \hline
        \multirow{3}{*}{\makecell[c]{Human\\vs\\Gumbel}}
             & Gumbel-Test    & \best{0.630} & \best{0.947} & \best{0.716} & \best{0.901} & 0.485 \\
             & JCASEMCD       & 0.575 & \competitive{0.607} & 0.585 & \competitive{0.589} & \competitive{0.562} \\
             & JCELLMCD       & \competitive{0.591} & 0.576 & \competitive{0.586} & 0.585 & \best{0.597} \\
             \hline
        \multirow{3}{*}{\makecell[c]{Human\\vs\\Inverse}}
             & Inverse-Test   & \best{0.625} & \best{0.948} & \best{0.713} & \best{0.902} & 0.479 \\
             & JCASEMCD       & 0.575 & \competitive{0.606} & \competitive{0.584} & \competitive{0.588} & \competitive{0.562} \\
             & JCELLMCD       & \competitive{0.589} & 0.571 & \competitive{0.584} & 0.582 & \best{0.597} \\
             \hline
        \multirow{3}{*}{\makecell[c]{Null\\vs\\Gumbel}}
             & Gumbel-Test    & \best{0.629} & \best{0.943} & \best{0.714} & \best{0.895} & 0.485 \\
             & JCASEMCD       & \competitive{0.590} & 0.597 & \competitive{0.592} & \competitive{0.593} & \best{0.586} \\
             & JCELLMCD       & 0.585 & \competitive{0.601} & 0.589 & 0.592 & \competitive{0.578} \\
             \hline
        \multirow{3}{*}{\makecell[c]{Null\\vs\\Inverse}}
             & Inverse-Test   & \best{0.624} & \best{0.945} & \best{0.712} & \best{0.898} & 0.479 \\
             & JCASEMCD       & \competitive{0.590} & 0.590 & \competitive{0.590} & \competitive{0.590} & \best{0.590} \\
             & JCELLMCD       & 0.581 & \competitive{0.596} & 0.585 & 0.587 & \competitive{0.574} \\
             \hline
    \end{tabular}
\end{table}

The results in Table~\ref{tab:watermark_all_metrics} show that the hypothesis-testing baselines of \citet{li2025statistical} perform strongly when the watermarking rule is known and matched to the test procedure, achieving the best F1 score, accuracy, precision, and specificity. This is expected because these methods are tailored to the known watermarking mechanism and directly exploit its statistical signal. In contrast, the proposed method is mechanism-agnostic and does not rely on explicit watermarking rules. Although its overall performance is generally lower, it provides a more balanced classification behavior: the testing-based methods attain very high precision and specificity but relatively low recall, indicating conservative detection of the negative class, namely watermarked text, whereas the proposed method improves recall while maintaining moderate precision, leading to a more balanced precision--recall trade-off.

This distinction becomes more important when the watermarking mechanism is unknown or the generated text is watermark-free. In this setting, we treat human-written text as positive samples and watermark-free LLM-generated text as negative samples. Since the testing framework of \citet{li2025statistical} depends on knowledge of the watermarking rule, its effectiveness may degrade substantially outside the matched watermark setting, as shown in Table~\ref{tab:watermark_human_vs_null}. These results indicate that specialized testing procedures are preferable when the watermarking rule is fully known, whereas the proposed framework offers a more general, task-agnostic detection strategy.

\begin{table}[htbp]
    \centering
    \small
    \caption{Performance comparison for detecting human-generated text and OPT-generated text without watermark. Best results are shown in bold, and second-best results are marked with $\dagger$. All metrics are computed with respect to the OPT-generated samples.}
    \label{tab:watermark_human_vs_null}
    \begin{tabular}{@{}c|ccccc@{}}
        \hline
        {Method} & {F1} & {Specificity} & {Accuracy} & {Precision} & {Recall} \\
        \hline
        Gumbel-Test  & 0.196 & \competitive{0.930} & 0.523 & 0.624 & 0.116 \\
        Inverse-Test & 0.099 & \best{0.948} & 0.501 & 0.511 & 0.055 \\
        JCASEMCD     & \best{0.788} & 0.728 & \best{0.777} & \best{0.752} & \best{0.827} \\
        JCELLMCD     & \competitive{0.735} & 0.662 & \competitive{0.720} & \competitive{0.697} & \competitive{0.778} \\
        \hline
    \end{tabular}
\end{table}

\subsection{Hallucination detection}

Hallucination detection aims to determine whether generated text contains factually incorrect or internally inconsistent content. Compared with LLM-generated text detection, hallucination detection is more directly concerned with factual reliability and logical consistency. Hallucination detection can be conducted in either a reference-free or reference-based setting, depending on whether external evidence or reference answers are available. In both settings, correct facts are treated as positive samples and hallucinatory text as negative samples. We consider two benchmark datasets: HaluEval \citep{li2023halueval}, a large-scale benchmark used for reference-free hallucination detection, and the biography-domain ``wiki\_bio\_gpt3\_hallucination'' dataset \citep{manakul2023selfcheckgpt}, which contains 1,908 sentence-level samples with Wikipedia evidence and hallucination labels, enabling reference-based evaluation.

\begin{table}[H]
    \centering 
    \caption{Performance comparison for hallucination detection. Best results are shown in bold, and second-best results are marked with $\dagger$. All metrics are computed with respect to the hallucinated text.}
    \label{tab:halu_eval}
    \small
    \begin{tabular}{@{}c|ccccccc@{}}
        \hline
        {Method} & \makecell{ROC\\AUC} & \makecell{PR\\AUC} & {F1} & {Specificity} & {Accuracy} & {Precision} & {Recall}\\
        \hline
        DeBERTa-v3-large-1 & 0.623 & 0.707 & \best{0.734} & 0.325 & 0.631 & \competitive{0.665} & \best{0.818}\\
        DeBERTa-v3-large-2 & \competitive{0.644} & \competitive{0.724} & \best{0.734} & 0.338 & \competitive{0.634} & \best{0.668} & \competitive{0.814}\\
        \makecell{JCASEMCD(w/o ref)} & 0.608 & 0.489 & 0.568 & \competitive{0.518} & 0.576 & 0.501 & 0.656\\
        \makecell{JCELLMCD(w/o ref)} & \best{0.751} & \best{0.734} & 0.642 & \best{0.577} & \best{0.648} & 0.565 & 0.744\\
        \makecell{JCASEMCD(w ref)} & 0.554 & 0.640 & \competitive{0.663} & 0.441 & 0.581 & 0.661 & 0.666\\
        \makecell{JCELLMCD(w ref)} & 0.557 & 0.653 & 0.652 & 0.447 & 0.571 & 0.657 & 0.648 \\
        \hline
    \end{tabular}
\end{table}

As baselines, we use two DeBERTa-based NLI models, DeBERTa-v3-large-mnli-fever-anli-ling-wanli-binary and DeBERTa-v3-large-zeroshot-v1.1-all-33, denoted as DeBERTa-NLI and DeBERTa-ZS, respectively; both require reference information. The results in Table~\ref{tab:halu_eval} show that the reference-free JCELLMCD achieves the best overall performance, with ROC AUC and PR AUC values of 0.751 and 0.734, respectively, outperforming the DeBERTa baselines by alleviating their low-specificity issue. JCELLMCD also consistently outperforms JCASEMCD, suggesting that the cellwise strategy is more effective in capturing localized hallucination anomalies.


\subsection{Adversarial examples detection}

Adversarial examples are deliberately perturbed inputs designed to mislead a model while remaining visually or semantically similar to benign samples. In this setting, normal images are positive samples, and adversarial images are negative samples. Since such perturbations can induce abnormal patterns in deep feature representations even when the input-space changes are small, adversarial example detection provides another natural setting for evaluating the proposed MDS-based framework. Existing detection methods include Local Intrinsic Dimensionality (LID), which detects adversarial samples by characterizing the intrinsic dimensionality of their local neighborhoods \citep{ma2018characterizing}, and the Mahalanobis distance-based detector proposed by \citet{lee2018simple}.

We study adversarial example detection in image classification tasks. On CIFAR-10 and CIFAR-100 \citep{krizhevsky2009learning}, adversarial examples are generated using four representative attack methods: FGSM \citep{goodfellow2015explaining}, BIM \citep{kurakin2018adversarial}, CWL2 \citep{carlini2017towards}, and DeepFool \citep{moosavi2016deepfool}. We also consider ImageNet \citep{deng2009imagenet} adversarial examples generated by Diff-PGD \citep{xue2023diffusion}. The target classifiers used to generate these adversarial examples are mainly ResNet \citep{he2016deep} and DenseNet \citep{huang2017densely} models. We use ResNet as the feature extractor to obtain deep representations for the proposed framework, and compare its performance with LID and the Mahalanobis distance-based method of \citet{lee2018simple}. The results are reported in Figure~\ref{fig:adv_heatmap_f1_spec_acc_prec_recall}. 

The proposed framework matches or exceeds baselines in simpler settings (e.g., CIFAR-10 under FGSM/BIM) by effectively capturing adversarial deviations. However, it underperforms the standard Mahalanobis baseline in complex scenarios (e.g., CIFAR-100 under CWL2/DeepFool), likely due to the instability of joint covariance estimation in high-dimensional, multi-class spaces. Thus, while effective for moderate-complexity tasks, it requires further stabilization or dimensionality reduction for more challenging settings.

\begin{figure}[htbp]
    \centering
    \includegraphics[width=\textwidth]{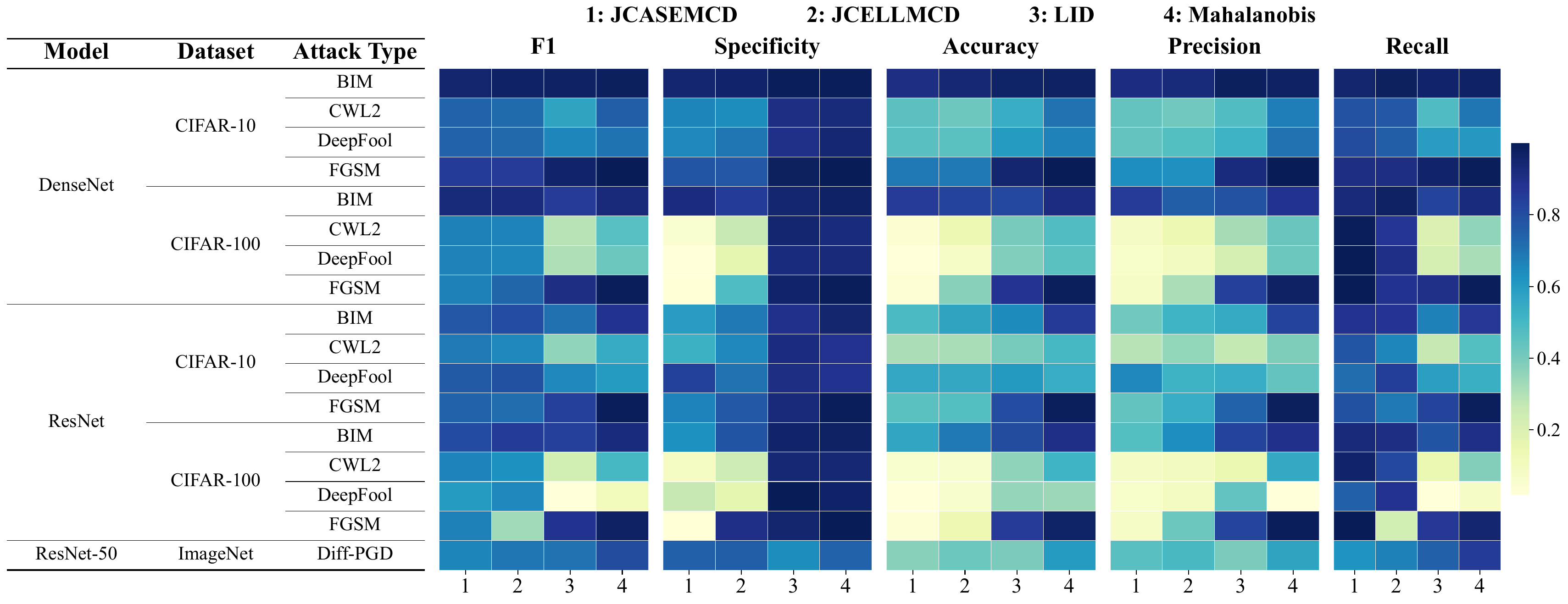}
    \caption{Heatmap of different metrics. All metrics are computed with respect to the adversarial examples.}
    \label{fig:adv_heatmap_f1_spec_acc_prec_recall}
\end{figure}


\section{Conclusion} \label{sec7}

This paper proposed a unified framework for detecting AI-related anomalous samples, including LLM-generated text, hallucinations, watermarked text, and adversarial examples. The framework combined pretrained deep representations, robust estimation of the mean and covariance structure of positive samples, and MDSs for detection. To handle multi-class positive samples with shared and class-specific structures, we developed joint casewise and cellwise MCD estimators and established the convergence of the corresponding optimization algorithms and the high-breakdown-point properties of the estimators. Experiments demonstrated the effectiveness and generality of the proposed framework. These results showed that combining deep representations, robust joint covariance estimation, and MDS provided a flexible and statistically principled approach for detecting diverse AI-related anomalous samples. Future work will consider relaxing the normality assumption, improving high-dimensional covariance estimation, and extending the framework to online and distribution-shift settings.

\newpage 
\begin{appendix}
    \section{Details of implementation and optimization}\label{app:details}

    \subsection{Initialization for JCASEMCD and JCELLMCD}\label{app:init}

    To initialize the estimation procedure, we employ a robust univariate filtering strategy to mitigate the adverse effects of severe cellwise outliers. Recall that $\boldsymbol{Z}^{(k)}$ is the $n_k \times q$ data matrix of the $k$-th group, $k=1,2,\ldots,K$. For each variable \(j \in \{1,2,\ldots,q\}\), we compute the coordinate-wise median \(\mathrm{med}_j^{(k)} = \mathrm{median} (\boldsymbol{z}_j^{(k)}) \) and the scaled median absolute deviation (MAD): $\mathrm{MAD}_j^{(k)} = 1.4826 \cdot \operatorname{median}_{1\le i\le n_k} |z_{ij}^{(k)}-\mathrm{med}_j^{(k)}|$, where the factor \(1.4826 = 1 / \Phi^{-1}(0.75)\) makes the scaled MAD equal to the standard deviation under a normal distribution. We then calculate the absolute MAD-scaled outlier score $|z_{ij}^{(k)} - \mathrm{med}_j^{(k)}| / \text{MAD}_j^{(k)}$. An observation cell $z_{ij}^{(k)}$ is flagged as a potential outlier and assigned a binary weight $w_{ij} = 0$ if its score exceeds a threshold; otherwise, $w_{ij} = 1$. Following the suggestion of \citet{iglewicz1993how}, we set the threshold to 3.5. Alternative threshold values, such as 2.5 and 3, can also be considered. To guarantee numerical stability, we enforce a minimum subset size requirement, denoted as $h_k$. If the number of unflagged cells in any given column falls below $h_k$, we retain the $h_k$ observations with the smallest robust distances and flag the rest. Subsequently, the flagged cellwise outliers are temporarily imputed using their respective column medians, yielding a robustly imputed data matrix $\boldsymbol{Z}^{\sharp(k)}$. 

    For JCELLMCD, the initial location vector $\widehat{\boldsymbol{\mu}}^{(k,0)}$ and scatter matrix $\widehat{\boldsymbol{\Sigma}}^{(k,0)}$ are computed as the empirical mean and empirical covariance matrix of $\boldsymbol{Z}^{\sharp(k)}$, respectively, followed by a regularization step to ensure $\widehat{\boldsymbol{\Sigma}}^{(k,0)}$ is strictly symmetric positive definite. For JCASEMCD, the above cellwise-cleaned estimates are used as preliminary robust estimates. Specifically, we compute the Mahalanobis distances of the original observations in $\boldsymbol{Z}^{(k)}$ using these preliminary estimates, retain the $h_k$ observations with the smallest distances, and then recompute $\widehat{\boldsymbol{\mu}}^{(k,0)}$ and $\widehat{\boldsymbol{\Sigma}}^{(k,0)}$ from the retained original observations, again followed by a regularization step to ensure symmetric positive definiteness. Finally, $\widehat{\boldsymbol{\Theta}}^{(0)}$ and $\widehat{\boldsymbol{\Gamma}}^{(k,0)}$ are generated through the Hadamard product decomposition, and we present a scheme in Appendix~\ref{app:decomp_indent} to guarantee the uniqueness of the decomposition. This procedure provides a highly resistant starting point for the subsequent iterative algorithms.  

    \subsection{Identifiability of covariance decomposition}\label{app:decomp_indent}

    Although the covariance decomposition is defined through the Hadamard product, such a decomposition is not unique. For example, multiplying the shared structure matrix by a positive constant while dividing the heterogeneous structure by the same constant leaves their Hadamard product unchanged. To remove this scaling ambiguity and ensure identifiability, we impose an RMS-correlation normalization rule on the decomposition.

    Specifically, let $\boldsymbol{\Sigma}^{(k)}=(\sigma_{ij}^{(k)})_{1\le i,j\le q}$, $k=1,2,\ldots,K$. For each off-diagonal entry $i\ne j$, we first define the group-specific correlation coefficient
    \[
    \rho_{ij}^{(k)}=\frac{\sigma_{ij}^{(k)}}{\sqrt{\sigma_{ii}^{(k)}\sigma_{jj}^{(k)}}},\quad k=1,2,\ldots,K.
    \]
    We then define the shared structure matrix $\boldsymbol{\Theta}$ by
    \[
    \theta_{ij}=\left\{\frac{1}{K}\sum_{k=1}^{K}\left(\rho_{ij}^{(k)}\right)^2\right\}^{1/2},\quad i\ne j,
    \]
    and set $\theta_{ii}=1$, $i=1,2,\ldots,q$. Under this rule, $\boldsymbol{\Theta}$ captures the common correlation strength across groups rather than the signed average correlation, and therefore avoids the cancellation of positive and negative correlations across groups.

    Given $\boldsymbol{\Theta}$, the heterogeneous structure matrices $\{\boldsymbol{\Gamma}^{(k)}\}_{k=1}^{K}$ are uniquely determined as follows. For the diagonal entries, we set
    \[
    \gamma_{ii}^{(k)}=\sigma_{ii}^{(k)},\quad i=1,2,\ldots,q,\quad k=1,2,\ldots,K.
    \]
    For the off-diagonal entries, if $\theta_{ij}>0$, we define
    \[
    \gamma_{ij}^{(k)}=\frac{\sigma_{ij}^{(k)}}{\theta_{ij}},\quad i\ne j,\quad k=1,2,\ldots,K.
    \]
    If $\theta_{ij}=0$, then by the definition of $\theta_{ij}$, we have $\rho_{ij}^{(k)}=0$ for all $k=1,2,\ldots,K$, and hence $\sigma_{ij}^{(k)}=0$ for all groups. In this case, we set $\gamma_{ij}^{(k)}=0$ for all $k$, which avoids division by zero while preserving the exact reconstruction. Therefore, the covariance matrices satisfy
    \[
    \boldsymbol{\Sigma}^{(k)}=\boldsymbol{\Theta}\odot\boldsymbol{\Gamma}^{(k)},\quad k=1,2,\ldots,K.
    \]
    This normalization fixes the scaling ambiguity of the Hadamard product decomposition and yields a unique decomposition. It also admits a natural interpretation: the shared structure $\boldsymbol{\Theta}$ captures the common dependence strength across groups, whereas the heterogeneous structures $\{\boldsymbol{\Gamma}^{(k)}\}_{k=1}^{K}$ retain group-specific scale, sign, and deviation information.

    \subsection{Gaussian random projection dimensionality reduction}\label{app:gaussian_reduction}

    Since the deep representations extracted by pretrained models are usually of high dimension, directly estimating the covariance matrix in the original representation space leads to two practical issues. First, when the feature dimension is large while the sample size is relatively limited, the sample covariance matrix and its robust alternatives may become numerically unstable or even nearly singular, which in turn affects the accuracy of Mahalanobis distance calculation as well as the stability of the decomposition into shared and heterogeneous structures in the joint estimation procedure. Second, the proposed joint casewise/cellwise MCD methods require repeated updates of the mean vectors, covariance matrices, weight matrices, shared parameters, and group-specific parameters across multiple groups. These steps involve matrix inversion, determinant evaluation, and iterative optimization. As the representation dimension increases, both the computational complexity and memory cost grow substantially, resulting in reducing efficiency during both training and detection. Therefore, in practical implementation, it is necessary to perform an appropriate dimension reduction on the deep representations, while preserving their structural information as much as possible, so as to improve the stability and computational efficiency of the subsequent robust joint estimation procedure.

    Motivated by these considerations, we employ Gaussian Random Projection as a preprocessing step for the deep representations in our algorithm implementation. Specifically, let the original deep representation be $\boldsymbol{z}_i \in \mathbb{R}^q$, where $q$ may be large. We generate a random matrix $\mathcal{R} \in \mathbb{R}^{d \times q}$ whose entries are independently drawn from the Gaussian distribution $\mathcal{N}(0,1/d)$, and project the original representation into a lower-dimensional space by $\mathcal{R} \boldsymbol{z}_i \in \mathbb{R}^d$, $d \ll q$, where $d$ is the dimension of the projected variable.

    All subsequent robust mean estimation, covariance estimation, and Mahalanobis distance computation are then carried out in the projected low-dimensional space. The use of Gaussian random projection has several advantages. First, it does not require additional training or complex modeling of the original representations, and can therefore be easily incorporated into the existing detection pipeline. Second, according to the Johnson--Lindenstrauss type results, random projection preserves the Euclidean geometric structure among samples with high probability, which helps maintain the relative distance relationships between normal and negative samples after dimension reduction. Third, by substantially reducing the dimension of the covariance matrix, it alleviates the numerical instability in matrix inversion and determinant computation, while also significantly reducing the computational burden of the joint optimization procedure. 

    \subsection{Optimization details for $\boldsymbol{\Theta}$ and $\boldsymbol{\Gamma}^{(k)}$}
    \label{app:theta_gamma_optimization}

    In this appendix, we provide the detailed numerical procedures for updating the parameters $\boldsymbol{\Theta}$ and $\boldsymbol{\Gamma}^{(k)}$ under $\ell_1$ regularization. These updates correspond to the parameter update steps described in Sections~\ref{sec5Optimization1} and~\ref{sec5Optimization2}, namely the Fast-MCD algorithm for the joint casewise MCD estimator and the EM-based procedure for the joint cellwise MCD estimator. The appendix includes the gradient derivations, proximal gradient updates, and the definition of the soft-thresholding operator used in the iterative optimization.

    When updating $\boldsymbol{\Theta}$ and $\boldsymbol{\Gamma}^{(k)}$, the algorithm requires solving optimization subproblems under an $\ell_1$ penalty. We employ the proximal gradient method \citep{beck2009fista} for these updates. Recall that the reparameterized form of the covariance matrix is
    \begin{equation*}
        \boldsymbol{\Sigma}^{(k)} = (\sigma_{ij}^{(k)})_{1\leq i,j\leq q}, \quad 
        \sigma_{ij}^{(k)} = 
        \begin{cases}
            \theta_{ij} \, \gamma_{ij}^{(k)}, & i \neq j, \\
            \gamma_{ii}^{(k)}, & i = j.
        \end{cases}
    \end{equation*}

    For the JCASEMCD, we solve the alternating optimization problem:
    \begin{itemize}
        \item Fix $\widehat{\boldsymbol{\Theta}}^{(t)}$ and update each $\widehat{\boldsymbol{\Gamma}}^{(k,t+1)}$ to decrease the objective for $k = 1,2,\ldots, K$:
            \begin{align*}
                \widehat{\boldsymbol{\Gamma}}^{(k,t+1)} &= \arg\min_{\boldsymbol{\Gamma}} \sum_{i=1}^{n_k} w_i^{(k,t)} \tilde{\ell} \left(\widehat{\boldsymbol{\mu}}^{(k,t+1)}, \widehat{\boldsymbol{\Theta}}^{(t)} \odot \boldsymbol{\Gamma};\boldsymbol{z}_i^{(k)}\right) + \lambda_2 \sum_{i \neq j} |\gamma_{ij}|.
            \end{align*}
        \item Fix all $\widehat{\boldsymbol{\Gamma}}^{(k,t+1)}$ and update the shared structure $\widehat{\boldsymbol{\Theta}}^{(t+1)}$:
            \begin{align*}
                \widehat{\boldsymbol{\Theta}}^{(t+1)} &= \arg\min_{\boldsymbol{\Theta}} \sum_{k=1}^{K} \sum_{i=1}^{n_k} w_i^{(k,t)} \tilde{\ell} \left(\widehat{\boldsymbol{\mu}}^{(k,t+1)}, \boldsymbol{\Theta} \odot \widehat{\boldsymbol{\Gamma}}^{(k,t+1)};\boldsymbol{z}_i^{(k)} \right) + \lambda_1 \sum_{i \neq j} |\theta_{ij}|.
            \end{align*}
    \end{itemize}

    Define
    \begin{equation*}
        \begin{aligned}
            S_{\mathrm{case}}^{(k,t)}
            &= \frac{1}{\sum_{i=1}^{n_k} w_i^{(k,t)}} \sum_{i=1}^{n_k} w_i^{(k,t)} (\boldsymbol{z}_i^{(k)} - \widehat{\boldsymbol{\mu}}^{(k,t+1)}) (\boldsymbol{z}_i^{(k)} - \widehat{\boldsymbol{\mu}}^{(k,t+1)})^\top \\
            &= \frac{1}{h_k} \sum_{i=1}^{n_k} w_i^{(k,t)} (\boldsymbol{z}_i^{(k)} - \widehat{\boldsymbol{\mu}}^{(k,t+1)}) (\boldsymbol{z}_i^{(k)} - \widehat{\boldsymbol{\mu}}^{(k,t+1)})^\top.
        \end{aligned}
    \end{equation*}

    For notational clarity, define
    \begin{equation*}
        \widehat{\boldsymbol{\Sigma}}_{\boldsymbol{\Gamma}}^{(k,t)}
        = \widehat{\boldsymbol{\Theta}}^{(t)} \odot \widehat{\boldsymbol{\Gamma}}^{(k,t)},
        \quad
        \widehat{\boldsymbol{\Sigma}}_{\boldsymbol{\Theta}}^{(k,t)}
        = \widehat{\boldsymbol{\Theta}}^{(t)} \odot \widehat{\boldsymbol{\Gamma}}^{(k,t+1)}.
    \end{equation*}

    The smooth parts of the subproblems are $Q_{\mathrm{case},\boldsymbol{\Gamma}}^{(t)} = \sum_{k=1}^K Q_{\mathrm{case},\boldsymbol{\Gamma}}^{(k,t)}$ and $Q_{\mathrm{case},\boldsymbol{\Theta}}^{(t)} = \sum_{k=1}^K Q_{\mathrm{case},\boldsymbol{\Theta}}^{(k,t)}$ with
    \[
        Q_{\mathrm{case},\boldsymbol{\Gamma}}^{(k,t)} = h_k \left[ \ln |\widehat{\boldsymbol{\Sigma}}_{\boldsymbol{\Gamma}}^{(k,t)}| + \mathrm{tr}((\widehat{\boldsymbol{\Sigma}}_{\boldsymbol{\Gamma}}^{(k,t)})^{-1} S_{\mathrm{case}}^{(k,t)}) \right],
    \]
    and
    \[
        Q_{\mathrm{case},\boldsymbol{\Theta}}^{(k,t)} = h_k \left[ \ln |\widehat{\boldsymbol{\Sigma}}_{\boldsymbol{\Theta}}^{(k,t)}| + \mathrm{tr}((\widehat{\boldsymbol{\Sigma}}_{\boldsymbol{\Theta}}^{(k,t)})^{-1} S_{\mathrm{case}}^{(k,t)}) \right].
    \]

    The non-smooth part, corresponding to $\ell_1$ regularization, is handled by proximal operators. From matrix calculus:
    \begin{equation*}
        \nabla_{\widehat{\boldsymbol{\Sigma}}_{\boldsymbol{\Gamma}}^{(k,t)}} Q_{\mathrm{case},\boldsymbol{\Gamma}}^{(k,t)} = h_k \left[ (\widehat{\boldsymbol{\Sigma}}_{\boldsymbol{\Gamma}}^{(k,t)})^{-1} - (\widehat{\boldsymbol{\Sigma}}_{\boldsymbol{\Gamma}}^{(k,t)})^{-1} S_{\mathrm{case}}^{(k,t)} (\widehat{\boldsymbol{\Sigma}}_{\boldsymbol{\Gamma}}^{(k,t)})^{-1} \right],
    \end{equation*}
    and
    \begin{equation*}
        \nabla_{\widehat{\boldsymbol{\Sigma}}_{\boldsymbol{\Theta}}^{(k,t)}} Q_{\mathrm{case},\boldsymbol{\Theta}}^{(k,t)} = h_k \left[ (\widehat{\boldsymbol{\Sigma}}_{\boldsymbol{\Theta}}^{(k,t)})^{-1} - (\widehat{\boldsymbol{\Sigma}}_{\boldsymbol{\Theta}}^{(k,t)})^{-1} S_{\mathrm{case}}^{(k,t)} (\widehat{\boldsymbol{\Sigma}}_{\boldsymbol{\Theta}}^{(k,t)})^{-1} \right].
    \end{equation*}

    Applying the chain rule, for $\widehat{\boldsymbol{\Gamma}}^{(k,t)}$:
    \begin{equation*}
        \left[\nabla_{\widehat{\boldsymbol{\Gamma}}^{(k,t)}} Q_{\mathrm{case},\boldsymbol{\Gamma}}^{(k,t)}\right]_{ij} =
        \begin{cases}
            \hat{\theta}_{ij}^{(t)} [\nabla_{\widehat{\boldsymbol{\Sigma}}_{\boldsymbol{\Gamma}}^{(k,t)}} Q_{\mathrm{case},\boldsymbol{\Gamma}}^{(k,t)} ]_{ij}, & i \neq j,\\
            [\nabla_{\widehat{\boldsymbol{\Sigma}}_{\boldsymbol{\Gamma}}^{(k,t)}} Q_{\mathrm{case},\boldsymbol{\Gamma}}^{(k,t)} ]_{ii}, & i=j.
        \end{cases}
    \end{equation*}

    At iteration $t$, we perform gradient descent:
    \begin{equation*}
        \tilde{\gamma}_{ij}^{(k, t+1)} = \hat{\gamma}_{ij}^{(k,t)} - \eta_{\boldsymbol{\Gamma},k}^{(t)} [\nabla_{\widehat{\boldsymbol{\Gamma}}^{(k,t)}} Q_{\mathrm{case},\boldsymbol{\Gamma}}^{(k,t)}]_{ij},
    \end{equation*}
    followed by the proximal update:
    \begin{equation*}
        \hat{\gamma}_{ij}^{(k,t+1)} = 
        \begin{cases}
            \operatorname{Soft}(\tilde{\gamma}_{ij}^{(k, t+1)}, \eta_{\boldsymbol{\Gamma},k}^{(t)} \lambda_2), & i \neq j,\\
            \tilde{\gamma}_{ii}^{(k, t+1)}, & i = j,
        \end{cases}
    \end{equation*}
    where \(\operatorname{Soft}(x, \rho) = \mathrm{sign}(x) \cdot \max\{|x| - \rho, 0\}.\) Similarly, for $\widehat{\boldsymbol{\Theta}}^{(t+1)}$:
    \begin{equation*}
        \hat{\theta}_{ij}^{(t+1)} = 
        \begin{cases}
            \operatorname{Soft}(\tilde{\theta}_{ij}^{(t+1)}, \eta_{\boldsymbol{\Theta}}^{(t)} \lambda_1), & i \neq j,\\
            1, & i=j,
        \end{cases}
    \end{equation*}
    with
    \begin{equation*}
        \begin{aligned}
            \tilde{\theta}_{ij}^{(t+1)} =& \hat{\theta}_{ij}^{(t)} - \eta_{\boldsymbol{\Theta}}^{(t)} [\nabla_{\widehat{\boldsymbol{\Theta}}^{(t)}} Q_{\mathrm{case},\boldsymbol{\Theta}}^{(t)}]_{ij},\\
            [\nabla_{\widehat{\boldsymbol{\Theta}}^{(t)}} Q_{\mathrm{case},\boldsymbol{\Theta}}^{(t)}]_{ij} =& 
        \begin{cases}
            \sum_{k=1}^K \hat{\gamma}_{ij}^{(k,t+1)} [\nabla_{\widehat{\boldsymbol{\Sigma}}_{\boldsymbol{\Theta}}^{(k,t)}} Q_{\mathrm{case},\boldsymbol{\Theta}}^{(k,t)}]_{ij}, & i \neq j,\\
            0, & i=j.
        \end{cases}
        \end{aligned}
    \end{equation*}

    For the JCELLMCD, the two subproblems are:
    \begin{itemize}
        \item Fix $\widehat{\boldsymbol{\Theta}}^{(t)}$ and $\widehat{\boldsymbol{\mu}}^{(k, t+1)}$, update $\widehat{\boldsymbol{\Gamma}}^{(k,t+1)}$:
            \begin{equation*}
                \widehat{\boldsymbol{\Gamma}}^{(k,t+1)} = \arg\min_{\boldsymbol{\Gamma}} n_k \left[ \ln |\widehat{\boldsymbol{\Theta}}^{(t)} \odot \boldsymbol{\Gamma}| + \mathrm{tr}((\widehat{\boldsymbol{\Theta}}^{(t)} \odot \boldsymbol{\Gamma})^{-1} S_{\mathrm{cell}}^{(k,t)}) \right] + \lambda_2 \sum_{i \neq j} |\gamma_{ij}|.
            \end{equation*}
        \item Fix $\{\widehat{\boldsymbol{\Gamma}}^{(k,t+1)}\}_{k=1}^{K}$ and $\widehat{\boldsymbol{\mu}}^{(k,t+1)}$, update $\widehat{\boldsymbol{\Theta}}^{(t+1)}$:
            \begin{equation*}
                \widehat{\boldsymbol{\Theta}}^{(t+1)} = \arg\min_{\boldsymbol{\Theta}} \sum_{k=1}^K n_k \left[ \ln |\boldsymbol{\Theta} \odot \widehat{\boldsymbol{\Gamma}}^{(k,t+1)}| + \mathrm{tr}((\boldsymbol{\Theta} \odot \widehat{\boldsymbol{\Gamma}}^{(k,t+1)})^{-1} S_{\mathrm{cell}}^{(k,t)}) \right] + \lambda_1 \sum_{i \neq j} |\theta_{ij}|.
            \end{equation*}
    \end{itemize}

    For notational clarity, define
    \begin{equation*}
        \widehat{\boldsymbol{\Sigma}}_{\boldsymbol{\Gamma}}^{(k,t)}
        = \widehat{\boldsymbol{\Theta}}^{(t)} \odot \widehat{\boldsymbol{\Gamma}}^{(k,t)},
        \quad
        \widehat{\boldsymbol{\Sigma}}_{\boldsymbol{\Theta}}^{(k,t)}
        = \widehat{\boldsymbol{\Theta}}^{(t)} \odot \widehat{\boldsymbol{\Gamma}}^{(k,t+1)}.
    \end{equation*}
    The $\ell_1$ regularization is handled via the proximal operator. The smooth gradient satisfies:
    \begin{equation*}
        \nabla_{\widehat{\boldsymbol{\Sigma}}_{\boldsymbol{\Gamma}}^{(k,t)}} Q_{\mathrm{smooth},\boldsymbol{\Gamma}}^{(k,t)} = n_k \left[ (\widehat{\boldsymbol{\Sigma}}_{\boldsymbol{\Gamma}}^{(k,t)})^{-1} - (\widehat{\boldsymbol{\Sigma}}_{\boldsymbol{\Gamma}}^{(k,t)})^{-1} S_{\mathrm{cell}}^{(k,t)} (\widehat{\boldsymbol{\Sigma}}_{\boldsymbol{\Gamma}}^{(k,t)})^{-1} \right],
    \end{equation*}
    and
    \begin{equation*}
        \nabla_{\widehat{\boldsymbol{\Sigma}}_{\boldsymbol{\Theta}}^{(k,t)}} Q_{\mathrm{smooth},\boldsymbol{\Theta}}^{(k,t)} = n_k \left[ (\widehat{\boldsymbol{\Sigma}}_{\boldsymbol{\Theta}}^{(k,t)})^{-1} - (\widehat{\boldsymbol{\Sigma}}_{\boldsymbol{\Theta}}^{(k,t)})^{-1} S_{\mathrm{cell}}^{(k,t)} (\widehat{\boldsymbol{\Sigma}}_{\boldsymbol{\Theta}}^{(k,t)})^{-1} \right],
    \end{equation*}
    with
    \begin{equation*}
        Q_{\mathrm{smooth},\boldsymbol{\Gamma}}^{(k,t)} = n_k \left[ \ln |\widehat{\boldsymbol{\Sigma}}_{\boldsymbol{\Gamma}}^{(k,t)}| + \mathrm{tr}((\widehat{\boldsymbol{\Sigma}}_{\boldsymbol{\Gamma}}^{(k,t)})^{-1} S_{\mathrm{cell}}^{(k,t)}) \right]
    \end{equation*}
    and
    \begin{equation*}
        Q_{\mathrm{smooth},\boldsymbol{\Theta}}^{(k,t)} = n_k \left[ \ln |\widehat{\boldsymbol{\Sigma}}_{\boldsymbol{\Theta}}^{(k,t)}| + \mathrm{tr}((\widehat{\boldsymbol{\Sigma}}_{\boldsymbol{\Theta}}^{(k,t)})^{-1} S_{\mathrm{cell}}^{(k,t)}) \right].
    \end{equation*}
    The chain rule gives
    \begin{equation*}
        [\nabla_{\widehat{\boldsymbol{\Gamma}}^{(k,t)}} Q_{\mathrm{smooth},\boldsymbol{\Gamma}}^{(k,t)}]_{ij} =
        \begin{cases}
            \hat{\theta}_{ij}^{(t)} [\nabla_{\widehat{\boldsymbol{\Sigma}}_{\boldsymbol{\Gamma}}^{(k,t)}} Q_{\mathrm{smooth},\boldsymbol{\Gamma}}^{(k,t)}]_{ij}, & i \neq j,\\
            [\nabla_{\widehat{\boldsymbol{\Sigma}}_{\boldsymbol{\Gamma}}^{(k,t)}} Q_{\mathrm{smooth},\boldsymbol{\Gamma}}^{(k,t)}]_{ii}, & i=j.
        \end{cases}
    \end{equation*}
    Gradient descent and proximal update:
    \begin{equation*}
        \tilde{\gamma}_{ij}^{(k,t+1)} = \hat{\gamma}_{ij}^{(k,t)} - \eta_{\boldsymbol{\Gamma},k}^{(t)} [\nabla_{\widehat{\boldsymbol{\Gamma}}^{(k,t)}} Q_{\mathrm{smooth},\boldsymbol{\Gamma}}^{(k,t)}]_{ij}, \quad
        \hat{\gamma}_{ij}^{(k,t+1)} =
        \begin{cases}
            \operatorname{Soft}(\tilde{\gamma}_{ij}^{(k,t+1)}, \eta_{\boldsymbol{\Gamma},k}^{(t)} \lambda_2), & i \neq j,\\
            \tilde{\gamma}_{ii}^{(k,t+1)}, & i=j.
        \end{cases} 
    \end{equation*}
    Update for $\widehat{\boldsymbol{\Theta}}^{(t+1)}$:
    \begin{equation*}
        \hat{\theta}_{ij}^{(t+1)} =
        \begin{cases}
            \operatorname{Soft}(\tilde{\theta}_{ij}^{(t+1)}, \eta_{\boldsymbol{\Theta}}^{(t)} \lambda_1), & i \neq j,\\
            1, & i=j,
        \end{cases}
        \quad
        \tilde{\theta}_{ij}^{(t+1)} = \hat{\theta}_{ij}^{(t)} - \eta_{\boldsymbol{\Theta}}^{(t)} [\nabla_{\widehat{\boldsymbol{\Theta}}^{(t)}} Q_{\mathrm{smooth},\boldsymbol{\Theta}}^{(t)}]_{ij},
    \end{equation*}
    \begin{equation*}
        [\nabla_{\widehat{\boldsymbol{\Theta}}^{(t)}} Q_{\mathrm{smooth},\boldsymbol{\Theta}}^{(t)}]_{ij} =
        \begin{cases}
            \sum_{k=1}^K \hat{\gamma}_{ij}^{(k,t+1)} [\nabla_{\widehat{\boldsymbol{\Sigma}}_{\boldsymbol{\Theta}}^{(k,t)}} Q_{\mathrm{smooth},\boldsymbol{\Theta}}^{(k,t)}]_{ij}, & i \neq j,\\
            0, & i=j.
        \end{cases}
    \end{equation*}

    To enforce the constraint $\lambda_{\min}(\widehat{\boldsymbol{\Sigma}}^{(k)}) \geq a_k > 0$, where $\widehat{\boldsymbol{\Sigma}}^{(k)}=\widehat{\boldsymbol{\Theta}} \odot \widehat{\boldsymbol{\Gamma}}^{(k)}$, a backtracking line search is employed to determine the step size $\eta_{\boldsymbol{\Theta}}^{(t)}$ and $\eta_{\boldsymbol{\Gamma},k}^{(t)}$ at each proximal gradient update. A candidate update is accepted only if the corresponding unprojected covariance matrices satisfy the minimum eigenvalue constraint and the penalized objective function is non-increasing; otherwise, the step size is reduced and the update is retried. In addition, to improve numerical stability, we use minimum eigenvalue clipping as a numerical safeguard when reconstructing covariance matrices and in linear algebra operations such as Cholesky factorization and log-determinant computation. 
    
    \section{Proofs for results}\label{app:proof}

    \subsection{Proof of Theorem~\ref{thm:jointedfastmcd_converge}}\label{apd:jointedfastmcd_converge}

    \begin{proof}
        First, according to the algorithm, the update of $\boldsymbol{\mu}$ corresponds to solving the maximum likelihood estimate, which minimizes the above objective function and therefore does not increase it. Specifically, fixing $\boldsymbol{H}^{(k,t)}$ and $\boldsymbol{\Sigma}^{(k,t)}$, where $\boldsymbol{\Sigma}^{(k,t)}=\boldsymbol{\Theta}^{(t)}\odot\boldsymbol{\Gamma}^{(k,t)}$, the terms related to $\boldsymbol{\mu}^{(k)}$ in the $k$-th group are
        \begin{equation*}
            \mathcal{L}_{\boldsymbol{\mu}^{(k)}} = \sum_{i=1}^{n_k} w_i^{(k,t)} (\boldsymbol{z}_i^{(k)} - \boldsymbol{\mu}^{(k)})^\top (\boldsymbol{\Sigma}^{(k,t)})^{-1} (\boldsymbol{z}_i^{(k)} - \boldsymbol{\mu}^{(k)}).
        \end{equation*}
        Taking the gradient with respect to $\boldsymbol{\mu}^{(k)}$ yields:
        \begin{equation*}
            \nabla_{\boldsymbol{\mu}^{(k)}} \mathcal{L}_{\boldsymbol{\mu}^{(k)}} = -2(\boldsymbol{\Sigma}^{(k,t)})^{-1} \sum_{i=1}^{n_k} w_i^{(k,t)} (\boldsymbol{z}_i^{(k)} - \boldsymbol{\mu}^{(k)}).
        \end{equation*}
        Setting the gradient to zero gives
        \begin{equation*}
            \sum_{i=1}^{n_k} w_i^{(k,t)} (\boldsymbol{z}_i^{(k)} - \boldsymbol{\mu}^{(k)}) = 0,
        \end{equation*}
        which implies
        \begin{equation*}
            \boldsymbol{\mu}^{(k,t+1)} = \frac{\sum_{i=1}^{n_k} w_i^{(k,t)} \boldsymbol{z}_i^{(k)}}{\sum_{i=1}^{n_k} w_i^{(k,t)}} = \frac{1}{h_k} \sum_{i=1}^{n_k} w_i^{(k,t)} \boldsymbol{z}_i^{(k)}.
        \end{equation*}
        Moreover, the Hessian matrix is
        \begin{equation*}
            \nabla_{\boldsymbol{\mu}^{(k)}}^2 \mathcal{L}_{\boldsymbol{\mu}^{(k)}} = 2 h_k (\boldsymbol{\Sigma}^{(k,t)})^{-1} \succ 0.
        \end{equation*}
        Since this is a strictly convex quadratic function, the above provides the unique global minimum. Therefore,
        \begin{equation*}
            \mathcal{L}_{\text{case}}(\boldsymbol{\mu}^{(t+1)}, \boldsymbol{\Sigma}^{(t)}, \boldsymbol{H}^{(t)}) \leq \mathcal{L}_{\text{case}}(\boldsymbol{\mu}^{(t)}, \boldsymbol{\Sigma}^{(t)}, \boldsymbol{H}^{(t)}),
        \end{equation*}
        where $\boldsymbol{\mu}^{(t+1)}=\{\boldsymbol{\mu}^{(k,t+1)}\}_{k=1}^K$, $\boldsymbol{\Sigma}^{(t)}=\{\boldsymbol{\Sigma}^{(k,t)}\}_{k=1}^K$ and $\boldsymbol{\Sigma}^{(k,t)}=\boldsymbol{\Theta}^{(t)}\odot\boldsymbol{\Gamma}^{(k,t)}$.

                For the update of $\boldsymbol{\Gamma}^{(k)}$, at iteration $t$, fixing $\boldsymbol{\mu}^{(k,t+1)}$, $\boldsymbol{H}^{(k,t)}$, and $\boldsymbol{\Theta}^{(t)}$, define the smooth part of the current subproblem as
        \begin{equation*}
            Q_{\mathrm{case},\boldsymbol{\Gamma}}^{(k,t)}(\boldsymbol{\Gamma}^{(k)})
            =
            Q_{\mathrm{case}}^{(k)}\left(\boldsymbol{\mu}^{(k,t+1)}, \boldsymbol{\Theta}^{(t)} \odot \boldsymbol{\Gamma}^{(k)}, \boldsymbol{H}^{(k,t)}\right).
        \end{equation*}
        The Hessian of this expression consists solely of matrix products involving $(\boldsymbol{\Sigma}^{(k)})^{-1}$ and the fixed sample covariance $S_{\mathrm{case}}^{(k,t)}$, given $\boldsymbol{\mu}^{(k,t+1)}$, $\boldsymbol{H}^{(k,t)}$ and $\boldsymbol{\Theta}^{(t)}$. Moreover, the relationship between $\boldsymbol{\Gamma}^{(k)}$ and $\boldsymbol{\Sigma}^{(k)}$ is linear, since $\boldsymbol{\Sigma}^{(k)} = \boldsymbol{\Theta}^{(t)} \odot \boldsymbol{\Gamma}^{(k)}$. Within the feasible region, the constraint $\lambda_{\min}(\boldsymbol{\Sigma}^{(k)}) \ge a > 0$ guarantees that the spectral norm of the inverse matrix is strictly bounded, i.e., $\|(\boldsymbol{\Sigma}^{(k)})^{-1}\|_2 \le 1/a$. Since $\|S^{(k,t)}_{\mathrm{case}}\|_2$ and the entries of $\boldsymbol{\Theta}^{(t)}$ are fixed for the current subproblem, all Hessian components are bounded locally. Therefore, there exists a finite local Lipschitz constant $L_{\boldsymbol{\Gamma}, k}^{(t)}$ such that $\nabla_{\boldsymbol{\Gamma}^{(k)}} Q_{\mathrm{case},\boldsymbol{\Gamma}}^{(k,t)}$ is Lipschitz continuous in the current subproblem (Lemma~1.2.2 in \citet{nesterov2004introductory}). Hence, for any $X, Y$ in the current feasible region,
        \begin{equation*}
            Q_{\mathrm{case},\boldsymbol{\Gamma}}^{(k,t)}(Y) \leq Q_{\mathrm{case},\boldsymbol{\Gamma}}^{(k,t)}(X) + \langle \nabla_{\boldsymbol{\Gamma}^{(k)}} Q_{\mathrm{case},\boldsymbol{\Gamma}}^{(k,t)}(X), Y - X \rangle + \frac{L_{\boldsymbol{\Gamma},k}^{(t)}}{2} \|Y - X\|_F^2.
        \end{equation*}
        Similarly, for the update of $\boldsymbol{\Theta}$, fixing $\boldsymbol{\mu}^{(k,t+1)}$, $\boldsymbol{H}^{(k,t)}$, and $\boldsymbol{\Gamma}^{(k,t+1)}$, define
        \begin{equation*}
            Q_{\mathrm{case},\boldsymbol{\Theta}}^{(t)}(\boldsymbol{\Theta})
            =
            \sum_{k=1}^{K}
            Q_{\mathrm{case}}^{(k)}\left(\boldsymbol{\mu}^{(k,t+1)}, \boldsymbol{\Theta} \odot \boldsymbol{\Gamma}^{(k,t+1)}, \boldsymbol{H}^{(k,t)}\right).
        \end{equation*}
        Since the mapping $\boldsymbol{\Theta} \mapsto \boldsymbol{\Theta} \odot \boldsymbol{\Gamma}^{(k,t+1)}$ is linear for each fixed $\boldsymbol{\Gamma}^{(k,t+1)}$, and since $\boldsymbol{\Gamma}^{(k,t+1)}$ is fixed in the current $\boldsymbol{\Theta}$-subproblem, the same argument implies that there exists a finite local Lipschitz constant $L_{\boldsymbol{\Theta}}^{(t)}$ such that $\nabla_{\boldsymbol{\Theta}} Q_{\mathrm{case},\boldsymbol{\Theta}}^{(t)}$ is Lipschitz continuous in the current subproblem. Let
        \begin{equation*}
            F_k^{(t)}(\boldsymbol{\Gamma}^{(k)}) = Q_{\mathrm{case},\boldsymbol{\Gamma}}^{(k,t)}(\boldsymbol{\Gamma}^{(k)}) + \lambda_2 \|\boldsymbol{\Gamma}^{(k)}\|_{1, \mathrm{off}}
        \end{equation*}
        and
        \begin{equation*}
            F^{(t)}(\boldsymbol{\Theta}) = Q_{\mathrm{case},\boldsymbol{\Theta}}^{(t)}(\boldsymbol{\Theta}) + \lambda_1 \|\boldsymbol{\Theta}\|_{1, \mathrm{off}}.
        \end{equation*}
        By the proximal gradient lemma \cite{beck2017first}, when $0 < \eta_{\boldsymbol{\Gamma},k}^{(t)} \le (1-\delta_{\boldsymbol{\Gamma}})/L_{\boldsymbol{\Gamma},k}^{(t)}$, applying this inequality with $X = \boldsymbol{\Gamma}^{(k,t)}, Y = \boldsymbol{\Gamma}^{(k,t+1)}$ and using the optimality of the proximal gradient update yields
        \begin{equation*}
            F_k^{(t)}(\boldsymbol{\Gamma}^{(k,t+1)}) \le F_k^{(t)}(\boldsymbol{\Gamma}^{(k,t)}) - \left( \frac{1}{2\eta_{\boldsymbol{\Gamma},k}^{(t)}} - \frac{L_{\boldsymbol{\Gamma},k}^{(t)}}{2} \right) \|\boldsymbol{\Gamma}^{(k,t+1)} - \boldsymbol{\Gamma}^{(k,t)}\|_F^2.
        \end{equation*}
        Similarly, when $0 < \eta_{\boldsymbol{\Theta}}^{(t)} \le (1-\delta_{\boldsymbol{\Theta}})/L_{\boldsymbol{\Theta}}^{(t)}$, we have
        \begin{equation*}
            F^{(t)}(\boldsymbol{\Theta}^{(t+1)}) \le F^{(t)}(\boldsymbol{\Theta}^{(t)}) - \left( \frac{1}{2\eta_{\boldsymbol{\Theta}}^{(t)}} - \frac{L_{\boldsymbol{\Theta}}^{(t)}}{2} \right) \|\boldsymbol{\Theta}^{(t+1)} - \boldsymbol{\Theta}^{(t)}\|_F^2.
        \end{equation*}
        Since the coefficients of the squared norm terms are nonnegative under the above local step-size conditions, the updates of $\boldsymbol{\Gamma}^{(k)}$ and $\boldsymbol{\Theta}$ do not increase the objective function. In summary,
        \begin{equation*}
            \mathcal{L}_{\mathrm{case}}(\boldsymbol{\Omega}^{(t+1)}, \boldsymbol{H}^{(t)}) \le \mathcal{L}_{\mathrm{case}}(\boldsymbol{\Omega}^{(t)}, \boldsymbol{H}^{(t)}).
        \end{equation*}

        Finally, the update rule for $\boldsymbol{H}$ ensures that the objective function does not increase. Given $\boldsymbol{\Omega}$, the terms related to $\boldsymbol{H}^{(k)}$ in the $k$-th group are
        \begin{equation*}
            \mathcal{L}_{\boldsymbol{H}^{(k)}} = \sum_{i=1}^{n_k} w_i^{(k)} \left[ \ln |\boldsymbol{\Sigma}^{(k)}| + q \ln(2\pi) + \mathrm{MD}^2(\boldsymbol{z}_i^{(k)}, \boldsymbol{\Omega}^{(k)}) \right],
        \end{equation*}
        where $\mathrm{MD}^2(\boldsymbol{z}_i^{(k)}, \boldsymbol{\Omega}^{(k)}) = (\boldsymbol{z}_i^{(k)} - \boldsymbol{\mu}^{(k)})^\top (\boldsymbol{\Sigma}^{(k)})^{-1} (\boldsymbol{z}_i^{(k)} - \boldsymbol{\mu}^{(k)})$. Since $\sum_{i=1}^{n_k} w_i^{(k)} = h_k$ is fixed, $\ln |\boldsymbol{\Sigma}^{(k)}| + q \ln(2\pi)$ acts as a constant. Therefore, minimizing $\mathcal{L}^{(k)}$ is equivalent to minimizing $\sum_{i=1}^{n_k} w_i^{(k)} \mathrm{MD}^2(\boldsymbol{z}_i^{(k)}, \boldsymbol{\Omega}^{(k)})$ with $w_i^{(k)} \in \{0, 1\}$ and $\sum_{i=1}^{n_k} w_i^{(k)} = h_k$. The solution to this 0-1 linear programming problem is to select the $h_k$ samples with the smallest Mahalanobis distances. Algorithm~\ref{alg:fastmcd_joint} performs exactly this. Hence,
        \begin{equation*}
            \mathcal{L}_{\text{case}}(\boldsymbol{\Omega}^{(t+1)}, \boldsymbol{H}^{(t+1)}) \leq \mathcal{L}_{\text{case}}(\boldsymbol{\Omega}^{(t+1)}, \boldsymbol{H}^{(t)}).
        \end{equation*}
        Thus,
        \begin{equation*}
            \mathcal{L}_{\mathrm{case}}(\boldsymbol{\Omega}^{(t+1)}, \boldsymbol{H}^{(t+1)}) \le \mathcal{L}_{\mathrm{case}}(\boldsymbol{\Omega}^{(t+1)}, \boldsymbol{H}^{(t)}) \le \mathcal{L}_{\mathrm{case}}(\boldsymbol{\Omega}^{(t)}, \boldsymbol{H}^{(t)}),
        \end{equation*}
        establishing the first property. Moreover, let $a = \min_{1\leq k \leq K}\{a_k\}$. For any $k$, given $\lambda_{\min}(\boldsymbol{\Sigma}^{(k)}) \ge a$, we have $|\boldsymbol{\Sigma}^{(k)}| = \prod_{j=1}^q \lambda_j(\boldsymbol{\Sigma}^{(k)}) \ge a^q$, which implies $\ln |\boldsymbol{\Sigma}^{(k)}| \ge q \ln a$. The Mahalanobis distance is non-negative:
        \begin{equation*}
            (\boldsymbol{z}_i^{(k)} - \boldsymbol{\mu}^{(k)})^\top (\boldsymbol{\Sigma}^{(k)})^{-1} (\boldsymbol{z}_i^{(k)} - \boldsymbol{\mu}^{(k)}) \ge 0.
        \end{equation*}
        The $\ell_1$ regularization term is also non-negative. Hence,
        \begin{equation*}
            \mathcal{L}_{\mathrm{case}}(\boldsymbol{\Omega}, \boldsymbol{H}) \ge \sum\limits_{k=1}^{K} Q_{\mathrm{case}}^{(k)} \ge \sum_{k=1}^K h_k \left[ q \ln a + q \ln(2\pi) \right] := L_{\min} > -\infty.
        \end{equation*}
        By the monotone convergence theorem, since the update sequence of $\mathcal{L}_{\mathrm{case}}$ is monotonically non-increasing and bounded below, it converges to a finite limit $\mathcal{L}^{*}$, establishing the second property.

        For the third property, we first show the iterative sequence has a limit point. Since
        \begin{equation*}
            \mathcal{L}_{\mathrm{case}}(\boldsymbol{\Omega}^{(t)}, \boldsymbol{H}^{(t)}) \le \mathcal{L}_{\mathrm{case}}(\boldsymbol{\Omega}^{(0)}, \boldsymbol{H}^{(0)}) := C_0,
        \end{equation*}
        all iterates are contained in the initial sublevel set. Moreover, the Mahalanobis-distance terms and the penalty term are nonnegative. Therefore, for each \(k\),
        \[
            h_k\ln|\boldsymbol{\Sigma}^{(k)}|
            \le
            C_0
            -
            h_k q\ln(2\pi)
            -
            \sum_{l \neq k} h_l \cdot q \{\ln a + \ln(2 \pi)\}
            =:C_k.
        \]
        Thus, \(|\boldsymbol{\Sigma}^{(k)}| \leq \exp(C_k/h_k)\). Since all eigenvalues of \(\boldsymbol{\Sigma}^{(k)}\) are bounded below by \(a\), we have
        \[
            |\boldsymbol{\Sigma}^{(k)}|
            =
            \prod_{j=1}^q \lambda_j(\boldsymbol{\Sigma}^{(k)})
            \ge
            a^{q-1}\lambda_{\max}(\boldsymbol{\Sigma}^{(k)}).
        \]
        Consequently, $\lambda_{\max}(\boldsymbol{\Sigma}^{(k)}) \leq a^{-(q-1)}\exp(C_k/h_k)$, which shows that \(\{\boldsymbol{\Sigma}^{(k,t)}\}_t\) is uniformly bounded for every
        group \(k\). Together with the lower eigenvalue constraint, this also implies
        that \(\{(\boldsymbol{\Sigma}^{(k,t)})^{-1}\}_t\) is uniformly bounded.

        We next claim that the location estimates are bounded. Let
        \[
            R_k=\max_{1\le i\le n_k}\|\boldsymbol{z}_i^{(k)}\|_2<\infty .
        \]
        For any selected subset \(\boldsymbol{H}^{(k,t)}\) of size \(h_k\),
        \[
        \sum_{i=1}^{n_k}w_i^{(k,t)}
        (\boldsymbol{z}_i^{(k)}-\boldsymbol{\mu}^{(k,t)})^\top
        (\boldsymbol{\Sigma}^{(k,t)})^{-1}
        (\boldsymbol{z}_i^{(k)}-\boldsymbol{\mu}^{(k,t)})
        \ge
        \frac{1}{\lambda_{\max}(\boldsymbol{\Sigma}^{(k,t)})}
        \sum_{i=1}^{n_k}w_i^{(k,t)}
        \|\boldsymbol{z}_i^{(k)}-\boldsymbol{\mu}^{(k,t)}\|_2^2 .
        \]
        Using \(\|\boldsymbol{z}_i^{(k)}-\boldsymbol{\mu}^{(k,t)}\|_2 \ge
        | \|\boldsymbol{\mu}^{(k,t)}\|_2 - R_k | \), we obtain
        \[
        \sum_{i=1}^{n_k}w_i^{(k,t)}
        (\boldsymbol{z}_i^{(k)}-\boldsymbol{\mu}^{(k,t)})^\top
        (\boldsymbol{\Sigma}^{(k,t)})^{-1}
        (\boldsymbol{z}_i^{(k)}-\boldsymbol{\mu}^{(k,t)})
        \ge
        \frac{1}{\lambda_{\max}(\boldsymbol{\Sigma}^{(k,t)})}
        (\|\boldsymbol{\mu}^{(k,t)}\|_2 - R_k)^2 .
        \]
        Since the left-hand side is bounded above on the same sublevel set and \(\lambda_{\max}(\boldsymbol{\Sigma}^{(k,t)})\) is uniformly bounded, it follows that \(\{\boldsymbol{\mu}^{(k,t)}\}_t\) is bounded.

        Finally, we prove the boundedness of the reparameterized variables. The lower bound on the likelihood part implies that the penalty term is uniformly bounded on the sublevel set:
        \[
            \lambda_1 \|\boldsymbol{\Theta}^{(t)}\|_{1,\mathrm{off}} + \lambda_2 \sum_k \|\boldsymbol{\Gamma}^{(k,t)}\|_{1,\mathrm{off}} \le C_0 - L_{\min} < \infty,
        \]
        Hence the off-diagonal entries of \(\boldsymbol{\Theta}^{(t)}\) and
        \(\boldsymbol{\Gamma}^{(k,t)}\) are bounded. The diagonal entries of \(\boldsymbol{\Theta}^{(t)}\) are fixed at one, while the diagonal entries of \(\boldsymbol{\Gamma}^{(k,t)}\) coincide with those of \(\boldsymbol{\Sigma}^{(k,t)}\) and are therefore bounded by the uniform bound on \(\lambda_{\max}(\boldsymbol{\Sigma}^{(k,t)})\). Thus
        \(\boldsymbol{\Theta}^{(t)}\) and \(\boldsymbol{\Gamma}^{(k,t)}\) are bounded. Therefore, the sequence $\{\boldsymbol{\Omega}^{(t)}\}$ is bounded. Furthermore, since $\boldsymbol{H}^{(t)}$ only takes a finite number of 0-1 combinations, by the Bolzano--Weierstrass theorem \cite{bartle2018introduction}, there exists a convergent subsequence $(\boldsymbol{\Omega}^{(t_r)}, \boldsymbol{H}^{(t_r)}) \to (\boldsymbol{\Omega}^*, \boldsymbol{H}^*)$, where $t_r$ is the subsequence index, $r \in \mathbb{N}^+$.

        Moreover, since $\{\boldsymbol{\Omega}^{(t)}\}$ is bounded and the feasible region satisfies $\lambda_{\min}(\boldsymbol{\Sigma}^{(k,t)}) \ge a_k > 0$, all iterates lie in a compact subset of the feasible region. On this compact set, the gradients of the smooth parts of the $\boldsymbol{\Gamma}^{(k)}$- and $\boldsymbol{\Theta}$-subproblems are globally Lipschitz continuous. Hence, there exist finite constants $L_{\boldsymbol{\Gamma},k}$ and $L_{\boldsymbol{\Theta}}$ such that, for all iterations $t$,
        \begin{equation*}
            L_{\boldsymbol{\Gamma},k}^{(t)} \le L_{\boldsymbol{\Gamma},k}, 
            \quad
            L_{\boldsymbol{\Theta}}^{(t)} \le L_{\boldsymbol{\Theta}}.
        \end{equation*}
        When the step sizes are chosen with a uniform margin, i.e.,
        \begin{equation*}
            0 < \eta_{\boldsymbol{\Gamma},k}^{(t)} \le \frac{1-\delta_{\boldsymbol{\Gamma}}}{L_{\boldsymbol{\Gamma},k}},
            \quad
            0 < \eta_{\boldsymbol{\Theta}}^{(t)} \le \frac{1-\delta_{\boldsymbol{\Theta}}}{L_{\boldsymbol{\Theta}}},
        \end{equation*}
        for some $\delta_{\boldsymbol{\Gamma}}, \delta_{\boldsymbol{\Theta}} \in (0,1)$, the proximal gradient descent inequalities imply the uniform sufficient decrease property. Specifically, there exist constants $c_{\boldsymbol{\Gamma},k}>0$ and $c_{\boldsymbol{\Theta}}>0$ such that
        \begin{equation*}
            F_k^{(t)}(\boldsymbol{\Gamma}^{(k,t+1)})
            \le
            F_k^{(t)}(\boldsymbol{\Gamma}^{(k,t)})
            -
            c_{\boldsymbol{\Gamma},k}
            \|\boldsymbol{\Gamma}^{(k,t+1)}-\boldsymbol{\Gamma}^{(k,t)}\|_F^2,
        \end{equation*}
        and
        \begin{equation*}
            F^{(t)}(\boldsymbol{\Theta}^{(t+1)})
            \le
            F^{(t)}(\boldsymbol{\Theta}^{(t)})
            -
            c_{\boldsymbol{\Theta}}
            \|\boldsymbol{\Theta}^{(t+1)}-\boldsymbol{\Theta}^{(t)}\|_F^2.
        \end{equation*}
        Summing these inequalities over $t$ and using the fact that $\mathcal{L}_{\mathrm{case}}$ is bounded below, we obtain 
        \begin{equation*}
            \sum_t \|\boldsymbol{\Gamma}^{(k,t+1)} - \boldsymbol{\Gamma}^{(k,t)}\|_F^2 < \infty, \quad \sum_t \|\boldsymbol{\Theta}^{(t+1)} - \boldsymbol{\Theta}^{(t)}\|_F^2 < \infty.
        \end{equation*}
                Hence,
        \begin{equation*}
            \boldsymbol{\Gamma}^{(k,t+1)} - \boldsymbol{\Gamma}^{(k,t)} \to 0, \quad \boldsymbol{\Theta}^{(t+1)} - \boldsymbol{\Theta}^{(t)} \to 0.
        \end{equation*}
        Moreover, since the step sizes are bounded away from zero, i.e.,
        \begin{equation*}
            \eta_{\boldsymbol{\Gamma},k}^{(t)} \ge \inf_{s} \eta_{\boldsymbol{\Gamma},k}^{(s)} >0, \quad \eta_{\boldsymbol{\Theta}}^{(t)} \ge \inf_{s} \eta_{\boldsymbol{\Theta}}^{(s)} >0, \quad t \geq 0.
        \end{equation*}
        it follows that
        \begin{equation*}
            \frac{1}{\eta_{\boldsymbol{\Gamma},k}^{(t)}}\left(\boldsymbol{\Gamma}^{(k,t+1)} - \boldsymbol{\Gamma}^{(k,t)}\right) \to 0,
            \quad
            \frac{1}{\eta_{\boldsymbol{\Theta}}^{(t)}}\left(\boldsymbol{\Theta}^{(t+1)} - \boldsymbol{\Theta}^{(t)}\right) \to 0.
        \end{equation*}
        For the proximal optimality condition of the $\boldsymbol{\Gamma}$-update:
        \begin{equation*}
            0 \in \nabla_{\boldsymbol{\Gamma}^{(k)}} Q_{\mathrm{case},\boldsymbol{\Gamma}}^{(k,t)}(\boldsymbol{\Gamma}^{(k,t)}) + \frac{1}{\eta_{\boldsymbol{\Gamma},k}^{(t)}} (\boldsymbol{\Gamma}^{(k,t+1)} - \boldsymbol{\Gamma}^{(k,t)}) + \lambda_2 \partial \|\boldsymbol{\Gamma}^{(k,t+1)}\|_{1,\mathrm{off}}.
        \end{equation*}
                Let $(\boldsymbol{\Omega}^{*}, \boldsymbol{H}^{*})$ be an accumulation point satisfying the interior condition
        \begin{equation*}
            \lambda_{\min}(\boldsymbol{\Theta}^{*}\odot\boldsymbol{\Gamma}^{(k)*}) > a_k,\quad k=1,2,\ldots,K.
        \end{equation*}
        By the continuity of the minimum eigenvalue, the eigenvalue constraints are inactive in a sufficiently small neighborhood of this accumulation point. Hence, along the convergent subsequence, the accepted proximal gradient updates coincide locally with the unconstrained proximal gradient updates for the corresponding penalized block subproblems.

        Taking the limit along the convergent subsequence and using the continuity of the gradient, the vanishing of the scaled difference term, and the closedness of the subdifferential of the $\ell_1$ norm yields
        \begin{equation*}
            0 \in \nabla_{\boldsymbol{\Gamma}^{(k)}} Q_{\mathrm{case}}^{(k)}(\boldsymbol{\Omega}^{*}, \boldsymbol{H}^{*}) + \lambda_2 \partial \|\boldsymbol{\Gamma}^{(k)*}\|_{1, \mathrm{off}}.
        \end{equation*}
        By the same reasoning for the $\boldsymbol{\Theta}$-update:
        \begin{equation*}
            0 \in \nabla_{\boldsymbol{\Theta}} \sum\limits_{k=1}^{K} Q_{\mathrm{case}}^{(k)}(\boldsymbol{\Omega}^{*}, \boldsymbol{H}^{*}) + \lambda_1 \partial \|\boldsymbol{\Theta}^{*}\|_{1, \mathrm{off}}.
        \end{equation*}
        Meanwhile, since the $\boldsymbol{\mu}$-update is an exact minimization:
        \begin{equation*}
            0 = \nabla_{\boldsymbol{\mu}^{(k)}} \mathcal{L}_{\mathrm{case}}(\boldsymbol{\Omega}^{*}, \boldsymbol{H}^{*}).
        \end{equation*}
        Finally, since the $\boldsymbol{H}$-update selects the $h_k$ samples with the smallest Mahalanobis distances given the current parameter estimates, $\boldsymbol{H}^{*}$ belongs to the set of minimizers of the discrete $\boldsymbol{H}$-subproblem given $\boldsymbol{\Omega}^{*}$; in the presence of ties, the conclusion is understood under the same tie-breaking rule used in the algorithm. Hence, any interior accumulation point $(\boldsymbol{\Omega}^{*}, \boldsymbol{H}^{*})$ is blockwise stationary with respect to the continuous variables, and $\boldsymbol{H}^{*}$ is optimal for the discrete block. This completes the proof of Theorem~\ref{thm:jointedfastmcd_converge}.
    \end{proof}

    \subsection{Proof of Theorem~\ref{thm:jointedcellmcd_converge}}\label{apd:jointedcellmcd_converge}

    \begin{proof}
        First, given $\boldsymbol{\Omega}$, for the $k$-th group, $i$-th row, and $j$-th column, let $o=\{\tau\neq j:w_{i\tau}^{(k)}=1\}$ denote the indices of the currently retained cells. By changing the status of cell $j$ from flagged to retained, the Schur complement yields: 
        \begin{equation*}
            |\boldsymbol{\Sigma}_{o \cup \{j\}}^{(k)}| = |\boldsymbol{\Sigma}_{oo}^{(k)}| \cdot C_{ij}^{(k)},
        \end{equation*}
        where $C_{ij}^{(k)} = \boldsymbol{\Sigma}_{jj}^{(k)} - \boldsymbol{\Sigma}_{j,o}^{(k)} (\boldsymbol{\Sigma}_{oo}^{(k)})^{-1} (\boldsymbol{\Sigma}_{j,o}^{(k)})^\top$. Meanwhile, the Mahalanobis distance can be decomposed as:
        \begin{equation*}
            \mathrm{MD}_{o \cup \{j\}}^2 = \mathrm{MD}_{o}^2 + \frac{(z_{ij}^{(k)} - \widehat{z}_{ij}^{(k)})^2}{C_{ij}^{(k)}},
        \end{equation*}
        where 
        \begin{equation*}
            \widehat{z}_{ij}^{(k)} = \boldsymbol{\mu}_j^{(k)} + \boldsymbol{\Sigma}_{j,o}^{(k)} (\boldsymbol{\Sigma}_{oo}^{(k)})^{-1} (\boldsymbol{z}_{i,o}^{(k)} - \boldsymbol{\mu}_{o}^{(k)}).
        \end{equation*}
        When $o=\varnothing$, we use the convention that $|\boldsymbol{\Sigma}_{oo}^{(k)}|=1$, $\mathrm{MD}_{o}^{2}=0$, $\widehat z_{ij}^{(k)}=\boldsymbol{\mu}_{j}^{(k)}$, and $C_{ij}^{(k)}=\boldsymbol{\Sigma}_{jj}^{(k)}$. Therefore, compared to keeping it ``flagged'', the increment in the objective function by retaining cell $(i,j)$ is exactly:
        \begin{equation*}
            \Delta_{ij}^{(k)} = \ln C_{ij}^{(k)} + \ln(2\pi) + \frac{(z_{ij}^{(k)} - \widehat{z}_{ij}^{(k)})^2}{C_{ij}^{(k)}} - b_j^{(k)}.
        \end{equation*}

        This matches the quantity defined in Algorithm~\ref{alg:cellwise-updater}. Consequently, fixing the other parameters, the optimization over the $j$-th column in the $k$-th group is equivalent to:
        \begin{equation*}
            \min_{w_{1j}^{(k)}, \dots, w_{n_k j}^{(k)} \in \{0,1\}} \sum_{i=1}^{n_k} w_{ij}^{(k)} \Delta_{ij}^{(k)} \quad \text{s.t.} \quad \sum_{i=1}^{n_k} w_{ij}^{(k)} \ge h_k.
        \end{equation*}
        This 0--1 linear programming problem admits the following solution: all cells with \(\Delta_{ij}^{(k)}<0\) should be retained, and if their number is less than \(h_k\), the remaining cells are selected in ascending order of \(\Delta_{ij}^{(k)}\) until the columnwise lower bound \(h_k\) is satisfied. The thresholding rule in Algorithm~\ref{alg:cellwise-updater} is precisely this optimal solution, where $\Delta_{(h_k)j}^{(k)}$ is the $h_k$-th smallest value in $\{\Delta_{ij}^{(k)}\}_{i=1}^{n_k}$. Thus, each columnwise weight update does not increase the objective function. Applying this argument sequentially over all columns $j=1,\dots,q$ and all groups $k=1,\dots,K$ gives
        \begin{equation*}
            \mathcal{L}_{\mathrm{cell}}(\boldsymbol{\Omega}^{(t)}, \boldsymbol{W}^{(t+1)}) \leq \mathcal{L}_{\mathrm{cell}}(\boldsymbol{\Omega}^{(t)}, \boldsymbol{W}^{(t)}).
        \end{equation*}

        Given $\boldsymbol{W}^{(t+1)}$, the update of $\boldsymbol{\Omega}$ can be viewed as a generalized EM step with penalty. We treat the flagged cells as missing variables and the retained cells as observed variables. Let $\mathcal{L}_{\mathrm{cell}}^{\mathrm{smooth}}(\boldsymbol{\Omega},\boldsymbol{W}^{(t+1)})$ denote the observed negative log-likelihood part of $\mathcal{L}_{\mathrm{cell}}(\boldsymbol{\Omega},\boldsymbol{W}^{(t+1)})$, excluding the cellwise flagging penalty and the $\ell_1$ regularization terms. For the current parameter $\boldsymbol{\Omega}^{(t)}$, the EM majorization relationship gives 
        \begin{equation*}
            \begin{aligned}
            &\quad \mathcal{L}_{\mathrm{cell}}^{\mathrm{smooth}}(\boldsymbol{\Omega},\boldsymbol{W}^{(t+1)})
            -
            \mathcal{L}_{\mathrm{cell}}^{\mathrm{smooth}}(\boldsymbol{\Omega}^{(t)},\boldsymbol{W}^{(t+1)}) \\
                \leq&
            Q_{\mathrm{smooth}}(\boldsymbol{\Omega}\mid \boldsymbol{\Omega}^{(t)},\boldsymbol{W}^{(t+1)})
            -
            Q_{\mathrm{smooth}}(\boldsymbol{\Omega}^{(t)}\mid \boldsymbol{\Omega}^{(t)},\boldsymbol{W}^{(t+1)}),
            \end{aligned}
        \end{equation*}
        where terms independent of $\boldsymbol{\Omega}$ are omitted from the surrogate. Define the penalized EM surrogate
        \begin{equation*}
            Q_{\mathrm{pen}}(\boldsymbol{\Omega}\mid \boldsymbol{\Omega}^{(t)},\boldsymbol{W}^{(t+1)})
            =
            Q_{\mathrm{smooth}}(\boldsymbol{\Omega}\mid \boldsymbol{\Omega}^{(t)},\boldsymbol{W}^{(t+1)})
            +\lambda_1\|\boldsymbol{\Theta}\|_{1,\mathrm{off}}
            +\lambda_2\sum_{k=1}^K\|\boldsymbol{\Gamma}^{(k)}\|_{1,\mathrm{off}}.
        \end{equation*}
        Since the cellwise flagging penalty is fixed when $\boldsymbol{W}^{(t+1)}$ is fixed, adding the regularization terms to both sides yields
        \begin{equation*}
            \mathcal{L}_{\mathrm{cell}}(\boldsymbol{\Omega},\boldsymbol{W}^{(t+1)})
            -
            \mathcal{L}_{\mathrm{cell}}(\boldsymbol{\Omega}^{(t)},\boldsymbol{W}^{(t+1)})
            \leq
            Q_{\mathrm{pen}}(\boldsymbol{\Omega}\mid \boldsymbol{\Omega}^{(t)},\boldsymbol{W}^{(t+1)})
            -
            Q_{\mathrm{pen}}(\boldsymbol{\Omega}^{(t)}\mid \boldsymbol{\Omega}^{(t)},\boldsymbol{W}^{(t+1)}).
        \end{equation*}
        Therefore, it is sufficient to find a feasible update $\boldsymbol{\Omega}^{(t+1)}$ satisfying
        \begin{equation*}
            Q_{\mathrm{pen}}(\boldsymbol{\Omega}^{(t+1)}\mid \boldsymbol{\Omega}^{(t)},\boldsymbol{W}^{(t+1)})
            \leq
            Q_{\mathrm{pen}}(\boldsymbol{\Omega}^{(t)}\mid \boldsymbol{\Omega}^{(t)},\boldsymbol{W}^{(t+1)}),
        \end{equation*}
        which guarantees
        \begin{equation*}
            \mathcal{L}_{\mathrm{cell}}(\boldsymbol{\Omega}^{(t+1)}, \boldsymbol{W}^{(t+1)})
            \leq
            \mathcal{L}_{\mathrm{cell}}(\boldsymbol{\Omega}^{(t)}, \boldsymbol{W}^{(t+1)}).
        \end{equation*}

        The update of $\boldsymbol{\mu}$ is the exact minimizer of the smooth EM surrogate with respect to the location parameters. Specifically, it is given by the sample mean of the conditional complete-data means. Hence, the $\boldsymbol{\mu}$-step does not increase $Q_{\mathrm{pen}}$.
        Next, fixing $\boldsymbol{\mu}^{(t+1)}$, $\boldsymbol{W}^{(t+1)}$, and $\boldsymbol{\Theta}^{(t)}$, the $\boldsymbol{\Gamma}^{(k)}$-subproblem has the form
        \begin{equation*}
            \min_{\boldsymbol{\Gamma}^{(k)}} 
            Q_{\mathrm{smooth},\boldsymbol{\Gamma}}^{(k,t)}(\boldsymbol{\Gamma}^{(k)})
            + \lambda_2 \|\boldsymbol{\Gamma}^{(k)}\|_{1,\mathrm{off}},
        \end{equation*}
        where $S_{\mathrm{case}}^{(k,t)}$ in the proof of Theorem~\ref{thm:jointedfastmcd_converge} is replaced by $S_{\mathrm{cell}}^{(k,t)}$. According to the proximal gradient update with backtracking described in \ref{app:theta_gamma_optimization}, the accepted update is feasible, satisfies $\lambda_{\min}(\boldsymbol{\Theta}^{(t)}\odot \boldsymbol{\Gamma}^{(k,t+1)})\ge a_k$ and yields
        \begin{equation*}
            Q_{\mathrm{smooth},\boldsymbol{\Gamma}}^{(k,t)}(\boldsymbol{\Gamma}^{(k,t+1)})
            + \lambda_2 \|\boldsymbol{\Gamma}^{(k,t+1)}\|_{1,\mathrm{off}}
            \leq
            Q_{\mathrm{smooth},\boldsymbol{\Gamma}}^{(k,t)}(\boldsymbol{\Gamma}^{(k,t)})
            + \lambda_2 \|\boldsymbol{\Gamma}^{(k,t)}\|_{1,\mathrm{off}}.
        \end{equation*}
        Similarly, fixing $\boldsymbol{\mu}^{(t+1)}$, $\boldsymbol{W}^{(t+1)}$, and $\{\boldsymbol{\Gamma}^{(k,t+1)}\}_{k=1}^K$, the $\boldsymbol{\Theta}$-subproblem is
        \begin{equation*}
            \min_{\boldsymbol{\Theta}}
            \sum_{k=1}^K Q_{\mathrm{smooth},\boldsymbol{\Theta}}^{(k,t)}(\boldsymbol{\Theta})
            + \lambda_1 \|\boldsymbol{\Theta}\|_{1,\mathrm{off}}.
        \end{equation*}
        The accepted proximal gradient update with backtracking is feasible, satisfies $\lambda_{\min}(\boldsymbol{\Theta}^{(t+1)}\odot \boldsymbol{\Gamma}^{(k,t+1)})\ge a_k$, $k=1,2,\ldots,K$, and gives
        \begin{equation*}
            \sum_{k=1}^K Q_{\mathrm{smooth},\boldsymbol{\Theta}}^{(k,t)}(\boldsymbol{\Theta}^{(t+1)})
            + \lambda_1 \|\boldsymbol{\Theta}^{(t+1)}\|_{1,\mathrm{off}}
            \leq
            \sum_{k=1}^K Q_{\mathrm{smooth},\boldsymbol{\Theta}}^{(k,t)}(\boldsymbol{\Theta}^{(t)})
            + \lambda_1 \|\boldsymbol{\Theta}^{(t)}\|_{1,\mathrm{off}}.
        \end{equation*}
        Combining the exact $\boldsymbol{\mu}$-update with the accepted proximal gradient updates for $\boldsymbol{\Gamma}^{(k)}$ and $\boldsymbol{\Theta}$, we obtain
        \begin{equation*}
            Q_{\mathrm{pen}}(\boldsymbol{\Omega}^{(t+1)}\mid \boldsymbol{\Omega}^{(t)},\boldsymbol{W}^{(t+1)})
            \leq
            Q_{\mathrm{pen}}(\boldsymbol{\Omega}^{(t)}\mid \boldsymbol{\Omega}^{(t)},\boldsymbol{W}^{(t+1)}).
        \end{equation*}
        Therefore,
        \begin{equation*}
            \mathcal{L}_{\mathrm{cell}}(\boldsymbol{\Omega}^{(t+1)}, \boldsymbol{W}^{(t+1)})
            \leq
            \mathcal{L}_{\mathrm{cell}}(\boldsymbol{\Omega}^{(t)}, \boldsymbol{W}^{(t+1)}).
        \end{equation*}
        Combining this with the $\boldsymbol{W}$-update from the first step establishes the monotonicity of the overall algorithm:
        \begin{equation*}
            \mathcal{L}_{\mathrm{cell}}(\boldsymbol{\Omega}^{(t+1)}, \boldsymbol{W}^{(t+1)})
            \leq
            \mathcal{L}_{\mathrm{cell}}(\boldsymbol{\Omega}^{(t)}, \boldsymbol{W}^{(t)}).
        \end{equation*}

        For any nonempty observed pattern $A \subset \{1, \dots, q\}$, $\boldsymbol{\Sigma}_A^{(k)}$ is a principal submatrix of $\boldsymbol{\Sigma}^{(k)}$. Let $a = \min_{1\leq k \leq K}\{a_k\}$. By Cauchy's interlace theorem (Theorem 4.3.28, \citep{horn2012matrix}),
        \begin{equation*}
            \lambda_{\min}(\boldsymbol{\Sigma}_A^{(k)}) \ge \lambda_{\min}(\boldsymbol{\Sigma}^{(k)}) \ge a.
        \end{equation*}
        Therefore,
        \begin{equation*}
            \ln |\boldsymbol{\Sigma}_A^{(k)}| \ge |A| \ln a > -\infty.
        \end{equation*}
        For the empty pattern $A=\varnothing$, we use the convention $\ln|\boldsymbol{\Sigma}_A^{(k)}|=0$ and $\mathrm{MD}_A^2=0$. The partial Mahalanobis distance is non-negative, the discrete penalty satisfies
        \begin{equation*}
            \sum_{k=1}^K \sum_{j=1}^q b_j^{(k)} \|\mathbf{1} - \boldsymbol{W}_{\cdot j}^{(k)}\|_0 \ge 0,
        \end{equation*}
        and the $\ell_1$ regularization terms are also non-negative. Consequently, $\mathcal{L}_{\mathrm{cell}}$ has a global lower bound. Since it decreases monotonically, the sequence $\mathcal{L}_{\mathrm{cell}}(\boldsymbol{\Omega}^{(t)},\boldsymbol{W}^{(t)})$ converges to a finite value, establishing the second property.

        We next show that the generated sequence has a convergent subsequence. Since 
        \begin{equation*}
            \mathcal{L}_{\mathrm{cell}}(\boldsymbol{\Omega}^{(t)},\boldsymbol{W}^{(t)})
            \leq
            \mathcal{L}_{\mathrm{cell}}(\boldsymbol{\Omega}^{(0)},\boldsymbol{W}^{(0)})=:C_0,
        \end{equation*}
        all iterates lie in the initial sublevel set. Fix a group $k$ and a variable $j$. The constraint $\|\boldsymbol{W}_{\cdot j}^{(k,t)}\|_0\ge h_k$ implies that at least one retained pattern contains the variable $j$. For any such nonempty pattern $A$ with $j\in A$, the corresponding log-determinant term $\ln |\boldsymbol{\Sigma}_A^{(k,t)}|$ is bounded above on the sublevel set, because all other terms in the objective are bounded below. Together with $\lambda_{\min}(\boldsymbol{\Sigma}_A^{(k,t)})\ge a$, we obtain
        \begin{equation*}
            |\boldsymbol{\Sigma}_{A}^{(k,t)}|
            =
            \prod_{l=1}^{|A|}
            \lambda_l(\boldsymbol{\Sigma}_{A}^{(k,t)})
            \ge
            a^{|A|-1}\lambda_{\max}(\boldsymbol{\Sigma}_{A}^{(k,t)}),
        \end{equation*}
        which implies that $\lambda_{\max}(\boldsymbol{\Sigma}_A^{(k,t)})$ is uniformly bounded. Hence the diagonal element $\sigma_{jj}^{(k,t)}$ is uniformly bounded. Since this holds for every $j$, all diagonal elements of $\boldsymbol{\Sigma}^{(k,t)}$ are uniformly bounded. As $\boldsymbol{\Sigma}^{(k,t)}$ is positive semidefinite, every $2\times 2$ principal submatrix is also positive semidefinite. Hence, for any $j\neq l$, $\sigma_{jj}^{(k,t)}\sigma_{ll}^{(k,t)}-(\sigma_{jl}^{(k,t)})^2\ge 0$, which implies $|\sigma_{jl}^{(k,t)}| \leq \sqrt{\sigma_{jj}^{(k,t)}\sigma_{ll}^{(k,t)}}$. Therefore, the boundedness of all diagonal entries implies the boundedness of all off-diagonal entries. Thus $\{\boldsymbol{\Sigma}^{(k,t)}\}_{t}$ is uniformly bounded for every $k$. The lower eigenvalue constraint also implies that $\{(\boldsymbol{\Sigma}^{(k,t)})^{-1}\}_{t}$ is uniformly bounded.

        We now prove that the location estimates are bounded. Let
        \begin{equation*}
            R_k=\max_{1\leq i\leq n_k}\|\boldsymbol{z}_i^{(k)}\|_2<\infty.
        \end{equation*}
        For each coordinate $j$, choose a retained pattern $A$ containing $j$. The corresponding partial Mahalanobis term satisfies
        \begin{equation*}
            (\boldsymbol{z}_{i,A}^{(k)}-\boldsymbol{\mu}_A^{(k)})^\top
            (\boldsymbol{\Sigma}_A^{(k)})^{-1}
            (\boldsymbol{z}_{i,A}^{(k)}-\boldsymbol{\mu}_A^{(k)})
            \geq
            \frac{1}{\lambda_{\max}(\boldsymbol{\Sigma}_A^{(k)})}
            \|\boldsymbol{z}_{i,A}^{(k)}-\boldsymbol{\mu}_A^{(k)}\|_2^2.
        \end{equation*}
        Since $\lambda_{\max}(\boldsymbol{\Sigma}_A^{(k)})$ is uniformly bounded and the objective is bounded above on the sublevel set, $\mu_j^{(k,t)}$ cannot diverge. Hence $\{\boldsymbol{\mu}^{(k,t)}\}_{t}$ is bounded for every $k$.

        Finally, the lower bound on the likelihood part implies that the penalty term is uniformly bounded on the sublevel set.
        Thus the off-diagonal entries of $\boldsymbol{\Theta}^{(t)}$ and $\boldsymbol{\Gamma}^{(k,t)}$ are bounded. The diagonal entries of $\boldsymbol{\Theta}^{(t)}$ are fixed at one, while the diagonal entries of $\boldsymbol{\Gamma}^{(k,t)}$ coincide with those of $\boldsymbol{\Sigma}^{(k,t)}$ and are therefore bounded. Hence $\boldsymbol{\Theta}^{(t)}$, $\boldsymbol{\Gamma}^{(k,t)}$, $\boldsymbol{\Sigma}^{(k,t)}$, and $\boldsymbol{\mu}^{(k,t)}$ are bounded.

        Since $\boldsymbol{W}^{(t)}$ only takes values from a finite set of 0-1 matrices, by the Bolzano--Weierstrass theorem there exists a convergent subsequence $(\boldsymbol{\Omega}^{(t_r)},\boldsymbol{W}^{(t_r)}) \to (\boldsymbol{\Omega}^*,\boldsymbol{W}^*)$, where $t_r \in \mathbb{N}^+$ is the index of the subsequence.  By the same sufficient-decrease argument as in the proof of Theorem~\ref{thm:jointedfastmcd_converge}, with $S_{\mathrm{case}}^{(k,t)}$ replaced by $S_{\mathrm{cell}}^{(k,t)}$, the accepted proximal gradient updates imply
        \begin{equation*}
            \sum_t \|\boldsymbol{\Gamma}^{(k,t+1)}-\boldsymbol{\Gamma}^{(k,t)}\|_F^2<\infty,
            \quad
            \sum_t \|\boldsymbol{\Theta}^{(t+1)}-\boldsymbol{\Theta}^{(t)}\|_F^2<\infty.
        \end{equation*}
        Therefore,
        \begin{equation*}
            \boldsymbol{\Gamma}^{(k,t+1)}-\boldsymbol{\Gamma}^{(k,t)}\to 0,
            \quad
            \boldsymbol{\Theta}^{(t+1)}-\boldsymbol{\Theta}^{(t)}\to 0.
        \end{equation*}
        The proximal optimality condition for the $\boldsymbol{\Gamma}^{(k)}$-update gives
        \begin{equation*}
            0 \in
            \nabla_{\boldsymbol{\Gamma}^{(k)}}Q_{\mathrm{smooth},\boldsymbol{\Gamma}}^{(k,t)}(\boldsymbol{\Gamma}^{(k,t)})
            +
            \frac{1}{\eta_{\boldsymbol{\Gamma},k}^{(t)}}
            \left(\boldsymbol{\Gamma}^{(k,t+1)}-\boldsymbol{\Gamma}^{(k,t)}\right)
            +
            \lambda_2\partial\|\boldsymbol{\Gamma}^{(k,t+1)}\|_{1,\mathrm{off}}.
        \end{equation*}
                Let $(\boldsymbol{\Omega}^{*},\boldsymbol{W}^{*})$ be an accumulation point satisfying the interior condition
        \begin{equation*}
            \lambda_{\min}(\boldsymbol{\Theta}^{*}\odot\boldsymbol{\Gamma}^{(k)*}) > a_k,\quad k=1,2,\ldots,K.
        \end{equation*}
        By the continuity of the minimum eigenvalue, the eigenvalue constraints are inactive in a sufficiently small neighborhood of this accumulation point. Hence, along the convergent subsequence, the accepted proximal gradient updates coincide locally with the unconstrained proximal gradient updates for the corresponding penalized block subproblems.

        Taking the limit along the convergent subsequence and using the continuity of the gradient, the vanishing of the scaled difference term, and the closedness of the subdifferential of the $\ell_1$ norm, we obtain
        \begin{equation*}
            0 \in
            \nabla_{\boldsymbol{\Gamma}^{(k)}}Q_{\mathrm{smooth}}^{(k)}(\boldsymbol{\Omega}^*\mid \boldsymbol{\Omega}^*,\boldsymbol{W}^*)
            +
            \lambda_2\partial\|\boldsymbol{\Gamma}^{(k)*}\|_{1,\mathrm{off}}.
        \end{equation*}
        Similarly,
        \begin{equation*}
            0 \in
            \nabla_{\boldsymbol{\Theta}}
            \sum_{k=1}^K Q_{\mathrm{smooth}}^{(k)}(\boldsymbol{\Omega}^*\mid \boldsymbol{\Omega}^*,\boldsymbol{W}^*)
            +
            \lambda_1\partial\|\boldsymbol{\Theta}^*\|_{1,\mathrm{off}}.
        \end{equation*}
        Since the $\boldsymbol{\mu}$-update is an exact minimization of the smooth EM surrogate,
        \begin{equation*}
            0=
            \nabla_{\boldsymbol{\mu}^{(k)}}Q_{\mathrm{smooth}}(\boldsymbol{\Omega}^*\mid \boldsymbol{\Omega}^*,\boldsymbol{W}^*).
        \end{equation*}
        At $\boldsymbol{\Omega}=\boldsymbol{\Omega}^*$, the EM surrogate and the observed smooth objective have the same first-order derivatives with respect to the continuous variables. Hence the above conditions are the first-order stationarity conditions of the penalized observed objective with $\boldsymbol{W}^*$ fixed.
        Moreover, the $\boldsymbol{W}$-update is coordinatewise optimal for the discrete variables. Therefore, $\boldsymbol{W}^*$ is coordinatewise optimal given $\boldsymbol{\Omega}^*$, with ties interpreted according to the same tie-breaking convention used in the algorithm. Hence every interior accumulation point $(\boldsymbol{\Omega}^*,\boldsymbol{W}^*)$ is coordinatewise optimal with respect to $\boldsymbol{W}$ and stationary with respect to the continuous variables. This completes the proof of Theorem~\ref{thm:jointedcellmcd_converge}.
    \end{proof}

    \subsection{Proof of Theorem~\ref{thm:jointcellmcd_breakdown_point}}\label{apd:jointcellmcd_breakdown_point}

    \begin{proof}
        First, we prove the existence of a clean benchmark with a finite objective value. Suppose that the number of contaminated cells in each column of each group does not exceed $n_k - h_k$. Then, in every column of group $k$, at least $h_k$ cells remain uncontaminated. Hence, for each group $k$, we can construct a feasible benchmark mask $\boldsymbol{W}_0^{(k)}$ by setting all contaminated cells to $0$ and all uncontaminated cells to $1$. This mask satisfies $\|\boldsymbol{W}_{0,\cdot j}^{(k)}\|_0 \ge h_k$, $j=1,2,\ldots,q$, $k=1,2,\ldots,K$. Choose a fixed feasible joint covariance parameter, for instance $\boldsymbol{\Theta}_0=\boldsymbol{I}_q$, $\boldsymbol{\Gamma}_0^{(k)}=c_k\boldsymbol{I}_q$, $c_k\ge a_k$, $\boldsymbol{\Sigma}_0^{(k)} = \boldsymbol{\Theta}_0\odot \boldsymbol{\Gamma}_0^{(k)} = c_k\boldsymbol{I}_q$. Also take any finite location vector $\boldsymbol{\mu}_0^{(k)}$, for example $\boldsymbol{\mu}_0^{(k)}=\boldsymbol{0}$. Since the benchmark mask retains only clean cells, the observed likelihood terms evaluated at
        $(\{\boldsymbol{\mu}_0^{(k)}\}_{k=1}^K,\boldsymbol{\Theta}_0,\{\boldsymbol{\Gamma}_0^{(k)}\}_{k=1}^K,\{\boldsymbol{W}_0^{(k)}\}_{k=1}^K)$ depend only on the original clean data and on fixed constants. Moreover, $\mathbf{P}(\boldsymbol{\Theta}_0,\{\boldsymbol{\Gamma}_0^{(k)}\}_{k=1}^{K})<\infty$ and the cellwise penalties are finite for the fixed nonnegative constants $b_j^{(k)}$. Therefore, there exists a constant $M<\infty$, independent of the magnitudes of the contaminated cells, such that the optimal joint cellwise objective value is at most $M$. This is because the objective function value corresponding to the optimal solution will not exceed the objective function value corresponding to the above feasible solutions. Thus, the required finite benchmark is provided.

        Second, we prove the scatter implosion breakdown point. Since every feasible solution of the joint cellwise estimator satisfies $\lambda_{\min}\left(\widehat{\boldsymbol{\Sigma}}^{(k)}\right)\ge a_k>0$, $k=1,2,\ldots,K$, the scatter matrix of any group cannot implode for any admissible contaminated sample.
        Hence,
        \[
            \varepsilon_{\mathrm{cell},-}(\widehat{\boldsymbol{\Omega}}, \boldsymbol{Z}^{\star}) = 1.
        \]

        Third, we prove the scatter explosion breakdown point by contradiction. Suppose that there exists a sequence of contaminated samples, with at most $n_k-h_k$ contaminated cells in each column of each group, such that
        \[
            \lambda_{\max}\left(\widehat{\boldsymbol{\Sigma}}^{(k_0)}\right)\to\infty
        \]
        for some group $k_0$. For the optimal cellwise mask in this group, define $A_i=\left\{j:\widehat w_{ij}^{(k_0)}=1\right\}$, $q_i = |A_i|$, and write $\widehat{\boldsymbol{\Sigma}}_{A_i,A_i}^{(k_0)}$ for the corresponding principal submatrix. When $A_i=\varnothing$, we use the standard convention $\ln|\widehat{\boldsymbol{\Sigma}}_{A_i,A_i}^{(k_0)}|=0$.
        Let $j_* \in \arg\max_{1\le j\le q} \hat{\sigma}_{jj}^{(k_0)}$. Since $\widehat{\boldsymbol{\Sigma}}^{(k_0)}$ is positive semidefinite,
        \[
            \lambda_{\max}\left(\widehat{\boldsymbol{\Sigma}}^{(k_0)}\right)
            \le
            q\max_{1\le j\le q}
            \hat{\sigma}_{jj}^{(k_0)}.
        \]
        Hence,
        \[
            \hat{\sigma}_{j_*j_*}^{(k_0)}
            \ge
            \frac{
            \lambda_{\max}\left(\widehat{\boldsymbol{\Sigma}}^{(k_0)}\right)
            }{q}.
        \]
        By the feasibility constraint
        $\|\widehat{\boldsymbol{W}}_{\cdot j_*}^{(k_0)}\|_0\ge h_{k_0}$,
        there exists at least one row $i_*$ such that
        $\widehat w_{i_*j_*}^{(k_0)}=1$, or equivalently $j_*\in A_{i_*}$. Therefore,
        \[
            \lambda_{\max}
            \left(
            \widehat{\boldsymbol{\Sigma}}_{A_{i_*},A_{i_*}}^{(k_0)}
            \right)
            \ge
            \hat{\sigma}_{j_*j_*}^{(k_0)}
            \ge
            \frac{
            \lambda_{\max}\left(\widehat{\boldsymbol{\Sigma}}^{(k_0)}\right)
            }{q}.
        \]
        Moreover, by the Cauchy interlacing theorem,
        \[
            \lambda_{\min}
            \left(
            \widehat{\boldsymbol{\Sigma}}_{A_i,A_i}^{(k_0)}
            \right)
            \ge
            \lambda_{\min}
            \left(
            \widehat{\boldsymbol{\Sigma}}^{(k_0)}
            \right)
            \ge a_{k_0}
        \]
        whenever $A_i\ne\varnothing$. Thus,
        \[
            \begin{aligned}
                \ln\left|
                \widehat{\boldsymbol{\Sigma}}_{A_{i_*},A_{i_*}}^{(k_0)}
                \right|
                &=
                \sum_{l=1}^{q_{i_*}}
                \ln
                \lambda_l
                \left(
                \widehat{\boldsymbol{\Sigma}}_{A_{i_*},A_{i_*}}^{(k_0)}
                \right)\\
                &\ge
                \ln
                \lambda_{\max}
                \left(
                \widehat{\boldsymbol{\Sigma}}_{A_{i_*},A_{i_*}}^{(k_0)}
                \right)
                +
                (q_{i_*}-1)\ln a_{k_0}\\
                &\ge
                \ln
                \frac{
                \lambda_{\max}
                \left(
                \widehat{\boldsymbol{\Sigma}}^{(k_0)}
                \right)
                }{q}
                +
                (q_{i_*}-1)\ln a_{k_0}.
            \end{aligned}
        \]
        For every other row $i$, we also have
        \[
            \ln\left|
            \widehat{\boldsymbol{\Sigma}}_{A_i,A_i}^{(k_0)}
            \right|
            \ge
            q_i\ln a_{k_0}
        \]
        with the above convention when $q_i=0$. Consequently, for a finite constant
        \[
            C_{k_0}
            =
            (n_{k_0}q-1)\min\{\ln a_{k_0},0\},
        \]
        depending only on $n_{k_0}$, $q$, and $a_{k_0}$, we obtain
        \[
            \sum_{i=1}^{n_{k_0}}
            \ln
            \left|
            \widehat{\boldsymbol{\Sigma}}_{A_i,A_i}^{(k_0)}
            \right|
            \ge
            \ln
            \frac{
            \lambda_{\max}
            \left(
            \widehat{\boldsymbol{\Sigma}}^{(k_0)}
            \right)
            }{q}
            +
            C_{k_0}.
        \]
        Since the partial Mahalanobis distances, the cellwise penalty, and the $\ell_1$ penalty are all nonnegative, the whole joint cellwise objective tends to $+\infty$ if
        $\lambda_{\max}(\widehat{\boldsymbol{\Sigma}}^{(k_0)})\to\infty$. This contradicts the finite benchmark $M$ established in the first step. Therefore,
        \[
            \varepsilon_{\mathrm{cell},+}(\widehat{\boldsymbol{\Omega}}, \boldsymbol{Z}^{\star}) \ge \beta_0.
        \]

        Finally, we prove the location breakdown point. Suppose again that the number of contaminated cells per column in each group does not exceed $n_k-h_k$. Consider a sequence of such contaminated samples and assume, for contradiction, that
        \[
            \left\|\widehat{\boldsymbol{\mu}}^{(k_0)}\right\|_2\to\infty
        \]
        for some group $k_0$. From the scatter explosion result proved above, along this sequence
        $\lambda_{\max}(\widehat{\boldsymbol{\Sigma}}^{(k_0)})$ cannot diverge. Hence, there exists a finite constant $C_\Sigma$ such that
        \[
            \lambda_{\max}
            \left(
            \widehat{\boldsymbol{\Sigma}}^{(k_0)}
            \right)
            \le C_\Sigma
        \]
        along the sequence. Let $M_0$ be a finite constant such that the absolute values of all clean cells in group $k_0$ are bounded by $M_0$.

        For each coordinate $j$, at least $h_{k_0}$ cells in column $j$ are clean, while every feasible mask retains at least $h_{k_0}$ cells in the same column. Since $h_{k_0}>n_{k_0}/2$, these two sets must intersect. Thus, for each $j=1,2,\ldots,q$, there exists an index $i(j)$ such that the cell $z_{i(j)j}^{(k_0)}$ is both clean and retained. Therefore,
        \[
            \left|z_{i(j)j}^{\star(k_0)}\right|\le M_0,
            \quad
            \widehat w_{i(j)j}^{(k_0)}=1.
        \]
        Using the lower bound on the partial Mahalanobis terms, we have
        \[
            \begin{aligned}
                \sum_{i=1}^{n_{k_0}}
                \mathrm{MD}^2
                \left(
                \boldsymbol{z}_{i,A_i}^{\star(k_0)},
                \widehat{\boldsymbol{\mu}}_{A_i}^{(k_0)},
                \widehat{\boldsymbol{\Sigma}}_{A_i,A_i}^{(k_0)}
                \right)
                &\ge
                \frac{1}{C_\Sigma}
                \sum_{i=1}^{n_{k_0}}
                \left\|
                \boldsymbol{z}_{i,A_i}^{\star(k_0)}
                -
                \widehat{\boldsymbol{\mu}}_{A_i}^{(k_0)}
                \right\|_2^2\\
                &\ge
                \frac{1}{C_\Sigma}
                \sum_{j=1}^{q}
                \left(
                z_{i(j)j}^{\star(k_0)}
                -
                \widehat{\mu}_{j}^{(k_0)}
                \right)^2\\
                &\ge
                \frac{1}{C_\Sigma}
                \left(
                \frac{1}{2}
                \left\|
                \widehat{\boldsymbol{\mu}}^{(k_0)}
                \right\|_2^2
                -
                qM_0^2
                \right).
            \end{aligned}
        \]
        The log-determinant terms are bounded from below by a finite constant because
        $\lambda_{\min}(\widehat{\boldsymbol{\Sigma}}^{(k_0)})\ge a_{k_0}$. Hence, if
        $\|\widehat{\boldsymbol{\mu}}^{(k_0)}\|_2\to\infty$, the joint cellwise objective must tend to $+\infty$, contradicting the finite benchmark $M$. Therefore,
        \[
            \varepsilon_{\mathrm{cell},\boldsymbol{\mu}}^*(\widehat{\boldsymbol{\Omega}}, \boldsymbol{Z}^{\star}) \ge \beta_0.
        \]

        To show that the common lower bound is sharp, consider
        \[
            k_* \in \arg\min_k \frac{n_k - h_k + 1}{n_k}.
        \]
        In this group, fix a column, say $j_*=1$, replace exactly $n_{k_*}-h_{k_*}+1$ cells in this column by a value $c$, and let $c\to\infty$, leaving all other cells unchanged. Let $\mathcal{I}$ be the set of rows whose first-column cell is contaminated. Since
        $|\mathcal{I}|=n_{k_*}-h_{k_*}+1$, any feasible mask satisfying
        $\|\boldsymbol{W}_{\cdot 1}^{(k_*)}\|_0\ge h_{k_*}$ must retain at least one contaminated $c$-cell in the first column, that is,
        \[
            \#\left(\left\{i:\widehat w_{i1}^{(k_*)}=1\right\}\cap\mathcal{I}\right)\ge 1.
        \]
        At convergence of the EM $\boldsymbol{\mu}$-block, the first coordinate satisfies
        \[
            \begin{aligned}
                \widehat{\mu}_{1}^{(k_*)}
                &=
                \frac{1}{n_{k_*}}
                \sum_{i=1}^{n_{k_*}}
                \widehat m_{i1}^{(k_*)}\\
                &=
                \frac{1}{n_{k_*}}
                \sum_{i:\widehat w_{i1}^{(k_*)}=1}
                \tilde{z}_{i1}^{(k_*)} 
                +
                \frac{1}{n_{k_*}}
                \sum_{i:\widehat w_{i1}^{(k_*)}=0}
                \left[
                \widehat{\mu}_{1}^{(k_*)}
                +
                \widehat{\boldsymbol{\Sigma}}_{1,\mathcal{O}_i}^{(k_*)}
                \left(
                \widehat{\boldsymbol{\Sigma}}_{\mathcal{O}_i,\mathcal{O}_i}^{(k_*)}
                \right)^{-1}
                \left(
                    \tilde{\boldsymbol{z}}_{i,\mathcal{O}_i}^{(k_*)}
                -
                \widehat{\boldsymbol{\mu}}_{\mathcal{O}_i}^{(k_*)}
                \right)
                \right]\\
                &=
                \frac{c}{n_{k_*}}
                \#
                \left(
                \left\{i:\widehat w_{i1}^{(k_*)}=1\right\}\cap\mathcal{I}
                \right)
                +
                \frac{1}{n_{k_*}}
                \sum_{\{i:\widehat w_{i1}^{(k_*)}=1\}\cap\mathcal{I}^c}
                z_{i1}^{\star(k_*)}\\
                &\quad+
                \frac{1}{n_{k_*}}
                \sum_{i:\widehat w_{i1}^{(k_*)}=0}
                \left[
                \widehat{\mu}_{1}^{(k_*)}
                +
                \widehat{\boldsymbol{\Sigma}}_{1,\mathcal{O}_i}^{(k_*)}
                \left(
                \widehat{\boldsymbol{\Sigma}}_{\mathcal{O}_i,\mathcal{O}_i}^{(k_*)}
                \right)^{-1}
                \left(
                \boldsymbol{z}_{i,\mathcal{O}_i}^{\star(k_*)}
                -
                \widehat{\boldsymbol{\mu}}_{\mathcal{O}_i}^{(k_*)}
                \right)
                \right],
            \end{aligned}
        \]
        where $\mathcal{O}_i=\{j:\widehat w_{ij}^{(k_*)}=1,\ j\ne 1\}$, and the usual empty-set convention is used if $\mathcal{O}_i=\varnothing$. If both
        $\widehat{\boldsymbol{\mu}}^{(k_*)}$ and
        $\widehat{\boldsymbol{\Sigma}}^{(k_*)}$ remained bounded as $c\to\infty$, then the left-hand side and all terms on the right-hand side except the first one would remain bounded. The inverse submatrices are bounded because
        $\lambda_{\min}(\widehat{\boldsymbol{\Sigma}}^{(k_*)})\ge a_{k_*}$. However, the first term on the right-hand side diverges to $+\infty$, a contradiction. Therefore, at contamination fraction
        $(n_{k_*}-h_{k_*}+1)/n_{k_*}=\beta_0$, either the location or the scatter matrix, or both, must break down. This proves that the common lower bound $\beta_0$ for scatter explosion and location breakdown is sharp.
    \end{proof}

    \subsection{Proof of Theorem~\ref{thm:jointfastmcd_breakdown_point}}\label{apd:jointfastmcd_breakdown_point}

    \begin{proof}
        The proof runs in parallel with the cellwise case, replacing the cellwise masks $\boldsymbol{W}^{(k)}$ with rowwise weights $w_i^{(k)}$. First, we prove the existence of a clean benchmark with a finite objective value. Suppose that the number of casewise contaminated rows in group $k$ does not exceed $n_k-h_k$. Then there exists a subset $\boldsymbol{H}_0^{(k)}$ of size $h_k$ consisting entirely of clean rows. Choose, for example,
        \[
            \boldsymbol{\Theta}_0=\boldsymbol{I}_q,
            \quad
            \boldsymbol{\Gamma}_0^{(k)}=c_k\boldsymbol{I}_q,
            \quad
            c_k\ge a_k,
            \quad
            \boldsymbol{\Sigma}_0^{(k)}
            =
            \boldsymbol{\Theta}_0\odot\boldsymbol{\Gamma}_0^{(k)}
            =
            c_k\boldsymbol{I}_q,
        \]
        and
        \[
            \boldsymbol{\mu}_0^{(k)}
            =
            \frac{1}{h_k}
            \sum_{i\in \boldsymbol{H}_0^{(k)}}
            \boldsymbol{z}_i^{\star(k)}.
        \]
        Evaluating the joint casewise objective at
        $(\{\boldsymbol{\mu}_0^{(k)}\}_{k=1}^K,\boldsymbol{\Theta}_0,\{\boldsymbol{\Gamma}_0^{(k)}\}_{k=1}^K,\{\boldsymbol{H}_0^{(k)}\}_{k=1}^K)$
        gives a finite value depending only on the clean data and fixed constants, not on the magnitudes of the contaminated rows. Also,
        $\mathbf{P}(\boldsymbol{\Theta}_0,\{\boldsymbol{\Gamma}_0^{(k)}\}_{k=1}^{K})<\infty$. Hence, there exists a finite constant $M<\infty$ such that, under at most $n_k-h_k$ contaminated rows in every group, the optimal joint casewise objective value is at most $M$.

        Second, the implosion breakdown value equals $1$. Indeed, every feasible solution satisfies
        \[
            \lambda_{\min}\left(\widehat{\boldsymbol{\Sigma}}^{(k)}\right)\ge a_k>0,
            \quad k=1,2,\ldots,K.
        \]
        Therefore, no scatter matrix can implode. Under the standard convention that a breakdown fraction is reported within $[0,1]$, this gives
        \[
            \delta_{\mathrm{case},-}(\widehat{\boldsymbol{\Omega}}, \boldsymbol{Z}^{\star})=1.
        \]

        Third, we prove the scatter explosion breakdown point. Suppose, for the sake of contradiction, that there exists a sequence of contaminated samples, with at most $n_k-h_k$ contaminated rows in each group, such that
        \[
            \lambda_{\max}
            \left(
            \widehat{\boldsymbol{\Sigma}}^{(k_0)}
            \right)
            \to\infty
        \]
        for some group $k_0$. Since
        $\lambda_{\min}(\widehat{\boldsymbol{\Sigma}}^{(k_0)})\ge a_{k_0}$, we have
        \[
            \begin{aligned}
                \ln
                \left|
                \widehat{\boldsymbol{\Sigma}}^{(k_0)}
                \right|
                =
                \sum_{j=1}^{q}
                \ln
                \lambda_j
                \left(
                \widehat{\boldsymbol{\Sigma}}^{(k_0)}
                \right)
                \ge
                \ln
                \lambda_{\max}
                \left(
                \widehat{\boldsymbol{\Sigma}}^{(k_0)}
                \right)
                +
                (q-1)\ln a_{k_0}.
            \end{aligned}
        \]
        Therefore, the log-determinant contribution of group $k_0$ in the joint casewise objective is bounded from below by
        \[
            h_{k_0}
            \left[
            \ln
            \lambda_{\max}
            \left(
            \widehat{\boldsymbol{\Sigma}}^{(k_0)}
            \right)
            +
            (q-1)\ln a_{k_0}
            \right].
        \]
        Since the Mahalanobis terms and the $\ell_1$ penalty are nonnegative, the joint casewise objective tends to $+\infty$ as
        $\lambda_{\max}(\widehat{\boldsymbol{\Sigma}}^{(k_0)})\to\infty$. This contradicts the finite benchmark $M$ established in the first step. Hence,
        \[
            \delta_{\mathrm{case},+}(\widehat{\boldsymbol{\Omega}}, \boldsymbol{Z}^{\star}) \ge \beta_0.
        \]

        Finally, we prove the location breakdown point. Suppose, for contradiction, that under at most $n_k-h_k$ contaminated rows in each group, there exists a sequence of contaminated samples such that
        \[
            \left\|
            \widehat{\boldsymbol{\mu}}^{(k_0)}
            \right\|_2
            \to\infty
        \]
        for some group $k_0$. From the scatter explosion result proved above,
        $\lambda_{\max}(\widehat{\boldsymbol{\Sigma}}^{(k_0)})$ cannot diverge along this sequence. Hence, there exists a finite constant $C_\Sigma$ such that
        \[
            \lambda_{\max}
            \left(
            \widehat{\boldsymbol{\Sigma}}^{(k_0)}
            \right)
            \le C_\Sigma.
        \]
        Any feasible subset of size $h_{k_0}$ must contain at least one clean row, because at most $n_{k_0}-h_{k_0}$ rows are contaminated and $h_{k_0}>n_{k_0}/2$. Let $\boldsymbol{z}_{i_0}^{\star(k_0)}$ be such a retained clean row, and let
        \[
            R_0=\max_{i}
            \left\|
            \boldsymbol{z}_{i}^{\star(k_0)}
            \right\|_2
        \]
        over the clean rows of group $k_0$. Then
        \[
            \left\|
            \boldsymbol{z}_{i_0}^{\star(k_0)}
            \right\|_2
            \le R_0.
        \]
        The Mahalanobis term corresponding to this retained clean row satisfies
        \[
            \begin{aligned}
                \left(
                \boldsymbol{z}_{i_0}^{\star(k_0)}
                -
                \widehat{\boldsymbol{\mu}}^{(k_0)}
                \right)^{\top}
                \left(
                \widehat{\boldsymbol{\Sigma}}^{(k_0)}
                \right)^{-1}
                \left(
                \boldsymbol{z}_{i_0}^{\star(k_0)}
                -
                \widehat{\boldsymbol{\mu}}^{(k_0)}
                \right)
                \ge
                \frac{
                \left(
                \left\|
                \widehat{\boldsymbol{\mu}}^{(k_0)}
                \right\|_2
                -
                R_0
                \right)^2
                }{
                C_\Sigma
                }.
            \end{aligned}
        \]
        Thus, if
        $\|\widehat{\boldsymbol{\mu}}^{(k_0)}\|_2\to\infty$, the joint casewise objective must tend to $+\infty$, contradicting the finite benchmark $M$. Therefore,
        \[
            \delta_{\mathrm{case},\boldsymbol{\mu}}^*(\widehat{\boldsymbol{\Omega}}, \boldsymbol{Z}^{\star}) \ge \beta_0.
        \]

        To show that the common lower bound is sharp, consider the most vulnerable group
        \[
            k_* \in \arg\min_k \frac{n_k-h_k+1}{n_k}.
        \]
        In this group, replace the first coordinate of exactly $n_{k_*}-h_{k_*}+1$ rows by a value $c$, and let $c\to\infty$, leaving all other entries unchanged. Let $\mathcal{I}$ be the set of contaminated rows. Since
        $|\mathcal{I}|=n_{k_*}-h_{k_*}+1$, any feasible subset of size $h_{k_*}$ must contain at least one row from $\mathcal{I}$. At any exact minimizer, the $\boldsymbol{\mu}$-block satisfies the weighted mean equation
        \[
            \widehat{\boldsymbol{\mu}}^{(k_*)}
            =
            \frac{1}{h_{k_*}}
            \sum_{i=1}^{n_{k_*}}
            \widehat w_i^{(k_*)}
            \boldsymbol{z}_i^{\star(k_*)}.
        \]
        Hence, for the first coordinate,
        \[
            \widehat{\mu}_1^{(k_*)}
            =
            \frac{c}{h_{k_*}}
            \#
            \left(
            \left\{i:\widehat w_i^{(k_*)}=1\right\}
            \cap
            \mathcal{I}
            \right)
            +
            \frac{1}{h_{k_*}}
            \sum_{\{i:\widehat w_i^{(k_*)}=1\}\cap\mathcal{I}^c}
            z_{i1}^{\star(k_*)}.
        \]
        The first term on the right-hand side diverges to $+\infty$, while the second term is bounded because it only involves clean data. Therefore, at contamination fraction
        $(n_{k_*}-h_{k_*}+1)/n_{k_*}=\beta_0$, either the location estimator diverges, or the scatter estimator has already broken down. This proves that the common lower bound $\beta_0$ is sharp.
    \end{proof}

    \section{Simulation studies}\label{app:simulation}

    We conducted a comprehensive set of simulation studies to evaluate the proposed joint estimation algorithm, and compared its performance with the non-joint MCD \citep{hubert2010minimum,rousseeuw1999fast} and cellMCD \citep{raymaekers2024cellwise} methods. For each configuration, the simulation was repeated 100 times, and the reported results are the means and standard deviations of the evaluation metrics over these 100 replications. 

    For the joint method, in each simulation replication we generated $K$ groups of $q$-dimensional observations, where the $k$-th group contained $n_k$ samples. Clean observations were generated from a multivariate normal distribution $\mathcal{N}_q(\boldsymbol{\mu}^{(k)},\boldsymbol{\Sigma}^{(k)})$. The group mean vectors were set as $\boldsymbol{\mu}^{(k)} = (k-1)\Delta \cdot \mathbf{1}_q$, where $\Delta$ controls the mean shift across groups. The covariance matrices followed $\boldsymbol{\Sigma}^{(k)} = \boldsymbol{\Theta} \odot \boldsymbol{\Gamma}^{(k)}$, where $\boldsymbol{\Theta}$ denotes the correlation structure shared by all groups, $\boldsymbol{\Gamma}^{(k)}$ is a group-specific positive definite matrix for the $k$-th group, and $\odot$ denotes the Hadamard product. The baseline correlation structures for $\boldsymbol{\Theta}$ and $\boldsymbol{\Gamma}^{(k)}$ were generated by two mechanisms: one was an AR(1)-type correlation structure $\rho^{|i-j|}$ with $\rho = 0.9$, $1\leq i,j \leq q$, referred to as setting A09; the other was a random correlation matrix with a prespecified condition number of 100, referred to as setting ALYZ \citep{agostinelli2015robust}. Additional coordinate-wise scale perturbations were then introduced into $\boldsymbol{\Gamma}^{(k)}$ to increase heterogeneity across groups.

    We considered two contamination schemes, namely casewise contamination and cellwise contamination. In the casewise setting, for each group we randomly selected a proportion $\varepsilon_{\mathrm{out}}$ of entire observations and replaced them with outliers located along the eigenvector corresponding to the smallest eigenvalue of $\boldsymbol{\Sigma}^{(k)}$, with a fixed Mahalanobis magnitude relative to $\boldsymbol{\mu}^{(k)}$. In the cellwise setting, for each variable we independently contaminated a proportion $\varepsilon_{\mathrm{out}}$ of cells, and imposed anomalous shifts along the smallest-eigenvalue direction within the corresponding subspace. During estimation, the trimming subset size for the $k$-th group was set to $h_k = \max\{\lfloor (1-\varepsilon_{\mathrm{out}})n_k \rfloor,\; q+1\}$. 

    Specifically, we set $K=5$, $q \in \{70, 20\}$, $n_k \in \{83, 100, 500\}$, $\Delta = 2$, and $\varepsilon_{\mathrm{out}} \in \{0.25,0.5\}$. The detailed combinations are shown in Table~\ref{tab:sim_results_est}. For each configuration, we summarized across independent replications the mean squared errors (MSE) of the mean vectors and the relative errors (RE) of covariance matrices, and also reported the Kullback-Leibler (KL) discrepancy for covariance estimation.
    For non-joint competitors, separate estimations were performed using \texttt{covMCD} from \texttt{robustbase} \citep{martin2026robustbase,valentin2009robustbase} for casewise contamination, and \texttt{cellMCD} from \texttt{cellWise} \citep{raymaekers2024cellwise} for cellwise contamination. Table~\ref{tab:sim_results_est} shows the proposed joint framework's advantage heavily depends on the contamination level and mode.

    \begin{table}[H]
        \centering
        \caption{Simulation results under different settings. ``-'' indicates that the metric exceeds $10^6$. The values in parentheses represent the standard deviation of the 100 replications. $\downarrow$ means the smaller the better.}
        \label{tab:sim_results_est}
        \resizebox{\textwidth}{!}{%
            \begin{tabular}{c|cccc|cccccc}
                \hline
                \multirow{2}{*}{Mode} & \multirow{2}{*}{$k$} & \multirow{2}{*}{$n_k$} & \multirow{2}{*}{$\varepsilon_{\mathrm{out}}$} & \multirow{2}{*}{$q$} & \multicolumn{2}{c}{MSE($\boldsymbol{\mu}$) $\downarrow$} & \multicolumn{2}{c}{KL($\boldsymbol{\Sigma}$) $\downarrow$} & \multicolumn{2}{c}{RE($\boldsymbol{\Sigma}$) $\downarrow$} \\
                \cline{6-7} \cline{8-9} \cline{10-11}
                                      &   &      &   &     & Joint & Non-joint & Joint & Non-joint & Joint & Non-joint \\
                                      \hline
                \multirow{15}{*}{casewise} & 1 & 500 & 0.500 & 70 & 0.067 (0.087) & \best{0.004 (0.001)} & 2053.184 (7.963) & \best{32.368 (20.637)} & 112.888 (6.122) & \best{0.567 (0.020)} \\
                                           & 2 & 500 & 0.500 & 70 & 0.067 (0.114) & \best{0.004 (0.001)} & 2053.486 (10.050) & \best{37.130 (22.755)} & 113.360 (5.760) & \best{0.569 (0.018)} \\
                                           & 3 & 500 & 0.500 & 70 & 0.076 (0.115) & \best{0.004 (0.001)} & 2052.545 (9.889) & \best{34.097 (18.656)} & 113.080 (5.571) & \best{0.569 (0.021)} \\
                                           & 4 & 500 & 0.500 & 70 & 0.058 (0.089) & \best{0.019 (0.003)} & \best{2054.223 (7.897)} & 64787.089 (6307.088) & 112.871 (6.092) & \best{1.006 (0.040)} \\
                                           & 5 & 500 & 0.500 & 70 & 0.073 (0.116) & \best{0.052 (0.080)} & \best{2052.555 (10.530)} & - & \best{112.747 (6.004)} & 129.392 (6.887) \\
                                           \cline{2-11}
                                           & 1 & 500 & 0.250 & 70 & \best{0.004 (0.001)} & 0.026 (0.009) & \best{6.933 (0.199)} & 78.172 (133.667) & \best{0.412 (0.008)} & 0.916 (0.085) \\
                                           & 2 & 500 & 0.250 & 70 & \best{0.004 (0.001)} & 0.031 (0.010) & \best{6.936 (0.210)} & 575.151 (879.586) & \best{0.413 (0.007)} & 0.991 (0.109) \\
                                           & 3 & 500 & 0.250 & 70 & \best{0.004 (0.001)} & 0.067 (0.087) & \best{6.947 (0.214)} & - & \best{0.413 (0.007)} & 75.739 (17.417) \\
                                           & 4 & 500 & 0.250 & 70 & \best{0.004 (0.001)} & 0.066 (0.090) & \best{6.886 (0.195)} & 576.000 (136.053) & \best{0.411 (0.007)} & 67.064 (3.762) \\
                                           & 5 & 100 & 0.250 & 70 & \best{0.018 (0.003)} & 0.063 (0.084) & \best{58.688 (1.642)} & 600.097 (164.398) & \best{0.936 (0.023)} & 67.766 (3.720) \\
                                           \cline{2-11}
                                           & 1 & 100 & 0.250 & 20 & \best{0.019 (0.007)} & 0.110 (0.126) & \best{3.160 (0.291)} & 123.034 (88.941) & \best{0.488 (0.030)} & 17.695 (2.431) \\
                                           & 2 & 100 & 0.250 & 20 & \best{0.018 (0.006)} & 0.102 (0.134) & \best{3.150 (0.297)} & 135.536 (119.664) & \best{0.488 (0.034)} & 17.592 (2.468) \\
                                           & 3 & 100 & 0.250 & 20 & \best{0.018 (0.006)} & 0.120 (0.134) & \best{3.114 (0.310)} & 118.094 (86.670) & \best{0.486 (0.033)} & 17.911 (2.318) \\
                                           & 4 & 100 & 0.250 & 20 & \best{0.018 (0.005)} & 0.106 (0.109) & \best{3.074 (0.298)} & 139.030 (89.511) & \best{0.483 (0.032)} & 17.442 (2.560) \\
                                           & 5 & 100 & 0.250 & 20 & \best{0.019 (0.006)} & 0.110 (0.129) & \best{3.123 (0.331)} & 121.324 (75.300) & \best{0.485 (0.033)} & 17.872 (2.467) \\
                                           \hline
                \multirow{15}{*}{cellwise} & 1 & 500 & 0.500 & 70 & \best{0.004 (0.001)} & 0.052 (0.067) & \best{35.450 (9.590)} & - & \best{0.884 (0.203)} & 129.312 (7.089) \\
                                           & 2 & 500 & 0.500 & 70 & \best{0.004 (0.001)} & 0.051 (0.086) & \best{35.181 (8.702)} & - & \best{0.863 (0.181)} & 129.760 (6.841) \\
                                           & 3 & 500 & 0.500 & 70 & \best{0.004 (0.001)} & 0.058 (0.086) & \best{35.135 (8.327)} & - & \best{0.881 (0.187)} & 129.531 (6.701) \\
                                           & 4 & 500 & 0.500 & 70 & \best{0.004 (0.001)} & 0.045 (0.068) & \best{35.586 (9.128)} & - & \best{0.879 (0.197)} & 129.240 (6.917) \\
                                           & 5 & 100 & 0.500 & 70 & \best{0.027 (0.012)} & 0.057 (0.089) & \best{102.690 (26.680)} & - & \best{2.371 (1.322)} & 129.114 (7.002) \\
                                           \cline{2-11}
                                           & 1 & 500 & 0.250 & 70 & \best{0.003 (0.001)} & 0.059 (0.069) & \best{14.332 (0.808)} & 576.208 (157.316) & \best{0.474 (0.010)} & 67.310 (3.780) \\
                                           & 2 & 500 & 0.250 & 70 & \best{0.003 (0.001)} & 0.063 (0.106) & \best{14.441 (0.764)} & 583.249 (182.543) & \best{0.474 (0.010)} & 67.110 (4.034) \\
                                           & 3 & 500 & 0.250 & 70 & \best{0.003 (0.001)} & 0.083 (0.079) & \best{14.434 (0.767)} & - & \best{0.475 (0.009)} & 109.445 (5.974) \\
                                           & 4 & 500 & 0.250 & 70 & \best{0.003 (0.001)} & 0.007 (0.006) & \best{14.367 (0.797)} & 12984.516 (26078.985) & \best{0.473 (0.010)} & 0.656 (0.177) \\
                                           & 5 & 100 & 0.250 & 70 & 0.015 (0.003) & \best{0.004 (0.001)} & 55.587 (2.579) & \best{31.899 (18.133)} & 0.733 (0.016) & \best{0.570 (0.020)} \\
                                           \cline{2-11}
                                           & 1 & 100 & 0.250 & 20 & \best{0.027 (0.009)} & 0.114 (0.124) & \best{7.399 (1.149)} & 101.185 (58.588) & \best{0.693 (0.070)} & 17.660 (2.353) \\
                                           & 2 & 100 & 0.250 & 20 & \best{0.027 (0.007)} & 0.028 (0.009) & \best{7.380 (1.197)} & 178.290 (457.589) & \best{0.690 (0.071)} & 0.935 (0.096) \\
                                           & 3 & 100 & 0.250 & 20 & \best{0.026 (0.009)} & 0.027 (0.009) & \best{7.428 (1.153)} & 66.127 (114.598) & \best{0.681 (0.065)} & 0.923 (0.084) \\
                                           & 4 & 100 & 0.250 & 20 & \best{0.025 (0.009)} & 0.027 (0.009) & \best{7.369 (0.993)} & 93.969 (180.187) & \best{0.689 (0.064)} & 0.922 (0.084) \\
                                           & 5 & 83  & 0.250 & 20 & 0.031 (0.010) & \best{0.028 (0.008)} & \best{8.967 (1.378)} & 78.029 (137.850) & \best{0.720 (0.072)} & 0.920 (0.096) \\
                                           \hline
            \end{tabular}%
        }
    \end{table}

    Under a high outlier ratio ($\varepsilon_{\mathrm{out}}=0.500$) with casewise contamination, the joint method trades some estimation accuracy for enhanced stability, preventing the severe numerical breakdown often observed in the non-joint approach. Conversely, under a moderate outlier ratio ($\varepsilon_{\mathrm{out}}=0.250$), the benefits of the joint method are evident: it consistently achieves lower estimation errors across different dimensions in casewise scenarios and improves covariance estimation stability under cellwise contamination.

\end{appendix} 

\newpage
\justifying
\bibliography{bib/bib}
\bibliographystyle{elsarticle-num-names}

\end{document}